\documentclass{article}

\usepackage[T1]{fontenc}
\usepackage{fullpage}
\usepackage{amsmath, amsthm, amsfonts, amssymb, mathrsfs, mathtools}
\usepackage[libertine]{newtxmath} % for \mathscr of lowercase letters
\usepackage{graphicx}
\usepackage{subfigure,tikz} % vector graphics %TC: uncomment if this is not included in pgfplots
    \usetikzlibrary{decorations.pathreplacing} % for curly braces in pictures
\usepackage{enumitem} % compact lists
\usepackage{nicefrac} % compact fractions
\usepackage[ruled, noline,noend]{algorithm2e} % algorithms
    \renewcommand{\algorithmcfname}{ALGORITHM}
\usepackage{pgfplots}
\pgfplotsset{compat=1.15}
\usetikzlibrary{arrows}
\usepackage{natbib} % bibliography
\usepackage{mleftright}%remove the extra horizontal spacing in left/right
    \mleftright
\usepackage{multicol}
\usepackage{thmtools,thm-restate} % to reference a lemma in the appendix
\usepackage[skip = 3pt]{caption}
\usepackage{authblk} % author block
\usepackage{hyperref} % TC: hyperlinks: KEEP IT SECOND-TO-LAST!
\usepackage[capitalize]{cleveref} % TC: hyperlinks: KEEP IT LAST!

\newtheorem{theorem}{Theorem}[section]
\newtheorem{lemma}{Lemma}[section]
\newtheorem*{theorem*}{Theorem}

\theoremstyle{definition}
\newtheorem{definition}{Definition}

\newcommand{\cD}{\mathcal{D}}
\newcommand{\cE}{\mathcal{E}}

\newcommand{\cG}{\mathcal{G}}

\newcommand{\cO}{\mathcal{O}}

\newcommand{\cR}{\mathcal{I}}

\newcommand{\cU}{\mathcal{U}}
\newcommand{\cV}{\mathcal{V}}
\newcommand{\cW}{\mathcal{W}}
\newcommand{\cX}{\mathcal{X}}
\newcommand{\cY}{\mathcal{Y}}
\newcommand{\cZ}{\mathcal{Z}}
\newcommand{\E}{\mathbb{E}}
\renewcommand{\exp}[1]{\mathbb{E}\left[ #1 \right]}

\newcommand{\I}{\mathbb{I}}

\newcommand{\N}{\mathbb{N}}

\renewcommand{\P}{\mathbb{P}}
\newcommand{\Q}{\mathbb{Q}}

\newcommand{\e}{\varepsilon}
\newcommand{\fhi}{\varphi}
\newcommand{\tht}{\vartheta}
\newcommand{\lrb}[1]{\left(#1\right)}
\newcommand{\brb}[1]{\bigl(#1\bigr)}
\newcommand{\Brb}[1]{\Bigl(#1\Bigr)}
\newcommand{\bbrb}[1]{\biggl(#1\biggr)}
\newcommand{\lsb}[1]{\left[#1\right]}
\newcommand{\bsb}[1]{\bigl[#1\bigr]}
\newcommand{\Bsb}[1]{\Bigl[#1\Bigr]}
\newcommand{\lcb}[1]{\left\{#1\right\}}
\newcommand{\bcb}[1]{\bigl\{#1\bigr\}}

\newcommand{\lce}[1]{\left\lceil#1\right\rceil}
\newcommand{\bce}[1]{\bigl\lceil#1\bigr\rceil}

\newcommand{\labs}[1]{\left\lvert#1\right\rvert}
\newcommand{\babs}[1]{\bigl\lvert#1\bigr\rvert}
\newcommand{\Babs}[1]{\Bigl\lvert#1\Bigr\rvert}
    \colorlet{RED}{red}
    \newcommand{\tccheck}{}

\DeclareMathOperator*{\argmax}{argmax}
\newcommand{\dif}{\,\mathrm{d}}

\DeclareMathOperator*{\gft}{GFT}
\DeclareMathOperator*{\GFT}{\gft\nolimits}
\newcommand{\lip}{Lipschitz}
\newcommand{\leb}{Lebesgue}
\newcommand{\s}{\subset}
\newcommand{\m}{\setminus}
\newcommand{\iop}{\infty}

\newcommand{\xs}{x^\star}
\newcommand{\ps}{p^\star}
\newcommand{\qs}{q^\star}
\newcommand{\Rs}{R^\star}

\newcommand{\fA}{\sA}
\newcommand{\fG}{\sG}

\newcommand{\fP}{\sP}

\newcommand{\sA}{\mathscr{A}}

\newcommand{\sG}{\boldsymbol{\mathscr{G}}}
\newcommand{\sP}{\mathscr{P}}
\newcommand{\sS}{\mathscr{S}}
\newcommand{\slf}{\tilde{\mathscr{f}}}
\newcommand{\slg}{\mathscr{g}}
\newcommand{\slh}{\mathscr{h}}
\newcommand{\bmu}{\boldsymbol{\mu}}
\renewcommand{\hat}[1]{\widehat{#1}}
\newcommand{\sPiid}{\sP_{\mathrm{iid}}}
\newcommand{\sPiv}{\sP_{\mathrm{iv}}}
\newcommand{\sPbd}{\sP_{\mathrm{bd}}^M}
\newcommand{\sPivbd}{\sP_{\mathrm{iv+bd}}^M}
\newcommand{\sPadv}{\sP_{\mathrm{adv}}}
\newcommand{\ber}{\mathrm{Ber}}
\newcommand{\ind}{\mathbb{I}}
\newcommand{\st}{S'_t}
\newcommand{\bt}{B'_t}

\newcommand{\psb}{\P_{(S,B)}}
\SetKwComment{Comment}{$\triangleright$\ }{}

\title{A Regret Analysis of Bilateral Trade\thanks{Partially supported by ERC Advanced Grant 788893 AMDROMA ``Algorithmic and Mechanism Design Research in Online Markets'' and MIUR PRIN project ALGADIMAR ``Algorithms, Games, and Digital Markets''. This work has also benefited from the AI Interdisciplinary Institute ANITI. ANITI is funded by the French ``Investing for the Future – PIA3'' program under the Grant agreement n. ANR-19-PI3A-0004.}}

\author[1]{Nicol\`o Cesa-Bianchi}
\author[2,3]{Tommaso R. Cesari}
\author[1,4]{Roberto Colomboni}
\author[5]{Federico Fusco}
\author[5]{Stefano Leonardi}

\affil[1]{Universit\`a degli Studi di Milano, Milano, Italy}
\affil[2]{Toulouse School of Economics (TSE), Toulouse, France}
\affil[3]{Artificial and Natural Intelligence Toulouse Institute (ANITI), Toulouse, France}
\affil[4]{Istituto Italiano di Tecnologia, Genova, Italy}
\affil[5]{Sapienza Universit\`a di Roma, Roma, Italy}

\begin{document}

\maketitle

\begin{abstract}
    Bilateral trade, a fundamental topic in economics, models the problem of intermediating between two strategic agents, a seller and a buyer, willing to trade a good for which they hold private valuations. Despite the simplicity of this problem, a classical result by Myerson and Satterthwaite (\citeyear{MyersonS83}) affirms the impossibility of designing a mechanism which is simultaneously efficient, incentive compatible, individually rational, and budget balanced. 
    
    This impossibility result fostered an intense investigation of meaningful trade-offs between these desired properties. Much work has focused on approximately efficient fixed-price mechanisms, i.e., Blumrosen and Dobzinski (\citeyear{BlumrosenD14, BlumrosenD16}), Colini-Baldeschi et al. (\citeyear{Colini-Baldeschi16}), which have been shown to fully characterize strong budget balanced and ex-post individually rational direct revelation mechanisms.
    All these results, however, either assume some knowledge on the priors of the seller/buyer valuations, or a black box access to some samples of the distributions, as in Dütting et al. (\citeyear{Duetting20}). 
    
    In this paper, we cast for the first time the bilateral trade problem in a regret minimization framework over $T$ rounds of seller/buyer interactions, with no prior knowledge on the private seller/buyer valuations. Our main contribution is a complete characterization of the regret regimes for fixed-price mechanisms with different models of feedback and private valuations, using as benchmark the best fixed price in hindsight. More precisely, we prove the following bounds on the regret:
    \begin{itemize}
        \item $\widetilde{\Theta}(\sqrt{T})$ for full-feedback (i.e., direct revelation mechanisms);
        \item $\widetilde{\Theta}(T^{2/3})$ for realistic feedback (i.e., posted-price mechanisms) and independent seller/buyer valuations with bounded densities; 
        \item  $\Theta(T)$ for realistic feedback and seller/buyer valuations with bounded densities; 
        \item  $\Theta(T)$ for realistic feedback and independent seller/buyer valuations;
        \item $\Theta(T)$ for the adversarial setting.
    \end{itemize}
\end{abstract}

\clearpage

\tableofcontents

\clearpage

\section{Introduction \tccheck}

In the bilateral trade problem, two strategic agents ---a seller and a buyer--- wish to trade some good. They both privately hold a personal valuation for it, and strive to maximize their own quasi-linear utility. 
An ideal mechanism for this problem would optimize the efficiency, i.e., the social welfare resulting by trading the item, while enforcing incentive compatibility (IC) and individual rationality (IR). 
The assumption that makes two-sided mechanism design  more complex than the one-sided counterpart is budget balance (BB): the mechanism cannot subsidize or make a profit from the market. 
Unfortunately, as Vickrey observed in his seminal work \cite{Vickrey61}, the optimal incentive compatible mechanism maximizing social welfare for bilateral trade may not be budget balanced.

A more general result due to Myerson and Satterthwaite \cite{MyersonS83} shows that a fully efficient mechanism for bilateral trade that satisfies IC, IR, and BB may not exist at all. 
This impossibility result holds even if prior information on the buyer and seller's valuations is available, the truthful notion is relaxed to Bayesian incentive compatibility (BIC), and the exact budget balance constraint is loosened to weak budget balance (WBB).
To circumvent this obstacle, a long line of research has focused on the design of approximating mechanisms that satisfy the above requirements while being nearly efficient. 

These approximation results build on a Bayesian assumption: seller and buyer valuations are drawn from two distributions, which are both known to the mechanism designer. 
Although in some sense necessary ---without any information on the priors there is no way to extract any meaningful approximation result  \cite{Duetting20}--- this assumption is unrealistic.  
Following a recent line of research \cite{Cesa-BianchiGM15,LykourisST16,DaskalakisS16}, in this work we study this basic mechanism design problem in a regret minimization setting.
Our goal is bounding the total loss in efficiency experienced by the mechanism in the long period by learning the salient features of the prior distributions. 

At each time $t$, a new seller/buyer pair arrives. 
The seller has a private valuation $s_t \in [0,1]$ representing the smallest price she is willing to accept in order to trade. 
Similarly, the buyer has a private value $b_t \in [0,1]$ representing the highest price that she will pay for the item. 
The mechanism sets a price $p_t$ which results in a trade if and only if $s_t \le p_t \le b_t$. 

There are two common utility functions that reflect the performance of the mechanism at each time step: the social welfare, which sums the utilities of the two players after the trade (and remains equal to the seller's valuation if no trade occurs), and the gain from trade, consisting in the net gain in the utilities. 
In formulae, 
\begin{itemize}
    \item \textbf{Social Welfare}: $\mathrm{SW} ( p_t, s_t,b_t ) = \mathrm{SW}_t(p_t)= s_t + (b_t-s_t) \I \{ s_t \leq p \leq b_t \}$;
     \item \textbf{Gain from Trade}: $\gft(p_t,s_t,b_t)= \gft\nolimits_t(p_t)=(b_t-s_t)\I \{ s_t \leq p \leq b_t \}$.
\end{itemize}
We begin by investigating the standard assumption in which $s_t$ and $b_t$ are realizations of $S_t$ and $B_t$, where $(S_1,B_1),(S_2,B_2),\ldots$ are i.i.d.\ random variables, supported in $[0,1]^2$, representing the valuations of seller and buyer respectively (stochastic i.i.d.\ setting). 
We also consider the case where $(s_1,b_1),(s_2,b_2),\dots$ is an arbitrary deterministic process (adversarial setting).

In our online learning framework, we aim at minimizing the \emph{regret} over a time horizon $T$:
\[
    \max_{p \in [0,1]} \exp{\sum_{t=1}^T \gft\nolimits_t(p) - \sum_{t=1}^T \gft\nolimits_t(p_t)} \;.
\]
Note that since $\gft\nolimits_t(p_t) = \mathrm{SW}_t(p_t) -s_t$ and $s_t$ does not depend on the choice of $p$,  gain from trade and social welfare lead to the same notion of regret. 

The regret is hence the difference between the expected total performance of our algorithm, which can only \emph{sequentially learn} the distribution, and our reference benchmark, corresponding to the the best fixed-price strategy assuming \emph{full knowledge} of the distribution. 
Our main goal is to design strategies with asymptotically vanishing time-averaged regret with respect to the best fixed-price strategy or, equivalently, regret sublinear in the time horizon $T$. 

The class of fixed price mechanisms is of particular importance in bilateral trade as they are simple to implement, clearly truthful, individually rational, budget balanced, and enjoy the desirable property of asking the agents very little information. Moreover, it can be shown that fixed prices are the {\em only} direct revelation mechanisms which enjoy strong budget balance, dominant strategy incentive compatibility, and ex-post individual rationality \cite{Colini-Baldeschi16}.

To complete the description of the problem, we need to specify the feedback obtained by the mechanism after each sequential round. We propose two main feedback models: 
\begin{itemize}
    \item {\em Full feedback.} In the full feedback model, the pair $(s_t,b_t)$ is revealed to the mechanism after the $t$-th trading round. The information collected by this feedback model corresponds to {\em direct revelation mechanisms}, where the agents communicate their valuations before each round, and the price proposed by the mechanism at time $t$ only depends on past bids.
    \item {\em Realistic feedback.}    In the harder realistic feedback model, only the relative orderings between $s_t$ and $p_t$ and between $b_t$ and $p_t$ are revealed after the $t$-th round. This model corresponds to {\em posted price mechanisms}, where seller and buyer separately accept or refuse the posted price. The price computed at time $t$ only depends on past bids, and the values $s_t$ and $b_t$ are never revealed to the mechanisms. 
\end{itemize}

\subsection{Overview of our Results \tccheck}

We investigate the stochastic setting (under various assumptions), the adversarial setting, and how regret bounds change depending on the quality of the received feedback.
In all cases, we provide matching upper and lower bounds.
In particular, our positive result are constructive: explicit algorithms are given for each of them.
More precisely, we show (see \cref{tab:general} for a summary):
\begin{table}
\centering
\begin{tabular}{l|l|l|l|l|l|}
\cline{2-6}
\multicolumn{1}{c|}{}      & \multicolumn{4}{c|}{\textbf{Stochastic (iid)}}                                                                              & \multicolumn{1}{c|}{\textbf{Adversarial}} \\ \cline{2-6} 
\multicolumn{1}{c|}{}      & \multicolumn{1}{c|}{\textbf{iid}} & \multicolumn{1}{c|}{\textbf{+iv}} & \multicolumn{1}{c|}{\textbf{+bd}} & \multicolumn{1}{c|}{\textbf{+iv+bd}} & \multicolumn{1}{c|}{\textbf{adv}}         \\ \hline
\multicolumn{1}{|l|}{\textbf{Full}} & $T^{1/2}$ (thm~\ref{thm:upper_full})                & $T^{1/2}$                & $T^{1/2}$                & $T^{1/2}$ (thm~\ref{thm:lower-full})                   & $T$ (thm~\ref{thm:adv-lower})                              \\ \hline
\multicolumn{1}{|l|}{\textbf{Real}} & $T$                      & $T$ (thm~\ref{thm:lower-real-iv})                      & $T$ (thm~\ref{thm:lower-real-bd})                      & $T^{2/3}$ (thms~\ref{thm:upper_real}+\ref{thm:lower-real-iv+bd})                   & $T$                              \\ \hline
\end{tabular}
\caption{Our main results for fixed price mechanisms. 
The rates are both upper and lower bounds, up to a $\log T$ factor.
The slots without references are immediate consequences of the others.}
\label{tab:general}
\end{table}
\begin{itemize}
    \item \cref{alg:followTheBestPrice} (Follow the Best Price) for the full-feedback model achieving a $\widetilde{\cO}(T^{1/2})$ regret in the stochastic (iid) setting (\cref{thm:upper_full});
    this rate cannot be improved by more than a $\log T$ factor, not even under some additional natural assumptions (\cref{thm:lower-full});
    \item \cref{alg:meta} (Scouting Bandits) for the harder realistic-feedback model achieving a $\widetilde{\cO}(T^{2/3})$ regret in a stochastic (iid) setting in which the valuations of the seller and the buyer are independent of each other (iv) and have bounded densities (bd) (\cref{thm:upper_real});
    this rate cannot be improved by more than a $\log T$ factor (\cref{thm:lower-real-iv+bd});
    \item impossibility results:
    \begin{itemize}
        \item for the realistic-feedback model, if either the (iv) or the (bd) assumptions are dropped from the previous stochastic setting, no strategy can achieve sublinear worst-case regret (\cref{thm:lower-real-iv,thm:lower-real-iv+bd});
        \item in an adversarial setting, no strategy can achieve sublinear worst-case regret, not even in the simpler full-feedback model (\cref{thm:lower-full}).
    \end{itemize}
\end{itemize}

\subsection{Technical Challenges \tccheck}

The two feedback models we consider are both challenging for different reasons.
    
\paragraph{Full feedback.} The full feedback model fits nicely in the learning with expert advice framework \cite{Nicolo06}.
Each price $p \in [0,1]$ can be viewed as an expert, and the revelation of $s_t$ and $b_t$ allows the mechanism to compute $\gft\nolimits_t(p)$ for all $p$, including the mechanism's own reward $\gft\nolimits_t(p_t)$. 
A common approach to reduce the cardinality of a continuous expert space is to assume some regularity (e.g., \lip{}ness) of the reward function, so that a finite grid of {\em representative} prices can be used. This approach yields a $\widetilde{O}(\sqrt{T})$ bound under density boundedness assumptions on the joint distribution of the seller and the buyer. 
By exploiting the structure of the reward function $\E \bsb{\gft\nolimits_t(\cdot)}$, we obtain the same regret bound without any assumptions on the distribution (other than iid).
In \cref{thm:upper_full}, we show how to decompose the expression of the expected gain from trade in pieces that can be quickly learned via sampling. The full feedback received in each new round is used to refine the estimate of the actual gain from trade as a function of the price, while the posted prices are chosen so to maximize it. Our {\em Follow the Leader} strategy is shown to achieve a $\widetilde{\cO}(\sqrt{T})$ bound in the stochastic (iid) setting. 
This holds for arbitrary joint distributions of the seller and the buyer.
In particular, even when the buyer and seller have a correlated behavior.
The main challenge for the lower bounds is how to embed a hard instance in a setting where we cannot control the gain from trade, but only the distributions of seller and buyer. We solve this problem by designing a reduction from a $2$-action partial monitoring game to our setting, and then using known lower bounds for partial monitoring.

\paragraph{Realistic Feedback.} Here, at each time $t$, only $\ind\{s_t \le p_t\}$ and $\ind\{p_t \le b_t\}$ are revealed to the mechanism. Hence, we face the two competing goals of estimating the underlying distributions while optimizing the estimated gain from trade. The realistic feedback model does not fit the expert prediction framework, nor the harder bandits model \cite{Nicolo06}, because the observations of $\ind\{s_t \le p_t\}$ and $\ind\{p_t \le b_t\}$ are not enough to reconstruct the gain from trade at time $t$. On the one hand, if the trade occurs, there is no way to directly infer the difference $b_t-s_t$. On the other hand, if the trade does not occur, little can be done to argue which prices would have resulted in a trade. We show how to decompose the expected gain from trade at a posted price $p$ into a {\em global} part that can be be quickly estimated by uniform sampling on the whole $[0,1]$ interval, and a {\em local} part that can be only learned by posting selected prices. \cref{thm:upper_real} shows a general technique (Scouting Bandits, \cref{alg:meta}) which takes advantage of this decomposition, and relies on any bandit algorithm to learn the local part of the expected gain from trade. We derive a sublinear regret of $\widetilde \cO(T^{2/3})$ in a stochastic (iid) setting in which the valuations of the seller and the buyer are independent of each other (iv) and have bounded densities (bd).
Dropping the (iv) assumption leads to a pathological \emph{lack of observability} phenomenon in which it is impossible to distinguish between two scenarios, and the optimal price in one of them is highly suboptimal in the other (\cref{thm:lower-real-bd}).
Dropping the (bd) assumption leads to a pathological {\em needle in a haystack} phenomenon in which all but one prices suffer a high regret and it is essentially impossible to find the optimal one among the continuum amount of suboptimal ones (\cref{thm:lower-real-iv}). 
Similarly to the full feedback lower bound, the realistic feedback lower bound is based on reducing a partial monitoring game to our setting. However, additional challenges arise in this case due to the specific nature of the realistic feedback, see \Cref{thm:lower-real-bd,thm:lower-real-iv}.

\paragraph{Adversarial setting.}
Finally, we investigate the adversarial setting in which the valuations of the buyer and the seller form an arbitrary deterministic process generated by an oblivious adversary.
This setting is significantly more challenging than the stochastic (iid) case.
Indeed, using a construction inspired by Cantor ternary set, we show that even under a full-feedback model, no strategy can lead to a sublinear worst-case regret

\paragraph{Lower Bound Techniques}
Due to space constraints, the proofs of the stochastic lower bounds, i.e., \Cref{thm:lower-full,thm:lower-real-bd,thm:lower-real-iv,thm:lower-real-iv+bd} are only sketched in the main text and completed in the Appendix. In particular, the formal reductions from various instances of partial monitoring rely on a very general notion of sequential games subsuming both partial monitoring and our problems. These reductions are shown through two key lemmas (\ref{l:embedding}, \ref{l:simulation}): an Embedding Lemma and e Simulation Lemma, which may be of independent interest.

\subsection{Further Related Work \tccheck}
The study of the bilateral trade problem dates back to the already mentioned seminal works of Vickrey \cite{Vickrey61} and Myerson and Satterthwaite \cite{MyersonS83}.
A more recent line of research focused on Bayesian mechanisms that achieve the IC, BB, and IR requirements while approximating the optimal social welfare or the gain form trade.  Blumrosen and Dobzinski \cite{BlumrosenD14} proposed the \emph{median mechanism} that sets a posted price equal to the median of the seller distribution and shows that this mechanism obtains an approximation factor of $2$ to the optimal social welfare. Subsequent work by the same authors \cite{BlumrosenD16}  improved the approximation guarantee to $e/(e-1)$ through a randomized mechanism whose prices depend on the seller distribution in a more intricate way. 
In ~\cite{Colini-Baldeschi16} it is demonstrated that all DSIC mechanisms that are BB and IR  must post a fixed price to the buyer and to the seller. 
In a different research direction aimed to characterize the information theoretical requirements of two-sided markets mechanisms, \citep{Duetting20} shows that setting the price equal to a single sample from the seller distribution gives a $2$-approximation to the optimal social welfare.  
In a parallel line of work it has been considered the harder objective of approximating the \emph{gain from trade}. 
An asymptotically tight fixed-price $O\big(\log\frac 1r\big)$ approximation  bound is also achieved in \cite{Colini-Baldeschi17}, with $r$ being the probability that a trade happens (i.e., the value of the buyer is higher than the value of the seller).  A BIC $2$-approximation of the second best with a simple mechanism is obtained in \cite{BrustleCWZ17}.

In the following we discuss the relationship between the approximation results mentioned above and the regret analysis we develop in this work that compares online learning mechanisms against the best ex-ante fixed price mechanism. First of all, in the realistic feedback setting, the approximation mechanisms for bilateral trade cannot be easily implemented.  For example, the single sample $2$-approximation to the optimal social welfare \cite{Duetting20} requires multiple rounds of interaction in order to obtain, approximately, a random sample from the distribution.  The median mechanism of \cite{BlumrosenD14} requires an even larger number of rounds in order to estimate the median of the seller distribution.  
It is also  interesting to relate the guarantee of our online algorithms with the one provided by the approximation mechanisms. Here we notice that the two approaches cannot be directly compared as there exist simple examples% 
\footnote{Consider a seller with value $\varepsilon>0$ or $0$ with equal probability and a buyer with value $1$. The best fixed price has welfare of $1$. For small $\varepsilon$, the median and the sample mechanism, respectively, obtains a welfare close to $1/2$ and $3/4$.} 
showing for the median and the sample mechanisms, respectively, a factor of $2$ and $4/3$ away from the optimum fixed price ex-ante, whereas our online learning approach provides a strictly better sublinear regret.

There is a vast body of literature on regret analysis in (one-sided) dynamic pricing and online posted price auctions ---see, e.g., the excellent survey published by \cite{den2015dynamic} and the tutorial slides by \cite{SZ15}. In their seminal paper, Kleinberg and Leighton prove a $O(T^{2/3})$ upper bound (ignoring logarithmic factors) on the regret in the adversarial setting \cite{kleinberg2003value}. Later works show simultaneous multiplicative and additive bounds on the regret when prices have range $[1,h]$ \cite{blum2004online,blum2005near}. These bounds have the form $\varepsilon\,G_T^{\star} + O\big((h\ln h)/\varepsilon^2\big)$ ignoring $\ln\ln h$ factors, where $G_T^{\star}$ is the total revenue of the optimal price $p^{\star}$. Recent improvements on these results prove that the additive term can be made $O(p^{\star}\brb{ \ln h)/\varepsilon^2 }$, where the linear scaling is now with respect to the optimal price rather than the maximum price $h$ \cite{bubeck2017online}. Other variants consider settings in which the number of copies of the item to sell is limited \cite{agrawal2014bandits,babaioff2015dynamic,badanidiyuru2013bandits}, buyers act strategically in order to maximize their utility in future rounds \cite{amin2013learning,devanur2014perfect,mohri2014optimal,drutsa2018weakly}, or there are features associated with the goods on sale \cite{DBLP:journals/mansci/CohenLL20}. In the stochastic setting, previous works typically assume parametric \cite{broder2012dynamic}, locally smooth \cite{kleinberg2003value}, or piecewise constant demand curves \cite{cesa2019dynamic,den2020discontinuous}.

\section{The Bilateral Trade learning protocol \tccheck} 
\label{s:bil-tr-model}
In this section, we present the learning protocol for the sequential problem of bilateral trade (see learning protocol~\ref{a:learning-model}).
We recall that the reward collected from a trade is the gain from trade, defined for all $p,s,b \in [0,1]$, by $\gft(p,s,b) := (b-s) \I \{s\le p \le b\}$.
{
%the following line in this block {...} change the name of the algorithm2e environment from Algorithm to Online Protocol
\renewcommand*{\algorithmcfname}{LEARNING PROTOCOL}
\begin{algorithm}
\For
{%
    time $t=1,2,\ldots$
}
{
    a new seller/buyer pair arrives with (hidden) valuations $(S_t,B_t) \in [0,1]^2$\;
    the learner posts a price $P_t \in [0,1]$\;
    the learner receives a (hidden) reward $ \GFT_t(P_t) := \gft ( P_t, S_t,B_t) \in [0,1]$\;
    a feedback $Z_t$ is revealed\;
}
 \caption{Bilateral Trade}
 \label{a:learning-model}
\end{algorithm}
}

At each time step $t$, a seller and a buyer arrive, each with a privately held valuation $S_t,B_t \in [0,1]$.
The learner then posts a price $P_t \in [0,1]$ and a trade occurs if and only if $S_t \le P_t \le B_t$.
When this happens, the learner gains a reward $\gft (P_t, S_t,B_t)$, which is not revealed. 
Some feedback $Z_t$ is revealed instead.
The nature of the sequence of valuations $(S_1,B_1),(S_2,B_2),\ldots$ and feedbacks $Z_1,Z_2,\ldots$ depends on the specific instance of the problem and is described below.

The goal of the learner is to determine a strategy $\alpha$ generating the prices $P_1,P_2, \ldots$ (as in Learning~Model~\ref{a:learning-model}) achieving sublinear \emph{regret}
\[
    R_T(\alpha)
:=
    \max_{p \in [0,1]} \E\lsb{\sum_{t=1}^T \gft ( p, S_t,B_t)  - \sum_{t=1}^T \gft ( P_t, S_t,B_t)}
    \;,
\]
where the expectation is taken with respect to the sequence of buyers and sellers and (possibly) the internal randomization of $\alpha$.
To lighten the notation, we denote by $\ps$ (one of) the $p\in[0,1]$ maximizing the previous expectation.

We now introduce several instances of bilateral trade, depending on the type of the received feedback and the nature of the environment.

\subsection{Feedback} 

\begin{description}
    \item[Full feedback:] the feedback $Z_t$ received at time $t$ is the entire seller/buyer pair $(S_t,B_t)$;
    in this setting, the seller and the buyer reveal their valuations at the end of a trade.
    \item[Realistic feedback:] the feedback $Z_t$ received at time $t$ is the pair $\brb{ \I\{S_t\le P_t\}, \, \I\{P_t \le B_t \} }$;
    in this setting, the seller and the buyer only reveal whether or not they accept the trade at price $P_t$.
\end{description}

\subsection{Environment} 
\begin{description}
    \item[Stochastic (iid):] 
    $(S_1,B_1),(S_2,B_2),\ldots$ is an i.i.d.\ sequence of seller/buyer pairs, where $S_t$ and $B_t$ could be (arbitrarily) correlated.
    
    We will also investigate the (iid) setting under the following further assumptions.
    \begin{description}
        \item[Independent valuations (iv):]
        $S_t$ and $B_t$ are independent of each other.
        \item[Bounded density (bd):] 
        $(S_t,B_t)$ admits a joint density bounded by some $M \ge 1$.
    \end{description}
    \item[Adversarial (adv):] 
    $(S_t,B_t)_{t\in\N}$ is an arbitrary deterministic sequence $(s_t,b_t)_{t\in\N} \s [0,1]^2$.
\end{description}

\section{Full-Feedback Stochastic (iid) Setting \tccheck}

We begin by considering the full-feedback model (corresponding to revelation mechanisms) under the assumption that the seller/buyer pairs $(S_1,B_1),(S_2,B_2),\ldots$ are $[0,1]^2$-valued i.i.d.\ random variables, without any further assumptions on their common distribution $(S,B)$ (in particular, $S$ and $B$ could be arbitrarily correlated). 
Here, sellers and buyers declare their actual valuations to the mechanism. The incentive compatibility is guaranteed by the fact that the posted prices does not depend on the declared valuations at each specific round, but only on past ones, so that there is no point in misreporting.

In \cref{sec:DKW}, we show that a {\em Follow the Leader} approach, which we call Follow the Best Price (FBP), whose pseudocode is given in \cref{alg:followTheBestPrice}, achieves a $O(\sqrt{T\log T})$ upper bound.
In \cref{sec:lower_bound_indep}, we provide a lower bound that matches this rate, up to a $\sqrt{\log T}$ factor.

\subsection{Follow the Best Price (FBP) \tccheck}
\label{sec:DKW}

We begin by presenting our Follow the Best Price (FBP) algorithm. It consists in posting the best price with respect to the samples that have been observed so far. Notably, it does not need preliminary knowledge of the time horizon $T$.

\begin{algorithm}
    \SetKwInput{kwInit}{init}
    \kwInit{select $P_1 \in [0,1]$ arbitrarily}
    \For{$t=1,2, \ldots$}
    {
        post price $P_t$\;
        receive feedback $(S_t,B_t)$\;
        compute $P_{t+1} \gets \argmax_{p \in\{S_1,B_1,\dots,S_t,B_t\}} \sum_{i=1}^t \GFT(p, S_i, B_i) \in [0,1] $\Comment*[r]{Ties broken arbitrarily}
    }
    \caption{Follow the Best Price (FBP)}
    \label{alg:followTheBestPrice}
\end{algorithm}

For each time $t\ge 2$, let $O_t$ be the sequence containing all the $t-1$ pairs of valuations observed so far, i.e., $O_t:=\brb{ (S_1,B_1), \dots, (S_{t-1},B_{t-1}) }$. 
Given $O_t$, one can reconstruct the actual $\GFT_i(\cdot) := \gft(\cdot, S_i, B_i)$ function at each past time step $i \le t-1$ and compute (one of) the best price(s)
\begin{equation}
    \label{eq:follow_best}
    P_t \in \argmax_{p \in [0,1]} \sum_{i=1}^{t-1} \gft\nolimits_i(p) \;.
\end{equation}
Note that at least one of the elements in the $\argmax$ belongs to the set of past valuations, given the structure of the gain from trade, so even a naive enumeration approach is computationally efficient.

Before moving on, we describe a property on $P_t$ which will be useful in the analysis.
Given the samples in $O_t$, it is possible to build an estimate of the random pair $(S,B)$ of which they are i.i.d.\ samples.
More precisely, one can consider the random pair $(\st,\bt)$ which follows the empirical distribution of the data, i.e., for all $(s,b) \in [0,1]^2$
\[
    \P \bsb{ (\st,\bt) = (s, b) \mid O_t }
=
    \frac{1}{t-1}\sum_{i=1}^{t-1} \I \bcb{ (S_i,B_i) = (s,b) } \;.
\] 
The formulation in \Cref{eq:follow_best} is then equivalent to finding a price $P_t$ that maximizes the expected gain from trade for the seller and buyer's valuations with respect to $(\st,\bt)$. 
Indeed, let $\gft'\nolimits_t$ be the gain from trade associated to $(\st,\bt)$, i.e., $\GFT_t'(\cdot) := \gft(\cdot, \st, \bt)$, then, for all $p \in [0,1]$,
\begin{equation}
    \label{eq:empirical}
    \exp{\gft'\nolimits_t(p) \mid O_t}
= 
    \frac{1}{t-1}\sum_{i=1}^{t-1} \gft\nolimits_i(p)\;,
\end{equation}
where the expectation conditioned to $O_t$ is with respect to a random sample of $(\st,\bt)$.

We show now that if at a certain time $t\ge 2$ the distribution of $(\st,\bt)$ given $O_t$ is close to the distribution of $(S,B)$, then our strategy performs well, on expectation, if compared to the optimal price $\ps$. 
In order to do so, we first write the following decomposition:
\begin{align*}
    \E\bsb{\gft\nolimits_t(\ps)} - \E\bsb{\gft\nolimits_t(P_t) \mid O_t} 
    & = \E\bsb{\gft\nolimits_t(\ps)} - \E\bsb{\gft'\nolimits_t(\ps)\mid O_t} \\
    & \qquad + \E\bsb{\gft'\nolimits_t(\ps)\mid O_t} - \E\bsb{\gft'\nolimits_t(p)\mid O_t}\big|_{p = P_t} \\
    & \qquad \qquad + \E\bsb{\gft'\nolimits_t(p)\mid O_t}\big|_{p = P_t} - \E\bsb{\gft\nolimits_t(p) }\big|_{p = P_t} \;.
\end{align*}
where, for a function $g$, we denoted $g(a):=g(p)\big|_{p=a}$.
Note that the middle term is always non-positive, since $P_t$ maximizes the expected gain from trade of $(\st,\bt)$ given $O_t$ by \eqref{eq:empirical}. 
Hence, we have
\begin{equation}
\label{eq:triangular}
    \E\bsb{\gft\nolimits_t(\ps)} - \E\bsb{\gft\nolimits_t(P_t) \mid O_t} \le   2\max_{q \in [0,1]} \Babs{ \E\bsb{\gft\nolimits_t(q)} - \E\bsb{\gft'\nolimits_t(q) \mid O_t} } \;.
\end{equation}
If at each time step $t\ge 2$ the mechanism inherits a good estimate of the distribution of $(S_t,B_t)$, then we show that the best price given the past performs almost as well as the optimal price $\ps$. 
This is a consequence of the following lemma, which reduces the problem to accurately estimating the distribution of $(S_t,B_t)$ on the rectangles $\mathcal{R}=\bcb{[a,b]\times[c,d] \mid a,b,c,d \in [0,1] }$.

\begin{lemma}[First Decomposition Lemma]
\label{lem:first_decomposition}
Let $\mu$ be any probability measure on $[0,1]^2$ and $p \in [0,1]$, then
\[
    \E_{(s,b)\sim \mu} \bsb{ \gft(p,s,b) }
= 
    \int_0^1 \mu\Bsb{[0,p]\times\bsb{ \max\{\lambda,p\},1 }} \dif \lambda - \int_0^1 \mu\bsb{[\lambda,p]\times [p,1]}\dif \lambda.
\]
As a consequence, for any $\e_t>0$, in the event $\bcb{ \forall R \in \mathcal{R}, \  \babs{ \psb[R] - \mathbb{P}_{(\st,\bt) \mid O_t}[R] } \le \e_t }$ (i.e., if the distribution of $(\st,\bt)$ given $O_t$ is $\e_t$-close to that of $(S_t,B_t)$, uniformly over rectangles), we have that
\begin{equation}
    \label{eq:first_decomposition}
    \max_{q \in [0,1]} \Babs{ \E\bsb{\gft\nolimits_t(q)} - \E\bsb{\gft'\nolimits_t(q)} } \le 2 \e_t \;.
\end{equation}
\end{lemma}
\begin{proof}
Consider $\mu$ and $p\in[0,1]$ as in the statement, we have
\begin{align*}
    \E_{(s,b)\sim \mu} \bsb{ \GFT(p,s,b) } 
&
= 
    \int_{[0,1]^2}(b-s)\ind\{s \le p \le b\}\dif\mu(s,b)
=
    \int_{[0,p]\times[p,1]} (b-s)\dif\mu(s,b)
\\
&
= 
    \int_{[0,p]\times[p,1]} \left(\int_0^b\dif\lambda - \int_0^s\dif\lambda \right)\dif\mu(s,b)
\\
&
=
    \int_0^1 \int_{[0,p]\times[p,1]} \ind\{\lambda\le b\} \dif\mu(s,b) \dif\lambda
    - \int_0^1 \int_{[0,p]\times[p,1]} \ind\{\lambda\le s\} \dif\mu(s,b) \dif\lambda 
\\
&
= 
    \int_0^1 \mu\bsb{[0,p]\times[\max\{\lambda,p\},1]}\dif\lambda 
     - \int_0^1 \mu\bsb{[\lambda,p]\times [p,1]}\dif\lambda \;.
\end{align*}
The consequence follows immediately from the decomposition and the fact that the subsets considered, i.e., $[0,p]\times\bsb{ \max\{\lambda,p\},1 }$ and $[\lambda,p]\times [p,1]$ are indeed rectangles for all choices of $p$ and $\lambda$, over which the two measures coincide up to an $\e_t$ additive factor. 
\end{proof}
We can now prove the regret guarantees of FBP that we claimed at the beginning of the section.
\begin{theorem}
\label{thm:upper_full}
    In the full-feedback stochastic (iid) setting, the regret of Follow the Best Price  satisfies, for all $T\in \N$ \[
        R_T(\emph{FBP}) \le C\sqrt{T \log T}
        \;,
        \qquad \text{ where }
        C \le 90 \;.
    \]
\end{theorem}
\begin{proof}
    For any time $t$, let $\e_t:=17\sqrt{{\ln(16T)}/{t}}$, and $G_t$ be the event that the distribution of $(\st,\bt)$ given $O_t$ is $\e_t$-close to that of $(S_t,B_t)$, uniformly over rectangles, i.e.,
    \[
        G_t
    :=
        \bcb{ \forall R \in \mathcal{R}, \  \babs{ \psb[R] - \mathbb{P}_{(\st,\bt) \mid O_t}[R] } \le \e_t } \;.
    \]
    Since the VC-dimension of $\mathcal{R}$ is $4$, we have that $\P[G_t^c] \le \e_t$. 
    This is an immediate consequence of VC-theory (see, e.g., \cite[Theorem 14.15]{mitzenmacher2017probability}).
    Putting this together with \cref{eq:triangular,eq:first_decomposition}, we conclude that for all $T\in \N$, the regret of the Follow the Best Price algorithm satisfies
\begin{align*}
&
    \exp{\sum_{t=1}^T\gft\nolimits_t(\ps)-\sum_{t=1}^T\gft\nolimits_t(P_t)} 
=
    \sum_{t=1}^T\E\bsb{\E[\gft\nolimits_t(p^{\star})]-\E[\gft\nolimits_t(P_t) \mid O_t]}
\\
&
\le
    \sum_{t=1}^T\E\bsb{\brb{\E[\gft\nolimits_t(p^{\star})]-\E[\gft\nolimits_t(P_t) \mid O_t]}\I_{G_t}} 
    + \sum_{t=1}^T \P[G_t^c]
    \le \sum_{t=1}^T 2\cdot(2\e_t) + \sum_{t=1}^T \e_t \le 90 \sqrt{T\ln T}.
\end{align*}
This concludes the proof.
\end{proof}

\subsection{\texorpdfstring{$\sqrt{T}$}{sqrt(T)} Lower Bound (iv+bd) \tccheck}
\label{sec:lower_bound_indep}

In this section, we show that the upper bound on the minimax regret we proved in \cref{sec:DKW} is tight, up to logarithmic factors.
No strategy can beat the $T^{1/2}$ rate when the seller/buyer pair $(S_t,B_t)$ is drawn i.i.d. from an unknown fixed distribution, even under the further assumptions that the valuations of the seller and buyer are independent of each other and have bounded densities.
For a full proof of the following theorem, see \cref{s:lower-full}.

\begin{theorem}
\label{thm:lower-full}
In the full-feedback 
model, for all horizons $T$, the minimax regret $\Rs_T$ satisfies 
\[
    \Rs_T 
:=
    \inf_{\alpha} \sup_{(S,B) \sim \cD} R_T(\alpha)
\ge
    c \sqrt{T} \;,
\]
where $c \ge {1}/{160}$, the infimum is over all of the learner's strategies $\alpha$, and the supremum is over all distributions $\cD$ of the seller $S$ and buyer $B$ such that:
\begin{itemize}
    \item[\emph{(iid)}] $(S_1,B_1),(S_2,B_2),\ldots \sim (S,B)$ is an i.i.d.\ sequence;
    \item[\emph{(iv)}] $S$ and $B$ are independent of each other;
    \item[\emph{(bd)}] $S$ and $B$ admit densities bounded by $M\ge4$.
\end{itemize}
\end{theorem}
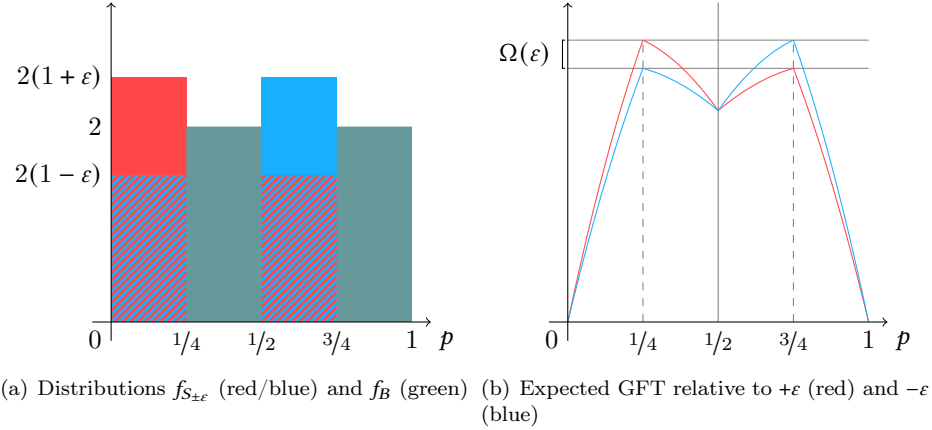
\begin{figure}
    \centering
    \subfigure[\label{f:sqrt-lower-a}Distributions $f_{S_{\pm \e}}$ (red/blue) and $f_B$ (green)]
    {
    \begin{tikzpicture}
    \def\k{0.5}
    \def\stretchY{0.65}
    \definecolor{myblue}{RGB}{25,175,255}
    \definecolor{myred}{RGB}{255,70,75}
    \definecolor{mygreen}{RGB}{64,128,128}
    % plots
    \fill[myred] (0, 0) rectangle ({2*\k}, {10*\k*\stretchY});
        \begin{scope}
        \clip (0,0) rectangle ({2*\k}, {6*\k*\stretchY});
        \foreach \x in {-20,...,20}
        {
            \draw[myblue,line width=1pt] ({-0.1-0.1*\x}, {0}) -- ({3.9-0.1*\x}, {4});
        }
        \end{scope}
    \fill[myblue] ({4*\k},0) rectangle ({6*\k}, {10*\k*\stretchY});
        \begin{scope}
        \clip ({4*\k}, 0) rectangle ({6*\k}, {6*\k*\stretchY});
        \foreach \x in {-40,...,20}
        {
            \draw[myred,line width=1pt] ({-0.1-0.1*\x}, {0}) -- ({5.9-0.1*\x}, {6});
        }
        \end{scope}
    \fill[mygreen!80] 
        ({2*\k}, 0) rectangle ({4*\k}, {8*\k*\stretchY})
        ({6*\k}, 0) rectangle ({8*\k}, {8*\k*\stretchY})
    ;
    % axes
    \draw[->] ({-0.5*\k}, {0*\k}) -- ({8.5*\k}, {0*\k}) node[below right] {$p$};
    \draw[->] ({0*\k}, -{0.5*\k}) -- ({0*\k}, {8.5*\k}) ;%node[above] {$f_{S,\pm\e}(p),f_B(p)$};
    % labels
    \draw (0,0) node[below left] {$0$}
        ({2*\k},0) node[below] {$\nicefrac{1}{4}$}
        ({4*\k},0) node[below] {$\nicefrac{1}{2}$}
        ({6*\k},0) node[below] {$\nicefrac{3}{4}$}
        ({8*\k},0) node[below] {$1$}
        (0, {10*\k*\stretchY}) node[left] {$2(1+\e)$}
        (0, {8*\k*\stretchY}) node[left] {$2$}
        (0, {6*\k*\stretchY}) node[left] {$2(1-\e)$}
        ;
    % % legend
    % \draw[fill = myred] ({8*\k}, {7.5*\k}) rectangle ({8.5*\k}, {8*\k});
    %     \draw ({8*\k}, {7.75*\k}) node[left] {$+\e=$};
    % \draw[fill = myblue] ({8*\k}, {6.5*\k}) rectangle ({8.5*\k}, {7*\k});
    %     \draw ({8*\k}, {6.75*\k}) node[left] {$-\e=$};
\end{tikzpicture}
}
    % \end{subfigure}
\subfigure[\label{f:sqrt-lower-b}Expected $\gft$ relative to $+\e$ (red) and $-\e$ (blue)]
{
    % \begin{subfigure}
    \begin{tikzpicture}[
    declare function={
        func(\x)
    = 
        (\x < 0) * (0)
        + and(\x >= 0, \x < 1/4) * ( (5/16) * \x * ( 5 - 4*\x ) )
        + and(\x >= 1/4, \x < 1/2) * ( (5/32) * ( -4*\x*\x + \x + 2 ) )
        + and(\x >= 1/2, \x < 3/4) * ( \x * ( 21/32 - (3/8)*\x ) )
        + and(\x >= 3/4, \x < 1) * ( (1/8) * ( 1 - \x ) * ( 8*\x + 3 ) )
        + (\x >= 1) * (0)
       ;
      }
    ]
    \def\k{0.5}
    \def\stretchY{3}
    \definecolor{myblue}{RGB}{25,175,255}
    \definecolor{myred}{RGB}{255,70,75}
    \def\colorOne{myblue}
    \def\colorTwo{myred}
    % dashed lines
    \draw[gray, dashed] ({\k*2}, 0) -- ({\k*2}, {(5/16)*\stretchY*\k*8});
    \draw[gray, dashed] ({\k*6}, 0) -- ({\k*6}, {(5/16)*\stretchY*\k*8});
    % vertical line
    \draw[gray, thin] ({\k*4}, 0) -- ({\k*4}, {\k*8.5});
    % horizontal lines
    \draw[gray, very thin] ({\k*0}, {(5/16)*\stretchY*\k*8}) -- ({\k*8}, {(5/16)*\stretchY*\k*8});
    \draw[gray, very thin] ({\k*0}, {(9/32)*\stretchY*\k*8}) -- ({\k*8}, {(9/32)*\stretchY*\k*8});
    % plots
    \draw[domain = 0:{\k*8}, myred, samples = 1024] plot (\x, {\stretchY*\k*8*func( \x/(\k*8) )});
    \draw[domain = 0:{\k*8}, myblue, samples = 1024] plot (\x, {\stretchY*\k*8*func( 1-\x/(\k*8) )});
    % axes
    \draw[->] ({-0.5*\k}, {0*\k}) -- ({8.5*\k}, {0*\k}) node[below right] {$p$};
    \draw[->] ({0*\k}, -{0.5*\k}) -- ({0*\k}, {8.5*\k}) ;%node[above] {$\E_{\P^{\pm \e}} \bsb{ \mathrm{GFT}_t(p) }$};
    % labels
    \draw (-1pt, {(5/16)*\stretchY*\k*8}) 
        -- (-2pt, {(5/16)*\stretchY*\k*8}) 
        -- (-2pt, {(9/32)*\stretchY*\k*8})
        -- (-1pt, {(9/32)*\stretchY*\k*8})
    ;
        \draw (-2pt, {((9/32 + 5/16)/2)*\stretchY*\k*8}) node[left, xshift=-1pt] {$\Omega( \e )$};
    \draw (0,0) node[below left] {$0$}
        ({2*\k},0) node[below] {$\nicefrac{1}{4}$}
        ({4*\k},0) node[below] {$\nicefrac{1}{2}$}
        ({6*\k},0) node[below] {$\nicefrac{3}{4}$}
        ({8*\k},0) node[below] {$1$}
        ;
\end{tikzpicture}
}
    % \end{subfigure}
    \caption{The best posted price is $\nicefrac{1}{4}$ (resp., $\nicefrac{3}{4}$) in the $+\e$ (resp., $-\e$) case.
    By posting $\nicefrac{1}{4}$, the player suffers a $\Omega( \e )$ regret in the $-\e$ case, and the same is true posting $\nicefrac{3}{4}$ if in $+\e$ case.
    }
    \label{f:root-t-full}
\end{figure}
\begin{proof}[Prook sketch]
We build a family of distributions $\cD_{\pm \e}$ of the seller and buyer $(S,B)$ parameterized by $\e\in[0,1]$.
For the seller, for any $\e \in [0,1]$, we define the density
\[
    f_{S,\pm\e} 
:= 
    2(1 \pm\e)\ind_{\lsb{0, \frac{1}{4}}} + 2(1\mp\e)\ind_{\lsb{\frac{1}{2},\frac{3}{4}}}\;. 
    \tag{\text{\cref{f:sqrt-lower-a}, in red/blue}}
\]
For the buyer, we define a single density (independently of $\e$)
\[
    f_B
:=
    2 \I_{\lsb{ \frac{1}{4}, \frac{1}{2} } \cup \lsb{ \frac{3}{4}, 1 }} \;.
    \tag{\text{\cref{f:sqrt-lower-a}, in green}}
\]
In the $+\e$ (resp., $-\e$) case, the optimal price belongs to the region $[0,\nicefrac{1}{2}]$ (resp., $(\nicefrac{1}{2}, 1]$, see \cref{f:sqrt-lower-b}).
By posting prices in the wrong region $(\nicefrac{1}{2}, 1]$ (resp., $[0,\nicefrac{1}{2}]$) in the $+\e$ (resp., $-\e$) case, the learner incurs a $\Omega(\e)$ regret.
Thus, if $\e$ is bounded-away from zero, the only way to avoid suffering linear regret is to identify the sign of $\pm\e$ and play accordingly.

This closely resembles the learning dilemma present in two-armed bandits.
In fact, a technical proof (see \cref{s:lower-full}), shows that our setting is harder (i.e., it has a higher minimax regret) than an instance of a stochastic two-armed bandit problem, which has a known lower bound on its minimax regret of $\frac{1}{8} \brb{ \frac{1}{20} \sqrt{T} }$ \cite{Nicolo06,BubeckC12}.
\end{proof}

\section{Realistic-Feedback Stochastic (iid) Setting \tccheck}
\label{s:real-feddb}

In this section, we tackle the problem in the more challenging realistic-feedback model, again under the assumption that the seller/buyer pairs $(S_1,B_1),(S_2,B_2),\ldots$ are $[0,1]^2$-valued i.i.d.\ random variables, all distributed as a common $(S,B)$.
We will first study the case in which $S$ and $B$ are independent (iv) and have bounded densities (bd), then discuss what happens if either of the two assumptions is lifted.

We recall that in the realistic-feedback model, the only information collected by the mechanism at the end of each round $t$ consists of $\ind\{S_t \le P_t\}$ and $\ind\{P_t \le B_t\}$. 
The main tool we use to leverage the structure of the objective function is the following decomposition Lemma. The first part of this result follows directly from specialising \Cref{eq:first_decomposition} to independent distributions; alternatively, it can be directly derived, as in \cite{MyersonS83}. 
\begin{lemma}[Second decomposition Lemma]
\label{lem:second_decomposition}
Let $S$ and $B$ be independent random variables in $[0,1]$, then for all prices $p \in [0,1]$ it holds
\begin{equation}
    \label{eq:decompostion}
    \E\bsb{ \GFT(p, S, B) }
=
    \P[S \le p]\int_p^1 \P[B \ge \lambda] \dif \lambda + \P[B\ge p]\int_0^p \P[S \le \lambda] \dif \lambda \;. 
\end{equation}
Moreover, if $S$ and $B$ admit densities bounded above by $M\ge 1$, then $\exp{\gft(\cdot)}$ is Lipschitz in the prices, with constant $4M$.
\end{lemma}
\begin{proof}
    We start form \Cref{eq:first_decomposition} and use the independence of the distributions:
    \begin{align*}
        \E\bsb{\gft(p,S,B)} &=
    \int_0^1 \P_{(S,B)} \bsb{ [0,p]\times[\max\{\lambda,p\},1] } \dif \lambda - \int_0^1 \P_{(S,B)} \bsb{ [\lambda,p]\times [p,1] } \dif \lambda\\
    &=\P[S \le p]\int_0^1 \P\bsb{B \ge \max\{\lambda,p\}} \dif \lambda - \P[B \ge p]\int_0^1 \P[ \lambda\le S \le p]\dif \lambda\\
    &=\P[S \le p]\int_p^1 \P[B \ge \lambda] \dif \lambda + \P[B\ge p]\int_0^p \P[S \le \lambda] \dif \lambda \;.
    \end{align*}
Now, we just need to address the Lipschitzness given that $S$ and $B$ admit densities bounded from above by some $M\ge 1$.
Note that assumption implies that the cumulative distribution functions (CDFs) of $S$ and $B$, denoted by $F_S$ and $F_B$, are $M$-\lip{}. 
Let $0 \le p < q \le 1$, then 
\begin{multline*}
    \babs{ \gft(q,S,B) - \gft(p,S,B) } = \Biggl\lvert F_S(p) \int_{p}^q\brb{ 1-F_B(\lambda)}\dif\lambda - \brb{ 1-F_B(q) } \int_{p }^q F_S(\lambda) \dif \lambda + \\
\begin{aligned}
    &  \qquad \brb{ F_B(q)-F_B(p) } \int_0^p \brb{ 1-F_B(\lambda) } \dif\lambda - \brb{ F_S(q)-F_S(p) } \int_q^1 F_S(\lambda) \dif\lambda \Biggr\rvert \\
    &\le M \left[2\int_{p}^q \dif\lambda + 2(q-p)\right]\le 4M \labs{q-p} \;,
\end{aligned}
\end{multline*}
where in the inequality we used that $F_S$ and $F_B$ are upper bounded by $1$ and that they are both $M$-\lip{}.
\end{proof}
The previous result is also important since it relates the regularity of the distributions, specifically the boundedness of the densities, to the regularity of the objective function, i.e., its \lip{}ness.

\subsection{Scouting Bandits (SB): from Realistic Feedback to Multi-Armed Bandits \tccheck}
\label{sec:lipschitz}
    The main challenge in designing a low-regret algorithm for this problem lies in the fact that posting a price does not reveal the corresponding gain from trade.
    This uncertainty then translates to a slow learning of the function $\E \bsb{\gft(\cdot, S,B)}$ over the interval of possible prices.

    This can be overcome by sampling. The structure of the gain from trade, however, is such that this sampling needs to be spread over the entire unit interval, i.e., to estimate $\E\bsb{\gft(p, S, B)}$ for a given price $p$ it is \emph{not} sufficient to simply post $p$ repeatedly. \cref{eq:decompostion} helps visualizing this phenomenon.
    While the local terms, i.e., $\P[S \le p]$ and $\P[B \ge p]$, can be reconstructed by multiple posting of $p$, the integral terms exhibit a global behaviour: they depend on what happens in  $(p,1]$ or $[0,p)$, and hence need prices to be posted in those regions to be estimated accurately.
    This rules out direct applications of well established algorithms, like action elimination or UCB \cite{Slivkins19}, which crucially depend on the locality of the exploration. Similarly, estimating naively the CDFs on a grid of prices and using this information to reconstruct both the global and the local terms falls short of yielding the desired $T^{2/3}$ regret bound.
    
    \begin{algorithm}
         \textbf{input:} bandit algorithm $\alpha$, upper bound on the densities $M$, time horizon $T$, and precision parameter $\e$\;
         $\ell \gets M^{2/3}$, $\delta \gets \e M^{1/3}$, $K \gets \lce{\ell/\e}$, $T_0 \gets \bce{ \ln({4K}/{\delta})/(2\e^2)}$\;
         $q_i \gets i({\e}/{\ell})$, for all $i=0,1,2,\dots,K-1$\;
         \For{$t=1,2,\dots,T_0$ }{
         draw $P_t$ from $[0,1]$ uniformly at random\;
         post price $P_t$ and observe feedback $\brb{ \ind\{S_t \le P_t\}, \, \ind\{P_t \le B_t\} }$\;
         let $\hat I^t_i \gets \ind\{S_t \le P_t\le q_i\}$, and 
        $\hat J^t_i \gets \hat \ind\{q_i \le P_t\le B_t\}$,  for all $i=0,1,\dots, K-1$
         }
         let $\hat I_i \gets \frac{1}{T_0}\sum_{t=1}^{T_0} \hat I_i^t$ and $\hat J_i \gets \frac{1}{T_0}\sum_{t=1}^{T_0} \hat J_i^t$, for all $i=0,1,\dots, K-1$\;
         initialize the bandit algorithm $\alpha$ with the number of arms $K$ and the horizon $T-T_0$\;
         \For{$t=T_0+1,\dots, T$}{
         receive an arm $i_t$ from $\alpha$\;
         post price $P_t \gets q_{i_t}$ and observe $\lrb{ \ind\{S_t \le P_t\}, \, \ind\{P_t \le B_t\}}$\;
         feed to $\alpha$ reward $r_t = \ind\{S_t \le P_t\} \hat{J}_{i_t} + \ind\{P_t \le B_t\} \hat{I}_{i_t}$
         }
          \caption{Scouting Bandits}
          \label{alg:meta}
    \end{algorithm}
   
    Our \cref{alg:meta} (Scouting Bandits) consists of exploiting the decomposition in \cref{eq:decompostion} and learns separately the global and local part of the gain from trade.
    First, a global exploration phase is run (Scouting), in which prices uniformly sampled in $[0,1]$ are posted and used to simultaneously estimate all the integral terms on a suitable grid. 
    Once this is done, we can run any bandit algorithm (Bandits) on the prices of the grid, complementing the realistic feedback received with the estimated integrals. 
    We use the assumption on the independence of $S$ and $B$ (iv) to apply \Cref{lem:second_decomposition} and bounded densities (bd) to have Lipschitzness of the expected gain from trade. 
    Later, we show how dropping either of these assumptions leads to linear regret (\Cref{thm:lower-real-iv,thm:lower-real-bd}).

    \begin{theorem}
    \label{thm:upper_real}
    In the realistic-feedback stochastic (iid) setting where the distributions of the seller and buyer are independent (iv) and have densities bounded by $M \ge 1$, the regret of Scouting Bandits (SB) run with a bandit algorithm $\alpha$, upper bound on the densities $M\ge 1$, time horizon $T$, and parameter $\e>0$ satisfies
    \[
        R_T(\emph{SB}) 
    =
        \cO \lrb{
            \frac{1}{\e^2}\ln{\frac{M}{\e}} + \e M^{1/3} T + \mathcal{R} \lrb{ \lce{ \frac{M^{2/3}}{\e} },T }
        } \;,
    \]
    where $\mathcal{R}(\kappa,\tau)$ is any monotone bound (in $\tau$) on the regret of $\alpha$ on $\kappa$ arms for time horizon $\tau$.
    In particular, if $\alpha$ is either the Action Elimination or UCB algorithm \cite{Slivkins19}, the resulting regret in $\cO\brb{ M^{1/3}T^{2/3}\ln( MT) }.$
    \end{theorem}
    We prove the result in two steps, first we show how the global exploration phase is indeed enough to build good estimates of the integrals for a suitable grid of prices, then we analyze the rest of the algorithm, conditioning on the event that such estimates are accurate.
    \begin{lemma}
    \label{lem:estimating_integrals}
        Fix any precision $\e>0$, probability $\delta>0$, regularity term $\ell>0$.
        Let $K =\lce{
        {\ell}/{\e}}$ and fix the grid of prices $q_i = i ({\e}/{\ell})$, for all $i=0,1,\dots, K-1$. 
        Moreover, for all $i = 0,1,\dots,K-1$,
        \[
            I_i
        =
            \int_0^{q_i} \P[S \le \lambda] \dif\lambda\;, 
            \qquad 
            J_i=\int_{q_i}^1 \P[B \ge \lambda] \dif\lambda \;.
        \]
        Consider the estimators $\hat I_i,\hat J_i$ determined at the end of the exploration phase of \Cref{alg:meta}, i.e., $\hat I_i = \frac{1}{T_0}\sum_{t=1}^{T_0} \hat I_i^t,$ and $\hat J_i =\frac{1}{T_0}\sum_{t=1}^{T_0} \hat J_i^t$. If $T_0 \ge\frac{1}{2\e^2}\ln\frac{4K}{\delta}$, it holds that \[
        \max_{i=0,1,\dots, K-1}\bcb{ \lvert I_i-\hat I_i\rvert,\, \lvert J_i-\hat J_i\rvert}<\e,
        \]
        with probability at least $1-\delta$, where the probability is with respect to $(S_1,B_1),\dots,(S_{T_0},B_{T_0})$.
    \end{lemma}
\begin{proof}
    We first show that the $\hat{I}_i^t, \hat{J}_i^t$ are unbiased estimators of $I_i,J_i$.
    For all $t \le T_0$, conditioning on $P_t \le q_i$, $P_t$ follows a uniform distribution in $[0,q_i]$, hence, calling $U_i$ a random variable uniformly distributed in that interval, one has:
\begin{align*}
    \E\bsb{ \hat I^t_i } &=  \P[S_t \le P_t \le q_i] = \P[P_t \le q_i]\P[S_t \le P_t \le q_i|P_t \le q_i]  \\
    &= \P[P_t \le q_i]\P[S \le U_i] = \int_{0}^{q_i} \P[S \le \lambda] \dif \lambda = I_i \;.
\end{align*}
A similar argument, conditioning on $P_t \ge q_i$, gives that $\E \bsb{\hat J_i} = J_i$.

    For all $i$, let $E_i=\bcb{\lvert\hat I_i - I_i\rvert> \e}$ and $F_i=\bcb{\lvert\hat J_i - J_i\rvert> \e}$ be the events in which there is an error greater that $\e$ in the estimates. 
    By the Chernoff-Hoeffding inequality one has that the probabilities of each event is upper bounded by $2e^{-2\e^2T_0}$. 
    Let $\mathcal{E}$ be the good event corresponding to all the integrals being estimated within an $\e$ accuracy. Clearly $\mathcal{E}$ is the complement of $\bigcup_{i = 0}^{K-1}(E_i \cup F_i)$. 
    Hence, we have that:
\[
    \P[\mathcal{E}^c] \le \sum_{i=0}^{K-1} \P[E_i]+ \sum_{i=0}^{K-1} \P[F_i]  \le  4K e^{-2\e^2T_0} \le \delta \;,
\]
where the last inequality has been obtained by simply plugging in $T_0\ge\frac{1}{2\e^2}\ln{\frac{4K}{\delta}}$.
\end{proof}
We are  now ready to prove the main result of this section.
After an initial global exploration phase, any optimal multi-armed bandit algorithm gives a regret for our bilateral trade problem that is optimal, up to logarithmic terms (by \cref{thm:lower-real-iv+bd}). 
\begin{proof}[Proof of \Cref{thm:upper_real}]
    Let $\e>0$ be a precision parameter we set later, then consider the result of \Cref{lem:estimating_integrals} on an initial exploration phase with $\delta \le \e M^{1/3}$ and $\ell = M^{2/3}$.
    Recall that, for all $t$, $\GFT_t(p) := \gft(p, S_t,B_t)$.
    By the Lipschitzness of the gain from trade (with constant $4M$, as shown in \Cref{lem:second_decomposition}) and the fact that the grid is spaced by $\frac{\e}{M^{2/3}}$, we get a discretization error at each time step which can be bounded as follows
    \begin{multline*}
        \max_{p \in [0,1]} \E\bsb{\gft(p,S,B)} - \max_{i=0,1,\dots, K-1}\E\bsb{\gft(q_i,S,B)} 
    \\
    \le 
        \E \bsb{\gft(p^{\star},S,B)} - \max_{i =i(p^{\star}), i(p^{\star})+1} \E\bsb{\gft(q_i,S,B)}
    \le 
        4M \min_{i =i(p^{\star}), i(p^{\star})+1} |q_i-p^{\star}| \le 2\e M^{1/3}
    \end{multline*}
    where $q_{i(p^{\star})}$ is the largest element in the grid smaller or equal to $p^{\star}$ and $q_{i(p^{\star})+1}$ is set to $1$ if $q_i(p^{\star}) = K-1$.
    Now, we have
    \begin{align*}
    R_T(\text{SB}) &= \max_{p \in [0,1]}\sum_{t=1}^T \E\bsb{\gft\nolimits_t(p)-\gft\nolimits_t(P_t)} \le T_0 + \sum_{t=T_0+1}^T \E\bsb{\gft\nolimits_t(p^{\star})-\gft\nolimits_t(P_t)} \\
    &\le T_0 +  2\e M^{1/3} T+ \max_{i=0,\dots, K-1}\exp{\sum_{t=T_0+1}^T\gft\nolimits_t(q_{i})-\sum_{t=T_0+1}^T\gft\nolimits_t(P_t)}
    \end{align*}
    Let $i^\star \in \argmax_{i=0,\ldots,K-1} \E \bsb{ \sum_{t=T_0+1}^T \gft\nolimits_t(q_i) }$.
    Let also $\mathcal{E}$ be the same ``good'' event as in the proof of \Cref{lem:estimating_integrals} for our choice of parameters.
    Conditioning with respect to it, we get
    \begin{align}
    \nonumber
    R_T(\text{SB}) &\le T_0 + 2\e M^{1/3} T + T\P\bsb{ \mathcal{E}^c } + \sum_{t=T_0+1}^T \E\bsb{\gft\nolimits_t(q_{i^{\star}})-\gft\nolimits_t(P_t)|\mathcal{E}}\\
      \label{eq:clean_event} 
     &\le T_0 +  2\e M^{1/3} T + T \delta + \sum_{t=T_0+1}^T \E\bsb{\gft\nolimits_t(q_{i^{\star}})-\gft\nolimits_t(P_t)|\mathcal{E}} \;.
    \end{align}
    We now focus on the last term.
    Conditioning on $\mathcal{E}$, for all $t > T_0$ and all $i \in \{0,1,\dots,K-1\}$ we have that the expected gain of posting price $q_i$ is $\e$-near to the expected reward for the multi-armed bandit instance associated. In fact on the one hand, by \Cref{lem:second_decomposition} and the fact that what $(S_t,B_t)$ is independent from $\mathcal{E}$
    \[
        \exp{\gft\nolimits_t(q_i)|\mathcal{E}}=\exp{\gft\nolimits_t(q_i)} = \P[S \le q_i]J_i + \P[q_i \le B]I_i.
    \]
    On the other hand, defining $r_t(i) := \ind\{S_t \le q_i\} \hat{J}_{i} + \ind\{q_i \le B_t\} \hat{I}_{i}$ and conditioning on $\mathcal{E}$, we get
    \[
        \exp{r_t(i)|\mathcal{E}} = \P[S_t \le q_i]\E\bsb{\hat{J_i}|\mathcal{E}} + \P [q_i \le B_t]\E\bsb{\hat{I_i}|\mathcal{E}}.
    \]
    Putting those two formulae together, we have the claimed inequality
    \begin{align*}
        \Babs{ \E\bsb{\gft\nolimits_t(q_i)-r_t(i)|\mathcal{E}} } 
    \le 
        \P[S_t \le q_i] \Babs{ J_i-\E\bsb{\hat{J_i}|\mathcal{E}} } + \P[q_i \le B_t] \Babs{I_i -\E \bsb{\hat{I_i}|\mathcal{E}} } \le 2\e \;.
    \end{align*}
    Plugging in this result, we get
    \[
        R_T(\text{SB}) \le T_0 + T \delta + 2\e M^{1/3} T + 4 \e T + \sum_{t=T_0+1}^T \E\bsb{r_t(i^\star)-r_t(i_t) \mid \mathcal{E}} \;.
    \]
    We now focus on the last term.
    Note that, for any arm $i$, the sequence $r_{T_0+1}(i),r_{T_0+2}(i), \ldots \in [0,2]$ is a $\P[\cdot \mid \cE]$-i.i.d. sequence of random variables.
    Thus, we can exploit the worst-case regret guarantees of $\alpha$ and the monotonicity of the regret, obtaining
    \[
        \sum_{t=T_0+1}^T \E\bsb{r_t(i^\star)-r_t(i_t) \mid \mathcal{E}}
    \le 
        \mathcal{R}(K,T-T_0)
    \;.
    \]
    Putting everything together, gives the first part of the result.
    
    For the second part, pick $\e= T^{-1/3}$ and consider any algorithm (e.g, Action Elimination or UCB) with $\mathcal{R}(\kappa, \tau) = \cO \brb{ \sqrt{\kappa \tau \log \tau } }$.
\end{proof}
    Note that Scouting Bandits needs to know in advance the time horizon $T$ to set the length of the initial exploration phase and also to pass that information to the multi-armed bandit algorithm embedded, if needed. This dependence, however, can be lifted with a standard doubling trick \cite{Nicolo06}.

\subsection{\texorpdfstring{$T^{2/3}$}{T\^{}(2/3)} Lower Bound Under Realistic Feedback (iv+bd) \tccheck}
\label{sec:candidate}

In this section, we show that the upper bound on the minimax regret we proved in \cref{sec:lipschitz} is tight, up to logarithmic factors.
No strategy can beat the $T^{2/3}$ rate when the seller/buyer pair $(S_t,B_t)$ is drawn i.i.d. from an unknown fixed distribution, even under the further assumptions that the valuations of the seller and buyer are independent of each other and have bounded densities.
For a full proof of the following theorem, see \cref{s:proof-t-two-thrid-lower-bound-appe}.
\begin{theorem}
\label{thm:lower-real-iv+bd}
In the realistic-feedback model, for all horizons $T$, the minimax regret $\Rs_T$ satisfies 
\[
    \Rs_T 
:=
    \inf_{\alpha} \sup_{(S,B) \sim \cD} R_T(\alpha)
\ge
    c T^{2/3} \;,
\]
where $c \ge 11/672$, the infimum is over all of the learner's strategies $\alpha$, and the supremum is over all distributions $\cD$ of the seller $S$ and buyer $B$ such that:
\begin{itemize}
    \item[\emph{(iid)}] $(S_1,B_1),(S_2,B_2),\ldots \sim (S,B)$ is an i.i.d.\ sequence;
    \item[\emph{(iv)}] $S$ and $B$ are independent of each other;
    \item[\emph{(bd)}] $S$ and $B$ admit densities bounded by $M\ge24$.
\end{itemize}
\end{theorem}
\begin{proof}[Proof sketch]

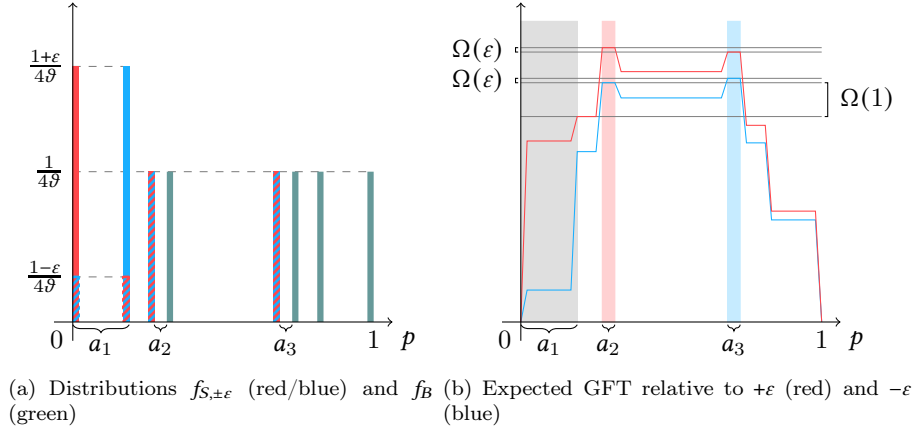
\begin{figure}
    \centering
    \subfigure[\label{f:t-two-third-lower-bound-a}Distributions $f_{S,\pm \e}$ (red/blue) and $f_B$ (green)]
    {
    \begin{tikzpicture}
    \def\k{0.5}
    \def\stretchY{0.5}
    \definecolor{myblue}{RGB}{25,175,255}
    \definecolor{myred}{RGB}{255,70,75}
    \definecolor{mygreen}{RGB}{64,128,128}
    % dashed lines
    \draw[gray, dashed] (0, {8*\k*\stretchY*(1+0.7)}) -- ({8*\k/6}, {8*\k*\stretchY*(1+0.7)});
    \draw[gray, dashed] (0, {8*\k*\stretchY*(1-0.7)}) -- ({8*\k/6}, {8*\k*\stretchY*(1-0.7)});
    \draw[gray, dashed] (0, {8*\k*\stretchY)}) -- ({8*\k}, {8*\k*\stretchY});
    % plots
    \fill[myred] (0, 0) rectangle ({8*\k/48}, {8*\k*\stretchY*(1+0.7)});
        \begin{scope}
        \clip (0,0) rectangle ({8*\k/48}, {8*\k*\stretchY*(1-0.7)});
        \foreach \x in {-2,...,10}
        {
            \draw[myblue,line width=1pt] ({-0.1-0.1*\x}, {0}) -- ({1.9-0.1*\x}, {2});
        }
        \end{scope}
        \draw[myblue, line width=1pt] 
        (0, {8*\k*\stretchY*(1-0.7)}) 
            -- ({8*\k/48}, {8*\k*\stretchY*(1-0.7)})
        ;
    \fill[myblue] ({8*\k/6}, 0) rectangle ({8*\k*3/16}, {8*\k*\stretchY*(1+0.7)});
        \begin{scope}
        \clip ({8*\k/6}, 0)  rectangle ({8*\k*3/16}, {8*\k*\stretchY*(1-0.7)});
        \foreach \x in {-2,...,10}
        {
            \draw[myred,line width=1pt] ({8*\k/6-0.1-0.1*\x}, {0}) -- ({8*\k/6+1.9-0.1*\x}, {2});
        }
        \end{scope}
        \draw[myred, line width=1pt] 
        ({8*\k/6}, {8*\k*\stretchY*(1-0.7)}) 
            -- ({8*\k*3/16}, {8*\k*\stretchY*(1-0.7)})
        ;
    \fill[myred] ({8*\k/4}, 0) rectangle ({8*\k*13/48}, {8*\k*\stretchY});
        \begin{scope}
        \clip ({8*\k/4}, 0) rectangle ({8*\k*13/48}, {8*\k*\stretchY});
        \foreach \x in {-2,...,20}
        {
            \draw[myblue,line width=1pt] ({8*\k/4-0.1-0.1*\x}, {0}) -- ({8*\k/4+1.9-0.1*\x}, {2});
        }
        \end{scope}
    \fill[myred] ({8*\k*2/3}, 0) rectangle ({8*\k*11/16}, {8*\k*\stretchY});
        \begin{scope}
        \clip ({8*\k*2/3}, 0) rectangle ({8*\k*11/16}, {8*\k*\stretchY});
        \foreach \x in {-2,...,20}
        {
            \draw[myblue,line width=1pt] ({8*\k*2/3-0.1-0.1*\x}, {0}) -- ({8*\k*2/3+1.9-0.1*\x}, {2});
        }
        \end{scope}
    \fill[mygreen!80] 
        ({8*\k*5/16}, 0) rectangle ({8*\k*1/3}, {8*\k*\stretchY})
        ({8*\k*35/48}, 0) rectangle ({8*\k*3/4}, {8*\k*\stretchY})
        ({8*\k*13/16}, 0) rectangle ({8*\k*5/6}, {8*\k*\stretchY})
        ({8*\k*47/48}, 0) rectangle ({8*\k}, {8*\k*\stretchY})
    ;
    % axes
    \draw[->] ({-0.5*\k}, {0*\k}) -- ({8.5*\k}, {0*\k}) node[below right] {$p$};
    \draw[->] ({0*\k}, -{0.5*\k}) -- ({0*\k}, {8.5*\k}) ;%node[above] {$f_{S,\pm\e}(p),f_B(p)$};
    % labels
    \draw [decorate,decoration={brace, amplitude=3pt, mirror}, yshift=-1pt]
        (0,0) -- ({\k*8*3/16},0) node [black, midway, below, yshift=-2pt] {$a_1$};
    \draw [decorate,decoration={brace, amplitude=2.4pt, mirror}, yshift=-1pt]
        ({\k*8*13/48},0) -- ({\k*8*5/16},0) node [black, midway, below, yshift=-2pt] {$a_2$};
    \draw [decorate,decoration={brace, amplitude=2.4pt, mirror}, yshift=-1pt]
        ({\k*8*11/16},0) -- ({\k*8*35/48},0) node [black, midway, below, yshift=-2pt] {$a_3$};
    \draw (0,0) node[below left] {$0$}
        ({8*\k},0) node[below] {$1$}
        (0, {8*\k*\stretchY*(1+0.7)}) node[left] {$\frac{1+\e}{4 \tht}$}
        (0, {8*\k*\stretchY}) node[left] {$\frac{1}{4 \tht}$}
        (0, {8*\k*\stretchY*(1-0.7)}) node[left] {$\frac{1-\e}{4\tht}$}
        ;
    % legend
    % \draw[fill = myred] ({8*\k}, {7.5*\k}) rectangle ({8.5*\k}, {8*\k});
    %     \draw ({8*\k}, {7.75*\k}) node[left] {$+\e=$};
    % \draw[fill = myblue] ({8*\k}, {6.5*\k}) rectangle ({8.5*\k}, {7*\k});
    %     \draw ({8*\k}, {6.75*\k}) node[left] {$-\e=$};
\end{tikzpicture}
}
\subfigure[\label{f:t-two-third-lower-bound-b}Expected $\gft$ relative to $+\e$ (red) and $-\e$ (blue)]
{\begin{tikzpicture}[
    declare function={
        func(\x,\eps)
    = 
        (\x < 0) * (0)
        + and(\x >= 0, \x < 1/48) * ( ( 3 * (46 - 32 * \x) * \x * ( \eps + 1 ) )/16 )
        + and(\x >= 1/48, \x < 1/6) * ( (17/96)*(1+\eps) )
        + and(\x >= 1/6, \x < 3/16) * ( -(69/8)*\x*(-1+\eps) + 6*\x*\x*(-1+\eps) + (1/96)*(-105+139*\eps) )
        + and(\x >= 3/16, \x < 1/4) * ( \eps/24 + 5/16 )
        + and(\x >= 1/4, \x < 13/48) * ( \eps/24 - 47/32 + 69*\x/8 - 6*\x*\x )
        + and(\x >= 13/48, \x < 5/16) * ( \eps/24 + 41/96 )
        + and(\x >= 5/16, \x < 1/3) * ( ( -3456*\x^2 + (1032 - 384*\eps)*\x + 152*\eps + 343 )/768 )
        + and(\x >= 1/3, \x < 2/3) * ( \eps/32 + 101/256 )
        + and(\x > 2/3, \x <= 11/16) * ( \eps/32 + (1/768) * (-2081 + 5880* \x - 3456*\x^2) )
        + and(\x > 11/16, \x < 35/48) * ( \eps / 32 + 41/96)
        + and(\x >= 35/48, \x < 3/4) * ( 37/32 + 27/8 * \x - 6*\x*\x + 19/48 * \eps - \x*\eps/2)
        + and(\x >= 3/4, \x < 13/16) * ( (\eps + 15) / 48 )
        + and(\x >= 13/16, \x < 5/6) * ( 49/32 + 27*\x/8 - 6*\x*\x + 41/96 * \eps - \x * \eps / 2)
        + and(\x >= 5/6, \x < 47/48) * ( ( \eps + 17) / 96 )
        + and(\x >= 47/48, \x < 1) * ( -1/8 * (-1 + \x) * (21 + 48 * \x + 4 * \eps) )
        + (\x >= 1) * (0)
       ;
      }
    ]
    \def\k{0.5}
    \def\stretchY{2}
    \definecolor{myblue}{RGB}{25,175,255}
    \definecolor{myred}{RGB}{255,70,75}
    \def\colorOne{myblue}
    \def\colorTwo{myred}
    % color regions of actions a_1, a_2, a_3
    \draw[gray!25, fill = gray!25] (0,0) rectangle ({\k*8*3/16}, {8*\k});
    \draw[myred!25, fill = myred!25] ({\k*8*13/48},0) rectangle ({\k*8*5/16}, {8*\k});
    \draw[myblue!25, fill = myblue!25] ({\k*8*11/16},0) rectangle ({\k*8*35/48}, {8*\k});
    % horizontal lines
    \draw[gray, very thin] ({\k*0}, {\stretchY*\k*8*(0.7/24 + 41/96)}) -- ({\k*8}, {\stretchY*\k*8*(0.7/24 + 41/96)});
    \draw[gray, very thin] ({\k*0}, {\stretchY*\k*8*(0.7 / 32 + 41/96)}) -- ({\k*8}, {\stretchY*\k*8*(0.7 / 32 + 41/96)});
    \draw[gray, very thin] ({\k*0}, {\stretchY*\k*8*(-0.7/32 + 41/96)}) -- ({\k*8}, {\stretchY*\k*8*(-0.7/32 + 41/96)});
    \draw[gray, very thin] ({\k*0}, {\stretchY*\k*8*(-0.7/24 + 41/96)}) -- ({\k*8}, {\stretchY*\k*8*(-0.7/24 + 41/96)});
    \draw[gray, very thin] ({\k*0}, {8*\k*\stretchY*(0.7/24 + 5/16)}) -- ({\k*8}, {8*\k*\stretchY*(0.7/24 + 5/16)});
    % plots
    \draw[myblue] (0,0)
        -- ({8*\k/48}, {8*\k*\stretchY*(17/96)*(1-0.7)})
        -- ({8*\k/6}, {8*\k*\stretchY*(17/96)*(1-0.7)})
        -- ({8*\k*3/16}, {8*\k*\stretchY*(-0.7/24 + 5/16)})
        -- ({8*\k/4}, {8*\k*\stretchY*(-0.7/24 + 5/16)})
        -- ({8*\k*13/48}, {8*\k*\stretchY*( -0.7/24 + 41/96 )})
        -- ({8*\k*5/16}, {8*\k*\stretchY*( -0.7/24 + 41/96 )})
        -- ({8*\k*1/3}, {8*\k*\stretchY*( -0.7/32 + 101/256 )})
        -- ({8*\k*2/3}, {8*\k*\stretchY*( -0.7/32 + 101/256 )})
        -- ({8*\k*11/16}, {8*\k*\stretchY*( -0.7 / 32 + 41/96)})
        -- ({8*\k*35/48}, {8*\k*\stretchY*( -0.7 / 32 + 41/96)})
        -- ({8*\k*3/4}, {8*\k*\stretchY*( (-0.7 + 15) / 48 )})
        -- ({8*\k*13/16}, {8*\k*\stretchY*( (-0.7 + 15) / 48 )})
        -- ({8*\k*5/6}, {8*\k*\stretchY*( ( -0.7 + 17) / 96 )})
        -- ({8*\k*47/48}, {8*\k*\stretchY*( ( -0.7 + 17) / 96 )})
        -- ({8*\k}, {8*\k*\stretchY*0})
    ;
    \draw[myred] (0,0)
        -- ({8*\k/48}, {8*\k*\stretchY*(17/96)*(1+0.7)})
        -- ({8*\k/6}, {8*\k*\stretchY*(17/96)*(1+0.7)})
        -- ({8*\k*3/16}, {8*\k*\stretchY*(0.7/24 + 5/16)})
        -- ({8*\k/4}, {8*\k*\stretchY*(0.7/24 + 5/16)})
        -- ({8*\k*13/48}, {8*\k*\stretchY*( 0.7/24 + 41/96 )})
        -- ({8*\k*5/16}, {8*\k*\stretchY*( 0.7/24 + 41/96 )})
        -- ({8*\k*1/3}, {8*\k*\stretchY*( 0.7/32 + 101/256 )})
        -- ({8*\k*2/3}, {8*\k*\stretchY*( 0.7/32 + 101/256 )})
        -- ({8*\k*11/16}, {8*\k*\stretchY*( 0.7 / 32 + 41/96)})
        -- ({8*\k*35/48}, {8*\k*\stretchY*( 0.7 / 32 + 41/96)})
        -- ({8*\k*3/4}, {8*\k*\stretchY*( (0.7 + 15) / 48 )})
        -- ({8*\k*13/16}, {8*\k*\stretchY*( (0.7 + 15) / 48 )})
        -- ({8*\k*5/6}, {8*\k*\stretchY*( ( 0.7 + 17) / 96 )})
        -- ({8*\k*47/48}, {8*\k*\stretchY*( ( 0.7 + 17) / 96 )})
        -- ({8*\k}, {8*\k*\stretchY*0})
    ;
    % axes
    \draw[->] ({-0.5*\k}, {0*\k}) -- ({8.5*\k}, {0*\k}) node[below right] {$p$};
    \draw[->] ({0*\k}, -{0.5*\k}) -- ({0*\k}, {8.5*\k}) ;%node[above] {$\E_{\P^{\pm \e}} \bsb{ \mathrm{GFT}_t(p) }$};
    % labels
    \draw [decorate,decoration={brace, amplitude=3pt, mirror}, yshift=-1pt]
        (0,0) -- ({\k*8*3/16},0) node [black, midway, below, yshift=-2pt] {$a_1$};
    \draw [decorate,decoration={brace, amplitude=2.4pt, mirror}, yshift=-1pt]
        ({\k*8*13/48},0) -- ({\k*8*5/16},0) node [black, midway, below, yshift=-2pt] {$a_2$};
    \draw [decorate,decoration={brace, amplitude=2.4pt, mirror}, yshift=-1pt]
        ({\k*8*11/16},0) -- ({\k*8*35/48},0) node [black, midway, below, yshift=-2pt] {$a_3$};
    \draw (-1pt, {\stretchY*\k*8*(0.7/24 + 41/96)}) 
        -- (-2pt, {\stretchY*\k*8*(0.7/24 + 41/96)}) 
        -- (-2pt, {\stretchY*\k*8*(0.7 / 32 + 41/96)})
        -- (-1pt, {\stretchY*\k*8*(0.7 / 32 + 41/96)})
    ;
    \draw (-2pt, {\stretchY*\k*8*((0.7/24 + 41/96) + (0.7 / 32 + 41/96))/2}) node[left, xshift=-1pt] {$\Omega(\e)$};
    \draw (-1pt, {\stretchY*\k*8*(-0.7/24 + 41/96)}) 
        -- (-2pt, {\stretchY*\k*8*(-0.7/24 + 41/96)}) 
        -- (-2pt, {\stretchY*\k*8*(-0.7 / 32 + 41/96)})
        -- (-1pt, {\stretchY*\k*8*(-0.7 / 32 + 41/96)})
    ;
    \draw (-2pt, {\stretchY*\k*8*((-0.7/24 + 41/96) + (-0.7 / 32 + 41/96))/2}) node[left, xshift=-1pt] {$\Omega(\e)$};
    \draw ({(8*\k*1.01)}, {\stretchY*\k*8*(-0.7 / 24 + 41/96)}) 
        -- ({(8*\k*1.02)}, {\stretchY*\k*8*(-0.7 / 24 + 41/96)})
        -- ({(8*\k*1.02)}, {8*\k*\stretchY*(0.7/24 + 5/16)})
        -- ({(8*\k*1.01)}, {8*\k*\stretchY*(0.7/24 + 5/16)})
    ;
    \draw ({(8*\k*1.02)}, {\stretchY*\k*8*((-0.7 / 24 + 41/96) + (0.7/24 + 5/16))/2}) node[right, xshift=1pt] {$\Omega(1)$};
    \draw (0,0) node[below left] {$0$}
        ({8*\k},0) node[below] {$1$};
\end{tikzpicture}
}
    \caption{
    The only three regions where it makes sense for the learner to post prices are $a_1, a_2, a_3$. Prices in  $a_1$ reveal information about the sign of $\pm\e$ suffering a $\Omega(1)$ regret; prices in $a_2$ are optimal if the distribution of the seller is the red one $(+\e)$ but incur $\Omega(\e)$ regret if it is the blue one $(-\e)$; the converse happens in $a_3$.%
    }
    \label{f:t-two-third-lower-bound}
\end{figure}

We build a family of distributions $\cD_{\pm \e}$ of the seller and buyer $(S,B)$ parameterized by $\e\in[0,1]$.
For the seller, for any $\e \in [0,1]$, we define the density
\[
    f_{S,\pm\e}
:=
    \frac{1}{4\tht} \lrb{
    (1\pm\e) \I_{[0,\tht]}
+ 
    (1\mp\e) \I_{\lsb{ \frac{1}{6}, \frac{1}{6} + \tht }} 
+
    \I_{\lsb{ \frac{1}{4}, \frac{1}{4} + \tht }}
+
    \I_{\lsb{ \frac{2}{3}, \frac{2}{3} + \tht }}
    } \;,
    \tag{\text{\cref{f:t-two-third-lower-bound-a}, in red/blue}}
\]
where $\tht := \nicefrac{1}{48}$ is a normalization constant.
For the buyer, we define a single density (independently of $\e$)
\[
    f_B
:=
    \frac{1}{4\tht} \lrb{
    \I_{\lsb{ \frac{1}{3}-\tht,\,\frac{1}{3} }}
+ 
    \I_{\lsb{ \frac{3}{4} - \tht,\, \frac{3}{4} }} 
+
    \I_{\lsb{ \frac{5}{6} - \tht,\, \frac{5}{6} }}
+
    \I_{\lsb{ 1-\tht,\, 1 }}
    }\;.
    \tag{\text{\cref{f:t-two-third-lower-bound-a}, in green}}
\]
In the $+\e$ (resp., $-\e$) case, the optimal price belongs to a region $a_2$ (resp., $a_3$, see \cref{f:t-two-third-lower-bound-b}).
By posting prices in the wrong region $a_3$ (resp., $a_2$) in the $+\e$ (resp., $-\e$) case, the learner incurs $\Omega(\e)$ regret.
Thus, if $\e$ is bounded-away from zero, the only way to avoid suffering linear regret is to identify the sign of $\pm\e$ and play accordingly.
Clearly, the feedback received from the buyer gives no information on $\pm\e$. 
Since the feedback received from the seller at time $t$ by posting a price $p$ is $\I\{S_t \le p\}$, one can obtain information about (the sign of) $\pm\e$ only by posting prices in the costly ($\Omega(1)$-regret) sub-optimal region $a_1$.

This closely resembles the learning dilemma present in the so-called \emph{revealing action} partial monitoring game \cite{Nicolo06}.
In fact, a technical proof (see \cref{s:proof-t-two-thrid-lower-bound-appe}), shows that our setting is harder (i.e., it has a higher minimax regret) than an instance of a revealing action problem, 
which has a known lower bound on its minimax regret of $\frac{11}{96} \brb{ \frac{1}{7} T^{2/3} }$ \cite{cesa2006regret}.
\end{proof}

\subsection{Linear Lower Bound Under Realistic Feedback (bd) \tccheck}

In this section, we show that no strategy that can achieve worst-case sublinear regret when the seller/buyer pair $(S_t,B_t)$ is drawn i.i.d. from an unknown fixed distribution, even under the further assumption that the valuations of the seller and buyer have bounded densities.
This is due to a lack of observability.
For a full proof of the following theorem, see \cref{s:lower-bd-appe}.
\begin{theorem}
\label{thm:lower-real-bd}
In the realistic-feedback model, for all horizons $T$, the minimax regret $\Rs_T$ satisfies 
\[
    \Rs_T 
:=
    \inf_{\alpha} \sup_{(S,B) \sim \cD} R_T(\alpha)
\ge
    c T \;,
\]
where $c \ge 1/24$, the infimum is over all of the learner's strategies $\alpha$, and the supremum is over all distributions $\cD$ of the seller $S$ and buyer $B$ such that:
\begin{itemize}
    \item[\emph{(iid)}] $(S_1,B_1),(S_2,B_2),\ldots \sim (S,B)$ is an i.i.d.\ sequence;
    \item[\emph{(bd)}] $S$ and $B$ admit densities bounded by $M\ge24$.
\end{itemize}
\end{theorem}
\begin{proof}[Proof sketch]

\begin{figure}
    \centering
    \subfigure[\label{f:linear-lower-bound-lip-a}Supports of distributions $f$ (blue) and $g$ (red)]
    {
    \begin{tikzpicture}
    \def\k{0.5}
    \definecolor{myblue}{RGB}{25,175,255}
    \definecolor{myred}{RGB}{255,70,75}
    \def\colorOne{myblue}
    \def\colorTwo{myred}
    % colored squares
    \fill[\colorOne] ({0*\k},{3*\k}) rectangle ({1*\k},{4*\k});
    \fill[\colorOne] ({4*\k},{5*\k}) rectangle ({5*\k},{6*\k});
    \fill[\colorOne] ({2*\k},{7*\k}) rectangle ({3*\k},{8*\k});
    \fill[\colorTwo] ({2*\k},{3*\k}) rectangle ({3*\k},{4*\k});
    \fill[\colorTwo] ({4*\k},{7*\k}) rectangle ({5*\k},{8*\k});
    \fill[\colorTwo] ({0*\k},{5*\k}) rectangle ({1*\k},{6*\k});
    % external square
    \draw ({0*\k}, {8*\k}) -- ({8*\k}, {8*\k}) -- ({8*\k}, {0*\k});
    % diagonal lines
    \draw ({0*\k}, {0*\k}) -- ({8*\k}, {8*\k});
    % \draw[gray] ({0*\k}, {2*\k}) -- ({6*\k}, {8*\k});
    % p*
    \draw ({\k*0}, {3*\k}) -- ({\k*3}, {3*\k}) -- ({\k*3}, {8*\k});
    \draw[gray, dashed] ({\k*3}, {0*\k}) -- ({\k*3}, {3*\k});
    % q*
    \draw ({0*\k}, {\k*5}) -- ({\k*5}, {5*\k}) -- ({\k*5}, {8*\k});
    \draw[gray, dashed] ({5*\k}, {\k*0}) -- ({\k*5}, {5*\k});
    % 1/2
    \draw ({0*\k}, {\k*4}) -- ({\k*4}, {4*\k}) -- ({\k*4}, {8*\k});
    \draw[gray, dashed] ({4*\k}, {\k*0}) -- ({\k*4}, {4*\k});
    % other dashed lines
    % \draw[gray, dashed] ({\k*1}, {\k*0}) -- ({\k*1}, {\k*6});
    % \draw[gray, dashed] ({\k*2}, {\k*0}) -- ({\k*2}, {\k*8});
    % \draw[gray, dashed] ({\k*0}, {\k*6}) -- ({\k*5}, {\k*6});
    % \draw[gray, dashed] ({\k*0}, {\k*7}) -- ({\k*5}, {\k*7});
    \draw[->] ({-0.5*\k}, {0*\k}) -- ({8.5*\k}, {0*\k}) node[below right] {$s$};
    \draw[->] ({0*\k}, -{0.5*\k}) -- ({0*\k}, {8.5*\k}) node[above left] {$b$};
    % labels
    \draw (0,0) node[below left] {$0$}
        ({\k*3, 0}) node[below] {$\nicefrac{3}{8}$}
        ({\k*4, 0}) node[below] {$\nicefrac{1}{2}$}
        ({\k*5, 0}) node[below] {$\nicefrac{5}{8}$}
        ({\k*8, 0}) node[below] {$1$}
        ({0, \k*3}) node[left] {$\nicefrac{3}{8}$}
        ({0, \k*4}) node[left] {$\nicefrac{4}{8}$}
        ({0, \k*5}) node[left] {$\nicefrac{5}{8}$}
        ({0, \k*8}) node[left] {$1$}
        ;
    \draw ({\k*3}, {\k*3}) node[circle, draw, fill, inner sep=0pt, minimum width=1pt] {};
        \draw ({\k*(3+0.4)}, {\k*(3-0.3)}) node[\colorOne] {$_{\ps}$};
    \draw ({\k*5}, {\k*5}) node[circle, draw, fill, inner sep=0pt, minimum width=1pt] {};
        \draw ({\k*(5+0.4)}, {\k*(5-0.3)}) node[\colorTwo] {$_{\qs}$};
    % % legend
    % \draw[fill = \colorOne] ({11.5*\k}, {7.5*\k}) rectangle ({12*\k}, {8*\k});
    %     \draw ({11.5*\k}, {7.75*\k}) node[left] {$\{f>0\}=$};
    % \draw[fill = \colorTwo] ({11.5*\k}, {6.5*\k}) rectangle ({12*\k}, {7*\k});
    %     \draw ({11.5*\k}, {6.75*\k}) node[left] {$\{g>0\}=$};
    \end{tikzpicture}
    }
    \quad
    \subfigure[\label{f:linear-lower-bound-lip-b}Expected $\gft$ relative to $f$ (blue) and $g$ (red)]
    {
    \begin{tikzpicture}[
    declare function={
        func(\x)
    = 
        (\x < 0) * (0)
        + and(\x >= 0, \x < 1/8) * ( (1/6)*\x*(7-8*\x) )
        + and(\x >= 1/8, \x < 2/8) * ( 1/8 )
        + and(\x >= 2/8, \x < 3/8) * ( (1/6)*(-8*\x*\x + 15*\x - 5/2) )
        + and(\x >= 3/8, \x < 4/8) * ( (1/24)*(-32*\x*\x + 4*\x + 11) )
        + and(\x >= 4/8, \x < 5/8) * ( (1/24)*(-32*\x*\x + 44*\x - 9) )
        + and(\x >= 5/8, \x < 6/8) * ( (1/6)*(-8*\x*\x + 9*\x - 1) )
        + and(\x >= 6/8, \x < 7/8) * ( 5/24 )
        + and(\x >= 7/8, \x < 1) * ( (1/6)*(-8*\x*\x + 5*\x +3) )
        + (\x >= 1) * (0)
      ;
      }
    ]
    \def\k{0.5}
    \def\stretchY{3}
    \definecolor{myblue}{RGB}{25,175,255}
    \definecolor{myred}{RGB}{255,70,75}
    \def\colorOne{myblue}
    \def\colorTwo{myred}
    % dashed vertical lines
    \draw[gray, dashed] ({3*\k}, {\k*0}) -- ({\k*3}, {\stretchY*\k*8*(1/3)});
    % \draw[gray, dashed] ({4*\k}, {\k*0}) -- ({\k*4}, {\stretchY*\k*8*(5/24)});
    \draw[gray, dashed] ({5*\k}, {\k*0}) -- ({\k*5}, {\stretchY*\k*8*(1/3)});
    % dashed horizontal lines
    \draw[gray, dashed] ({\k*0}, {\stretchY*\k*8*(1/3)}) -- ({\k*5}, {\stretchY*\k*8*(1/3)});
    % \draw[gray, dashed] ({\k*0}, {\stretchY*\k*8*(5/24)}) -- ({\k*4}, {\stretchY*\k*8*(5/24)});
    \draw[gray, dashed] ({\k*0}, {\stretchY*\k*8*(1/4)}) -- ({\k*5}, {\stretchY*\k*8*(1/4)});
    % vertical line
    \draw[gray, thin] ({\k*4}, 0) -- ({\k*4}, {\k*8.5});
    % plots
    \draw[domain = 0:{\k*8}, myred, samples = 1024] plot (\x, {\stretchY*\k*8*func(1-\x/(\k*8))});
    \draw[domain = 0:{\k*8}, myblue, samples = 1024] plot (\x, {\stretchY*\k*8*func(\x/(\k*8))});
    % axes
    \draw[->] ({-0.5*\k}, {0*\k}) -- ({8.5*\k}, {0*\k}) node[below right] {$p$};
    \draw[->] ({0*\k}, -{0.5*\k}) -- ({0*\k}, {8.5*\k});
    % labels
    \draw (0,0) node[below left] {$0$}
        ({\k*3, 0}) node[below] {$\nicefrac{3}{8}$}
        ({\k*4, 0}) node[below] {$\nicefrac{1}{2}$}
        ({\k*5, 0}) node[below] {$\nicefrac{5}{8}$}
        ({\k*8, 0}) node[below] {$1$}
        % ({0, 8*\k*\stretchY*5/24}) node[left] {$\nicefrac{5}{24}$}
        ({0, 8*\k*\stretchY/4}) node[left] {$\nicefrac{1}{4}$}
        ({0, 8*\k*\stretchY/3}) node[left] {$\nicefrac{1}{3}$}
        ;
    % legend
    % \draw[fill = \colorOne] ({9.5*\k}, {7.5*\k}) rectangle ({10*\k}, {8*\k});
    %     \draw ({9.5*\k}, {7.75*\k}) node[left] {$\E_{\P} \bsb{ \mathrm{GFT}_t(p) }=$};
    % \draw[fill = \colorTwo] ({9.5*\k}, {6.5*\k}) rectangle ({10*\k}, {7*\k});
    %     \draw ({9.5*\k}, {6.75*\k}) node[left] {$\E_{\Q} \bsb{ \mathrm{GFT}_t(p) }=$};
\end{tikzpicture} 
}
    \caption{Under realistic feedback, the two densities $f$ and $g$ are indistinguishable. The optimal price $p^\star$ for $f$ gives constant regret under $g$ and $q^\star$ does the converse.}
    \label{f:linear-lower-bound-lip}
\end{figure}
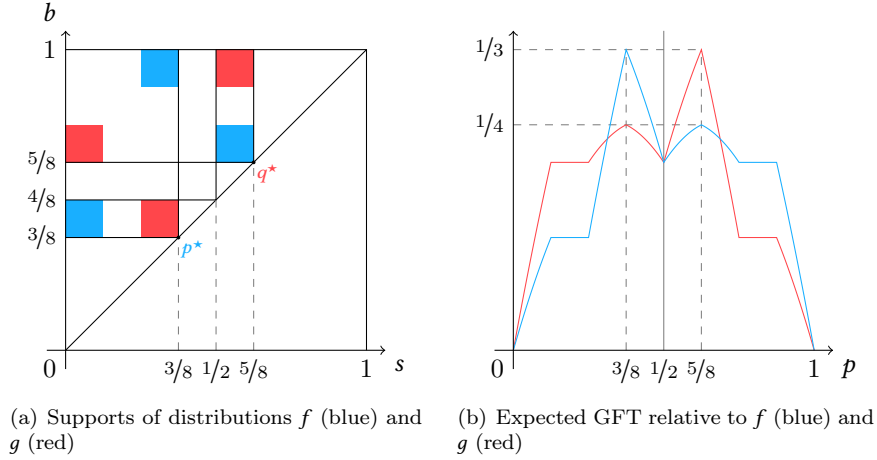

Consider the two joint densities $f$ and $g$ of the seller and buyer $(S,B)$ as the normalized indicator functions of the red and blue squares in \cref{f:linear-lower-bound-lip-a}. Formally 
\[
    f
= 
    \frac{64}{3}  \Brb{ \I_{\lsb{ \nicefrac{0}{8}, \,\nicefrac{1}{8} } \times \lsb{ \nicefrac{3}{8}, \,\nicefrac{4}{8} } } + \I_{\lsb{ \nicefrac{2}{8}, \,\nicefrac{3}{8} } \times \lsb{ \nicefrac{7}{8}, \,\nicefrac{8}{8} }} + \I_{\lsb{ \nicefrac{4}{8}, \,\nicefrac{5}{8} } \times \lsb{ \nicefrac{5}{8}, \, \nicefrac{6}{8} }} }
\]
and $g(s,b)=f(1-b,1-s)$.
In the $f$ (resp., $g$) case, the optimal price belongs to the region $[0,\nicefrac{1}{2}]$ (resp., $(\nicefrac{1}{2}, 1]$, see \cref{f:linear-lower-bound-lip-b}).
By posting prices in the wrong region $(\nicefrac{1}{2}, 1]$ (resp., $[0,\nicefrac{1}{2}]$) in the $f$ (resp., $g$) case, the learner incurs at least a $\nicefrac{1}{3}-\nicefrac{1}{4} = \nicefrac{1}{12}$ regret.
Thus, the only way to avoid suffering linear regret is to determine if the valuations of the seller and buyer are generated by $f$ or $g$.
For each price $p \in [0,1]$, consider the four rectangles with opposite vertices $(p,p)$ and $(u_i,v_i)$, where $\lcb{ (u_i,v_i) }_{i=1,\ldots,4}$ are the four vertices of the unit square.
Note that the only information on the distribution of $(S,B)$ that the learner can gather from the realistic feedback $\brb{ \I\{S_t \le p\}, \, \I \{p\le B_t\} }$ received after posting a price $p$ is (an estimate of) the area of the portion of the support of the distribution included in each of these four rectangles.
However, these areas coincide in the cases $f$ and $g$.
Hence, in under realistic feedback, $f$ and $g$ are completely indistinguishable.
Therefore, given that the optimal price in the $f$ (resp., $g$) case is $\nicefrac{3}{8}$ (resp., \nicefrac{5}{8}), the best that the learner can do is to sample prices uniformly at random in the set $\{\nicefrac{3}{8},\nicefrac{5}{8}\}$, incurring a regret of $\nicefrac{T}{24}$. 
For a formalization of this argument that leverages the techniques we described in the introduction, see \cref{s:lower-bd-appe}.
\end{proof}

\subsection{Linear Lower Bound Under Realistic Feedback (iv) \tccheck}
\label{sec:linear_real}

In this section, we prove that in the realistic-feedback case, no strategy can achieve sublinear regret without any limitations on how concentrated the distributions of the valuations of the seller and buyer are, not even if they are independent of each other (iv).

At a high level, if the two distributions are of the seller and the buyer are very concentrated in a small region, finding an optimal price is like finding a needle in a haystack.
For a full proof of the following theorem, see \cref{sec:linear_real-appe}. 
\begin{theorem}
\label{thm:lower-real-iv}
In the realistic-feedback model, for all horizons $T$, the minimax regret $\Rs_T$ satisfies 
\[
    \Rs_T 
:=
    \inf_{\alpha} \sup_{(S,B) \sim \cD} R_T(\alpha)
\ge
    c T \;,
\]
where $c \ge 1/8$, the infimum is over all of the learner's strategies $\alpha$, and the supremum is over all distributions $\cD$ of the seller $S$ and buyer $B$ such that:
\begin{itemize}
    \item[\emph{(iid)}] $(S_1,B_1),(S_2,B_2),\ldots \sim (S,B)$ is an i.i.d.\ sequence;
    \item[\emph{(iv)}] $S$ and $B$ are independent of each other.
\end{itemize}
\end{theorem}
\begin{proof}[Proof sketch]

\begin{figure}
    \centering
    \subfigure[\label{fig:linear_lower_real_iv-a}Distribution of $S^x$ (red) and $B^x$ (green)]{
    \begin{tikzpicture}
    \def\k{0.45}
    \def\x{7}
    \def\stretchY{3}
    \definecolor{myblue}{RGB}{25,175,255}
    \definecolor{myred}{RGB}{255,70,75}
    \def\colorOne{myblue}
    \def\colorTwo{myred}
    \definecolor{mygreen}{RGB}{64,128,128}
    % y axis
    \draw[->] ({0*\k}, {-0.5*\k}) -- ({0*\k}, {8*\k});
    % Dirac measures
    \draw[->, myred, thick] (0,0) -- (0, {2*\k*\stretchY});
    \draw[->, myred, thick] ({\x*(0.995)*\k}, 0) -- ({\x*(0.995)*\k}, {2*\k*\stretchY});
    \draw[->, mygreen, thick] ({\x*(1.005)*\k}, 0) -- ({\x*(1.005)*\k}, {2*\k*\stretchY});
    \draw[->, mygreen, thick] ({10*\k}, 0) -- ({10*\k}, {2*\k*\stretchY});
    \draw (0,0) node[below left] {$0$}
        ({\x*\k}, 0) node [below] {$x$}
        ({10*\k}, 0) node [below] {$1$}
        ({10.3*\k}, 0) node [below right] {$p$}
    ;
    % x axis
    \draw[->] ({-0.5*\k}, {0*\k}) -- ({10.4*\k}, {0*\k});
    \end{tikzpicture}}
    \subfigure[\label{fig:linear_lower_real_iv-b}Expected gain from trade relative to $S^x$ and $B^x$]
    {
    \begin{tikzpicture}
    \def\k{0.45}
    \def\x{7}
    \def\stretchY{3.5}
    \definecolor{myblue}{RGB}{25,175,255}
    \definecolor{myred}{RGB}{255,70,75}
    \def\colorOne{myblue}
    \def\colorTwo{myred}
    \definecolor{mygreen}{RGB}{64,128,128}
    \draw[gray, dashed] ({\x*\k},0) -- ({\x*\k},{2*\k*\stretchY});
    \draw[gray, dashed] ({10*\k},0) -- ({10*\k},{\x+10)/10*\k*\stretchY});
    \draw[gray, dashed] (0,{(\x+10)/10*\k*\stretchY}) -- ({\x*\k},{(\x+10)/10*\k*\stretchY});
    \draw[gray, dashed] (0,{2*\k*\stretchY}) -- ({\x*\k},{2*\k*\stretchY});
    \draw (0,{(20-\x)/10*\k*\stretchY}) -- ({\x*\k},{(20-\x)/10*\k*\stretchY});
    \draw ({\x*\k},{(\x+10)/10*\k*\stretchY}) -- ({10*\k},{(\x+10)/10*\k*\stretchY});
     \draw[->] (0, {-0.5*\k}) -- (0,{8*\k});
     \draw[->] ({-0.5*\k}, 0) -- ({10.3*\k},0);
    % % \draw ({3*\k*\stretchY},{(16/32)*4*\k*\stretchY}) node[circle, draw, green!50!black, fill = green!50!black, inner sep = 0pt, minimum width = 1.5pt] {};
     \draw ({\x*\k},{2*\k*\stretchY}) node[circle, draw, fill, inner sep = 0pt, minimum width = 1pt] {};
    \draw (0,0) node[below left] {$0$}
        ({\x*\k},0) node[below] {$x$}
        ({10*\k},0) node[below] {$1$}
        (0,{(\x+10)/10*\k*\stretchY}) node[left] {$(1+x)/4$} % 1/4 + x/4 = (1 + x)/4
        (0,{(20-\x)/10*\k*\stretchY}) node[left] {$(2-x)/4$} % 1/4 + x/4 = (1 + x)/4
        (0,{2*\k*\stretchY}) node[left] {$1/2$} % 1/4 + x/4 = (1 + x)/4
        ({10.3*\k},0) node[below right] {$p$}
        % (0,2.375) node[left] {$\mathrm{GFT}(p)$}
    ;
    \end{tikzpicture}
    }
    \caption{All prices, except for $x$, have high regret. However, under realistic feedback, finding $x$ in a finite time is as harder than finding a needle in a haystack.}
    \label{fig:linear_lower_real_iv}
\end{figure}
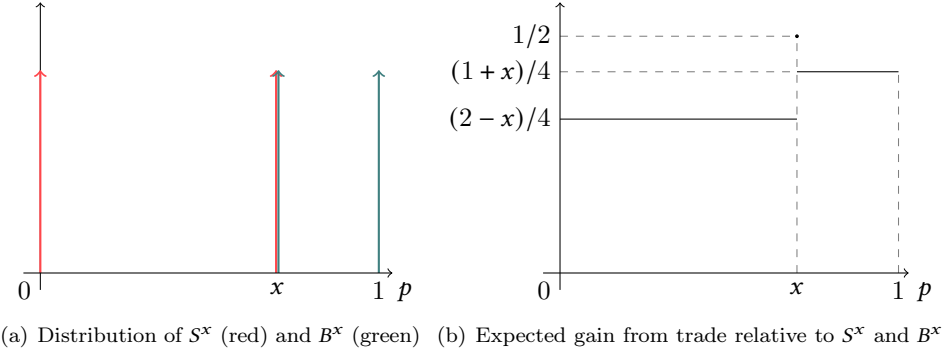

Consider a family of seller/buyer distributions $(S^x,B^x)$, parameterized by $x\in I$, where $I$ is a small interval centered in $\nicefrac 12$, $S^x$ and $B^x$ are independent of each other, and they satisfy
\[
S^x=\begin{cases}
    x &\text{with probability $\frac 12$}\\
    0 &\text{with probability $\frac 12$} 
\end{cases}\;,
\qquad
B^x=\begin{cases}
    x &\text{with probability $\frac 12$}\\
    1 &\text{with probability $\frac 12$} 
\end{cases}\;.
\]
The distributions and the corresponding gain from trade are represented in \cref{fig:linear_lower_real_iv-a} and \cref{fig:linear_lower_real_iv-b}, respectively. 
A direct verification shows that the best fixed price with respect to $(S^x,B^x)$ is $p=x$. 
Furthermore, by posting any other prices, the learner incurs a regret of approximately $1/2$ with probability $1/4$.
Since it is intuitively clear that no strategy can locate (exactly!) each possible $x\in I$ in a finite number of steps. 
This results, for any strategy, in regret of at least (approximately) $T/8$.
For the interested reader, a more detailed analysis is carried over in the \cref{sec:linear_real-appe}.
\end{proof}

\section{Adversarial Setting: Linear Lower Bound Under Full Feedback \tccheck}
\label{sec:adversarial}

In this section, we prove that even in the simpler full-feedback case, no strategy can achieve worst-case sublinear regret in an adversarial setting. 
Lower bounds for the adversarial setting have a slightly different structure that the stochastic ones. The idea of the proof is to build, for any strategy, an {\em hard} sequence of sellers and buyers' valuations $(s_1,b_1),(s_2,b_2),\ldots$ which causes the algorithm to suffer linear regret for any horizon $T$.
\begin{theorem}
\label{thm:adv-lower}
In the full-feedback adversarial (adv) setting, for all horizons $T \in \N$, the minimax regret $\Rs_T$ satisfies
\[
    \Rs_T 
:=
    \inf_{\alpha} \sup_{(s_1,b_1),(s_2,b_2),\ldots} R_T(\alpha)
\ge 
    c T \;,
\]
where $c \ge 1/4$, the infimum is over all of the learner's strategies $\alpha$, and the supremum is over all deterministic sequences $(s_1,b_1),(s_2,b_2),\ldots \in [0,1]^2$ of the seller and buyer's valuations.
\end{theorem}
\begin{proof}
Let $\e \in (0,\frac 1{18})$.
We begin by fixing any strategy $\alpha$ of the learner.
This is a sequence of functions $\alpha_t$ mapping the past feedback $(s_1,b_1), \ldots, (s_{t-1},b_{t-1})$, together with  some internal randomization, to the price $P_t$ to be posted by the learner at time $t$.
In other words, the strategy maintains a distribution $\nu_t$ over the prices that is updated after observing each new pair $(s_t,b_t)$ and used to draw each new price $P_t$.
We will show how to constructively determine a sequence of seller/buyer valuations that is hard for $\alpha$ to learn.
This sequence is oblivious to the prices $P_1,P_2,\ldots$ posted by $\alpha$, in the sense it does not have access to the realizations of its internal randomization.
The idea is, at any time $t$, to determine a seller/buyer pair $(s_t,b_t)$ either of the form $(c_t,1)$ or $(0,d_t)$, with $c_t \approx \frac{1}{2} \approx d_t$, such that the probability $\nu_t$ that the strategy picks a price $P_t \in [s_t,b_t]$ (i.e., that there is a trade) is at most $\nicefrac{1}{2}$ and, at the same time, there is common price $\ps$ which belongs to $[s_t,b_t]$ for all times $t$.
This way, since $b_t-s_t \approx \frac{1}{2}$ for all $t$, the regret of $\alpha$ with respect to $(s_1,b_1), (s_2,b_2), \ldots$ is at least (approximately) greater than or equal to $\nicefrac{T}{4}$.

The formal construction proceeds inductively as follows. Let
\[
    \begin{cases}
        c_1:=\frac{1}{2}-\frac{3}{2}\varepsilon, \  d_1:=\frac{1}{2}-\frac{1}{2}\varepsilon, \  s_1 := 0,  \ b_1:=d_1,
    &
        \text{ if } \nu_{1}\bsb{\bsb{0,\frac{1}{2}-\frac{1}{2}\varepsilon}}\le\frac{1}{2} \;,
    \\
        c_1:=\frac{1}{2}+\frac{1}{2}\varepsilon, \  d_1:=\frac{1}{2}+\frac{3}{2}\varepsilon, \  s_1 := c_1, \  b_1:=1,
    &
        \text{ otherwise}.
    \end{cases}
\]
Then, for any time $t$, given that $c_i, d_i, s_i, b_i$ are defined for all $i\le t$ and recalling that $\nu_{t+1}$ is the distribution over the prices at time $t+1$ (of the strategy $\alpha$ after observing the feedback $(s_1,b_1), \ldots, (s_t,b_t)$), let
\[
    \begin{cases}
        c_{t+1}:=c_{t}, \ d_{t+1}:=d_{t}-\frac{2\varepsilon}{3^{t}}, \ s_{t+1}:=0, \ b_{t+1}:=d_{t+1},
    &
        \text{if } \nu_{t+1}\bsb{ \bsb{0,c_t+\frac{\varepsilon}{3^{t}}} }\le\frac{1}{2} \;,
    \\
        c_{t+1}:=c_{t}+\frac{2\varepsilon}{3^{t}}, \ d_{t+1}:=d_{t}, \ s_{t+1} := c_{t+1}, \ b_{t+1}:=1,
    &
        \text{otherwise}.
    \end{cases}
\]
Then the sequence of seller/buyer valuations $(s_1,b_1),(s_2,b_2),\ldots$ defined above by induction satisfies:
\begin{itemize}
    \item $\nu_t\bsb{ [s_t,b_t] } \le \frac{1}{2}$, for each time $t$;
    \item there exists $\ps \in [0,1]$ such that $\ps\in[s_t,b_t]$, for each time $t$ (e.g. $\ps:=\lim_{t \to \infty} c_t$);
    \item $b_t - s_t \ge \frac{1-3\e}{2}$, for each time $t$.
\end{itemize}
This implies, for any horizon $T$,
\[
    R_T(\alpha)
=
    \sum_{t=1}^T \gft(\ps,s_t,b_t) -  \sum_{t=1}^T \E \lsb{ \gft ( P_t, s_t, b_t ) }
\ge
    \sum_{t=1}^T (b_t-s_t)\brb{1-\nu_t \bsb{ [s_t,b_t] }}
\ge 
    \frac{1-3\e}{4}T\;.
\]
Since $\e$ and $\alpha$ are arbitrary, 
this yields immediately $\Rs \ge {T}/{4}$.
\end{proof}

\section{Conclusions \tccheck}
This work initiates the study of the bilateral trade problem in a regret minimization framework.  
We prove tight bounds on the regret rates that can be achieved under various feedback and private valuation models.  

Our work opens several possibilities of future investigation. 
One first and natural research direction is related to the more general settings of two-sided markets with multiple buyers and sellers, different prior distributions, and complex valuation functions.  

A second direction is related to the tight characterization of the regret rates for weak budget balance mechanisms. These can be proved to be strictly better than strong budget balance mechanisms, at least for the realistic feedback setting with correlated distributions (details will appear in the full version of this work). 

Finally, we believe other classes of markets, which assume prior knowledge of the agent's preferences, could be fruitfully studied in a regret minimization framework.

\bibliographystyle{ACM-Reference-Format}
\bibliography{references}

% Appendix
\appendix

\clearpage

\section{Model and Notation \tccheck}

For all $T\in \N$, we denote the set of the first $T$ integers $\{1,\ldots,T\}$ by $[T]$.
If $\P$ is a probability measure and $X$ is a random variable, we denote by $\P_X$ the probability measure defined for any (measurable) set $E$, by $\P_X[E] := \P[ X \in E]$.
We denote the expectation of a random variable $X$ with respect to probability measure $\P$ by $\E_\P[X]$.
If a measure $\nu$ is absolutely continuous with respect another measure $\mu$ with density $f$, we denote $\nu$ by $f \mu$, so that for any (measurable) set $E$, $(f \mu) [E] := \nu[E] = \int_E f(x) \dif \mu(x)$. 
We denote the \leb{} measure on the interval $[0,1]$ by $\mu_L$ and the product \leb{} measure on $[0,1]^\N$ by $\bmu_L$.
For any set $E$ and $x\in E$, we denote the Dirac measure on $x$ by $\delta_x$ (the dependence on $E$ will always be clear from context).

\subsection{The Learning Model \tccheck} 
In this section, we introduce an abstract notion of sequential games which encompasses all the settings we discussed in the main part of the paper, providing a unified perspective.
This will be especially useful when proving lower bounds.
\begin{definition}[Sequential game]
\label{d:game-appe}
A \emph{(sequential) game} is a tuple $\sG := (\cX,\cY,\cZ,\rho,\fhi,\sP)$, where:
\begin{itemize}
    \item $\cX,\cY,\cZ$ are sets called the \emph{player's action space}, \emph{adversary's action space}, and \emph{feedback space};
    \item $\rho \colon \cX \times \cY \to [0,1]$ and $\fhi \colon \cX \times \cY \to \cZ$ are called the \emph{reward} and \emph{feedback} functions\footnote{More precisely, we need $\cX,\cY,\cZ$ to be non-empty measurable spaces and $\rho,\fhi$ to be measurable functions. To avoid clutter, in the following we will never mention explicitly these types of standard measurability assumptions unless strictly needed.};
    \item $\sP$ is a set of probabilities on the set $\cY^\N$ of sequences in $\cY$, called the \emph{adversary's behavior}.
\end{itemize}
\end{definition}
This definition generalizes the partial monitoring games of \citep{lattimore2020bandit,bartok2014partial} to settings with infinitely many arms and is able to model adversarial, i.i.d., and more general stochastic settings all at once.
Before proceeding, we introduce another few extra handy definitions that will be used throughout the paper.
\begin{definition}
\label{d:extra-stuff-appe}
If $\sG = (\cX,\cY,\cZ,\rho,\fhi,\sP)$ is a game, then we say the following.
\emph{The sample space} is the set $\Omega := \cY^{\N} \times [0,1]^{\N}$.
\emph{The adversary's actions} $\brb{ Y_t }_{t\in \N}$ and \emph{the player's randomization} $\brb{ U_t }_{t\in\N}$ are sequences of random variables defined, for all $t \in \N$ and $\omega = \brb{ (y_n)_{n\in\N}, (u_n)_{n\in \N} } \in \Omega$, by $Y_t(\omega):=y_t$ and $U_t(\omega) := u_t$.
\emph{The set of scenarios} $\sS$ is the set of probability measures $\P$ on $\Omega$ of the form $\P = \boldsymbol \mu \otimes \boldsymbol{\mu}_L$, where $\boldsymbol \mu \in \sP$.
\end{definition}
For the sake of conciseness, whenever we fix a game $\sG$, we will assume that all the objects (sets, functions, random variables) presented in Definitions~\ref{d:game-appe}--\ref{d:extra-stuff-appe} are fixed and denoted by the same letters without declaring them explicitly each time, unless strictly needed.

Note that this setting models an \emph{oblivious} adversary since its actions are independent of the player's past randomization, i.e., for all $t\in \N$, $\P_{Y_{t+1} \mid Y_1, \ldots, Y_t, U_1, \ldots, U_t } = \P_{Y_{t+1} \mid Y_1, \ldots, Y_t}$.
Note also that we are assuming that the randomization of the player's strategy is carried out by drawing numbers in the interval $[0,1]$ independently and uniformly at random.
We can restrict ourselves to this case in light of the Skorokhod Representation Theorem \cite[Section 17.3]{williams1991probability} without losing (much) generality.
We now introduce formally the strategies of the player, the resulting played actions, and the corresponding feedback.

\begin{definition}[Player's strategies, actions, and feedback]
Given a game $\sG$, we define a \emph{player's strategy} as a sequence of functions $\alpha = (\alpha_t)_{t\in \N}$ such that, for each $t \in \N$, $\alpha_{t} \colon [0,1]^t\times\cZ^{t-1} \to \cX$.\footnote{When $t=1$, $[0,1]^t \times \cZ^{t-1} := [0,1]$. In the following, we will always adopt this type of convention without mention it.}
Given a player's strategy $\alpha$, we define inductively (on $t$) the corresponding sequences of \emph{player's actions} $(X_t)_{t\in \N}$ and \emph{player's feedback} $(Z_t)_{t\in \N}$ by $X_t := \alpha_t(U_1, \ldots, U_t, Z_1, \ldots, Z_{t-1})$, $Z_t := \varphi (X_t, Y_t)$.
In the sequel, we will denote the set of all strategies for a game $\sG$ by $\sA(\sG)$.
% :
\end{definition}
To lighten the notation, we will write $\sA$ instead of $\sA(\sG)$ if it is clear from context.
We can now extend the standard notions of regret, worst-case regret, and minimax regret to our general setting.
\begin{definition}[Regret]

Given a game $\sG$ and a horizon $T\in \N$, we define the \emph{regret} (of $\alpha \in \sA$ in scenario $\P \in \sS$), the \emph{worst-case regret} (of $\alpha \in \sA$), and the \emph{minimax regret} (of $\sG$), respectively, by
\[
R^{\P}_T(\alpha) := \sup_{x \in \cX} \E_{\P}\lsb{\sum_{t=1}^T \rho(x,Y_t) - \sum_{t=1}^T \rho(X_t,Y_t)}\;,
\quad
R_T^{\sS}(\alpha) := \sup_{\P \in \sS} R^{\P}_T(\alpha)\;,
\quad
\Rs_T(\sG) := \inf_{\alpha \in \sA(\sG)} R_T(\alpha)\;.
\]
\end{definition}
Informally, if $\sG$ and $\widetilde{\sG}$ are two games and $\Rs_T(\sG) \ge \Rs_T(\tilde \sG)$, we say that $\tilde{\sG}$ is \emph{easier} than $\sG$ (or equivalently, that $\sG$ is \emph{harder} than $\tilde{\sG}$).
When it is clear from the context, we will omit the dependence on $\sG$ in $\Rs_T(\sG)$.

\subsection{Bilateral Trade as a Game \tccheck}
\label{s:biltrad-setting-appe}

We now formally cast the various instances of bilateral trade we introduced in \cref{s:bil-tr-model} into our sequential game setting.
In this context, we think of the learner as the \emph{player} and the environment as the \emph{adversary}.

\subsubsection{Player's Actions, Adversary's Actions, and Reward}

The player's action space $\cX$ is the unit interval $[0,1]$.
This corresponds to the player posting the same price price to both the seller and the buyer (strong budget balance).
The adversary's action space $\cY$ is $[0,1]^2$. 
They are the pairs of valuations of the seller and buyer.
The reward function $\rho$ is the gain from trade $\gft\colon [0,1] \times [0,1]^2 \to [0,1]$, $\brb{ p, (s,b) } \mapsto (b-s) \I\{ s \le p \le b\}$.

\subsubsection{Available Feedback} 

\begin{description}
    \item[Realistic] the feedback space $\cZ$ is the boolean square $\{0,1\}^2$ and the feedback function is $\fhi \colon [0,1] \times [0,1]^2 \to \{0,1\}^2$, $\brb{ p, (s,b) } \mapsto \brb{ \I\{s \le p\}, \I\{ p\le b \} }$.
    This corresponds to the seller and the buyer accepting or rejecting a trade at a price $p$.
    \item[Full] the feedback space $\cZ$ is the unit square $[0,1]^2$ and the feedback function is $\fhi \colon [0,1] \times [0,1]^2 \to [0,1]^2$, $\brb{ p, (s,b) } \mapsto (s,b)$.
    This corresponds to the seller and the buyer revealing their valuations at the end of a trade.
\end{description}

\subsubsection{Adversary's Behavior} 
\begin{description}
    \item[Stochastic (iid):] 
    the adversary's behavior $\sP = \sPiid$ consists of products of a single probability on $\cY = [0,1]^2$, i.e., $\bmu \in \sPiid$ if and only if there exists a probability measure $\mu$ on $[0,1]^2$ such that $\bmu = \otimes_{t\in \N} \, \mu$.
    This corresponds to a stochastic i.i.d.\ environment, where however the valuations of the seller and the buyer could be correlated.
    
    We will also investigate the following stronger assumptions.
    \begin{description}
        \item[Independent valuations (iv)] 
        the adversary's behavior $\sP = \sPiv$ is the subset of $\sPiid$ in which the valuations of the seller and the buyer are independent, i.e., $\bmu \in \sPiv$ if and only if there exist two probability measures $\mu_S,\mu_B$ on $[0,1]$ such that $\bmu = \otimes_{t\in \N} \, (\mu_S \otimes \mu_B)$.
        \item[Bounded density (bd)] 
        for a fixed $M \ge 1$, the adversary's behavior $\sP = \sPbd$ is the subset of $\sPiid$ in which the joint distribution of the valuations of buyer and seller has a density bounded by $M$, i.e., $\bmu \in \sPbd$ if and only if there exists a density $f\colon [0,1]^2 \to [0,M]$ such that  $\bmu = \otimes_{t\in \N} \, ( f \mu ) $, where $\mu = \mu_L \otimes \mu_L$.
        \item[Independent valuations with bounded density (iv+bd)] 
        for a fixed $M \ge 1$, the adversary's behavior $\sP = \sPivbd$ is the subset $\sPiv \cap \sPbd$ of $\sPiid$.
    \end{description}
    \item[Adversarial (adv):] 
    the adversary's behavior $\sP = \sPadv$ consists of products of Dirac measures on $\cY = [0,1]^2$, i.e., $\bmu \in \sPadv$ if and only if there exists a sequence $(s_t,b_t)_{t\in\N} \s [0,1]^2$ such that $\bmu = \otimes_{t\in \N} \, \delta_{(s_t,b_t)} $.
    This corresponds to a deterministic, oblivious, and adversarial environment (\cref{sec:adversarial}).
\end{description}

\section{Two Key Lemmas on Simplifying Sequential Games \tccheck}
\label{s:keylemmas}
In this section we introduce some useful techniques that could be of independent interest for proving lower bounds in sequential games.
The idea is to give sufficient conditions for given game to be harder than another, where the second one has a known lower bound on its minimax regret.

At a high level, the first lemma shows that if the adversary's actions are independent of each other, a game $\tilde{\sG}$ is easier than game $\sG$ if $\tilde{\sG}$ can be embedded in $\sG$ in such a way that the optimal player's actions of $\tilde{\sG}$ are no better than the ones in $\sG$, the suboptimal player's actions of $\tilde{\sG}$ no worse than the ones in $\sG$, and at distributional level, the quality of the feedback does not decrease in the second game. 
The proof is deferred to \cref{sec:lemmas}.

\begin{restatable}[Embedding]{lemma}{lembedding}
\label{l:embedding}
Let $\sG := (\cX,\cY,\cZ,\rho,\fhi,\sP)$ and $\tilde{\sG} := (\tilde{\cX},\tilde{\cY},\tilde{\cZ},\tilde{\rho},\tilde{\fhi},\tilde{\sP})$ be two games, $\sS, \tilde \sS$ their respective sets of scenarios, $(Y_t)_{t\in\N}, (\tilde Y_t)_{t\in\N}$ their adversaries' actions, and $T\in \N$ a horizon.
Assume that
$Y_1,\ldots,Y_T$ are $\P$-independent for any scenario $\P \in \sS$, $\tilde Y_1, \ldots, \tilde Y_T$ are $\tilde \P$-independent for any scenario $\tilde \P \in \tilde \sS$, and 
that there exist 
$
    \slf \colon \cX \to \tilde{\cX}
$, 
$\slg \colon \tilde{\cZ} \to \cZ
$,
and 
$\slh \colon \tilde{\sS} \to \sS
$
satisfying:
\begin{enumerate}
    \item \label{bullone} $\sup_{\tilde{x} \in \tilde{\cX}} \sum_{t=1}^T \E_{\tilde{\P}} \bsb{ \tilde{\rho}(\tilde{x},\tilde{Y}_t) } \le \sup_{x \in \cX}\sum_{t=1}^T \E_{\slh(\tilde{\P})} \bsb{ \rho(x,Y_t) }$ for any scenario $\tilde{\P} \in \tilde{\sS}$;
    \item \label{bulltwo} $\E_{\tilde{\P}}\bsb{ \tilde{\rho} \brb{ \slf(x),\tilde{Y}_t } } \ge \E_{\slh(\tilde{\P})}\bsb{ \rho(x,Y_t) }$ for any time $t\in [T]$, scenario $\tilde \P \in \tilde \sS$, and action $x \in \cX$;
    \item \label{bullthree} $\tilde{\P}_{\slg\lrb{\tilde{\fhi}\brb{\slf(x),\tilde{Y}_t}}} = \brb{\slh(\tilde{\P})}_{\fhi(x,Y_t)}$ for any time $t\in [T]$, scenario $\tilde{\P} \in \tilde{\sS}$, and action $x \in \cX$.
\end{enumerate}
Then 
$
\Rs_T(\sG) \ge \Rs_T(\tilde{\sG})
$.
\end{restatable}

The second lemma addresses feedback with uniformative (i.e., scenario-independent) components.
At a high level, if the feedback of some of the player's actions has one or more uninformative components, the game can be simplified by getting rid of the uninformative parts of the feedback.
The player can achieve this by simulating the uninformative parts of the feedback using their randomization.
The proof is deferred to \cref{sec:lemmas}.

\begin{restatable}[Simulation]{lemma}{lsimulation}
\label{l:simulation}
Let $\cV,\cW$ be two sets, $\sG := (\cX,\cY,\cZ,\rho,\fhi,\sP)$ a game with $\cZ = \cV \times \cW$, $\sS$ its set of scenarios, $(Y_t)_{t\in \N}$ its adversary's actions, $\pi \colon \cZ \to \cV$ the projection on $\cV$, and $T \in \N$ a horizon.
Assume that
$Y_1,\ldots,Y_T$ are $\P$-independent for any scenario $\P \in \sS$ and that there exist disjoint sets $\cR, \cU \subset \cX$ such that $\cR \cup \cU = \cX$ and
\begin{enumerate}
    \item \label{buone} for any time $t \in [T]$ and action $x \in \cR$ there exists $\psi_{t,x} \colon [0,1] \to \cW$ such that, for all $\P \in \sS$,
    \[
        \P_{\fhi(x,Y_t)} = \P_{\pi\brb{\fhi(x,Y_t)}} \otimes (\mu_L)_{\psi_{t,x}} \;;
    \]
    \item \label{butwo} for any time $t \in [T]$ and action $x \in \cU$, there exists $\gamma_{t,x}\colon[0,1] \to \cZ$ such that, for all $\P \in \sS$,
    \[
        \P_{\fhi(x,Y_t)} = (\mu_L)_{\gamma_{t,x}} \;.
    \]
\end{enumerate}
Let $* \in \cV$ and define
\[
\tilde{\fhi} \colon \cX \times \cY \to \cV, \ (x,y) \mapsto
\begin{cases}
\pi\brb{\varphi(x,y)} \;, &\text{ if } x \in \cR,\\
* \;, &\text{ if } x \in \cU.
\end{cases}
\]
Define the game $\tilde{\sG} := (\cX,\cY,\cV,\rho,\tilde{\fhi},\sP)$.
Then
$
\Rs_T(\sG) \ge \Rs_T(\tilde{\sG})
$.
\end{restatable}

\subsection{Proofs of the lemmas \tccheck}
\label{sec:lemmas}

In this section, we will give a full proof of the two important Embedding and Simulation lemmas introduces in \cref{s:keylemmas}.
To lighten the notation, for any $m,n \in \N$, with $m\le n$ and a family $(\lambda_k)_{k\in\N}$ we let $\lambda_{m:n} := (\lambda_m, \lambda_{m+1}, \ldots, \lambda_n)$ and similarly $\lambda_{n:m} := (\lambda_{n}, \lambda_{n-1} \ldots, \lambda_m)$.

We begin by proving the Embedding lemma.

\lembedding*

\begin{proof}
Fix any strategy $\alpha \in \sA(\sG)$. For each time $t \in \N$, define
\[
\tilde{\alpha}_t \colon [0,1]^t \times \tilde{\cZ}^{t-1} \to \tilde \cX, (u_1, \dots, u_t, \tilde{z}_1, \dots, \tilde{z}_{t-1}) \mapsto \slf\Brb{\alpha_t\brb{u_1,\dots, u_t, \slg(\tilde{z}_1), \dots, \slg(\tilde{z}_{t-1})}}.
\]
Then $\tilde{\alpha} := (\tilde{\alpha}_t)_{t \in \N} \in \sA(\tilde \sG)$. As usual, let $(Y_t)_{t \in \N}$ and $(U_t)_{t \in \N}$ be the adversary's actions and the player's randomization in game $\sG$ and $(X_t)_{t \in \N}$ and $(Z_t)_{t \in \N}$ the player's actions and the feedback according to the strategy $\alpha$. Let $(\tilde{Y}_t)_{t \in \N}, (\tilde{U}_t)_{t \in \N},(\tilde{X}_t)_{t \in \N},(\tilde{Z}_t)_{t \in \N}$ be the corresponding objects for the game $\tilde{\sG}$ and the strategy $\tilde{\alpha}$. Furthermore, define
\[
\hat{X}_1 = \alpha_1(\tilde{U}_1), \quad 
\hat{Z}_1 = \slg\brb{\tilde{\fhi}(\tilde{X}_1, \tilde{Y}_1)}, \quad 
\hat{X}_2 = \alpha_2(\tilde{U}_1,\tilde{U}_2,\hat{Z}_1), \quad 
\hat{Z}_2 = \slg\brb{\tilde{\fhi}(\tilde{X}_2, \tilde{Y}_2)}, \dots \;.
\]
Fix $\tilde{\P} \in \tilde{\sS}$, where $\tilde{\sS}$ are the scenarios of the game $\tilde \sG$. 
Then $\tilde{\P}_{\tilde{U}_1} = \brb{\slh(\tilde{\P})}_{U_1}$. Now, since $X_1 = \alpha_1(U_1)$ and $\hat{X}_1 =\alpha_1 (\tilde{U}_1)$, we also have that $\tilde{\P}_{\hat{X}_1, \tilde{U}_1} = \brb{\slh(\tilde{\P})}_{X_1,U_1} =: \Q_1$. 
Now, up to a set with $\Q_{1}$-probability zero, if $x_1 \in \cX$ and $u_1 \in [0,1]$, we get, using Item~(\ref{bullthree}):
\begin{align*}
    \tilde{\P}_{\hat{Z}_1 \mid \hat{X}_1=x_1, \tilde{U}_1=u_1} &= 
    \tilde{\P}_{\slg\Brb{\tilde{\fhi}\brb{\slf(\hat{X}_1), \tilde{Y}_1} } \mid \hat{X}_1=x_1, \tilde{U}_1=u_1} =
    \tilde{\P}_{\slg\Brb{\tilde{\fhi}\brb{\slf(x_1), \tilde{Y}_1} } }
    \\
    &=
    \brb{\slh(\tilde{\P})}_{\fhi(x_1,Y_1)}=
    \brb{\slh(\tilde{\P})}_{\fhi(X_1,Y_1) \mid X_1=x_1, U_1=u_1} = \brb{\slh(\tilde{\P})}_{Z_1 \mid X_1=x_1, U_1=u_1} \;.
\end{align*}
So, if $A_1 \subset \cZ$ and $D \subset \cX \times [0,1]$, then
\begin{align*}
\tilde{\P}_{\hat{Z}_1, \brb{\hat{X}_1, \tilde{U}_1}}(A_1 \times D) &= \int_{D} \P_{\hat{Z}_1 \mid \hat{X}_1=x_1, \tilde{U}_1 = u_1}(A_1) \dif\P_{\hat{X}_1, \tilde{U}_1}(x_1,u_1)
\\
&=
\int_{D} \brb{\slh(\tilde{\P})}_{Z_1 \mid X_1=x_1, U_1=u_1}(A_1) \dif\brb{\slh(\tilde{\P})}_{X_1, U_1}(x_1,u_1) = \brb{\slh(\tilde{\P})}_{Z_1, (X_1, U_1)}(A_1 \times D) \,,
\end{align*}
from which it follows that $\tilde{\P}_{\hat{Z}_1, \hat{X}_1, \tilde{U}_1} = \brb{\slh(\tilde{\P})}_{Z_1, X_1, U_1}$.
By induction, suppose that for time $t\in [T-1]$ we have that
\[
\tilde{\P}_{\hat{Z}_t,\dots,\hat{Z}_1,\hat{X}_t,\dots,\hat{X}_1, \tilde{U}_t,\dots,\tilde{U}_1} = \brb{\slh(\tilde{\P})}_{Z_t,\dots,Z_1,X_t,\dots,X_1, U_t,\dots,U_1} \;.
\]
Then, using independence we have that
\[
\tilde{\P}_{\hat{Z}_t,\dots,\hat{Z}_1,\hat{X}_t,\dots,\hat{X}_1, \tilde{U}_{t+1}, \tilde{U}_t,\dots,\tilde{U}_1} = \brb{\slh(\tilde{\P})}_{Z_t,\dots,Z_1,X_t,\dots,X_1, U_{t+1} , U_t,\dots,U_1} \;.
\]
Furthermore, since $X_{t+1} = \alpha_{t+1}(U_1,\dots, U_{t+1}, Z_1, \dots, Z_t)$ and $\hat{X}_{t+1} = \alpha_{t+1} (\tilde{U}_1,\dots, \tilde{U}_{t+1}, \hat{Z}_1, \dots, \hat{Z}_t)$, we have that
\[
\tilde{\P}_{\hat{Z}_t,\dots,\hat{Z}_1,\hat{X}_{t+1},\hat{X}_t,\dots,\hat{X}_1, \tilde{U}_{t+1}, \tilde{U}_t,\dots,\tilde{U}_1} = \brb{\slh(\tilde{\P})}_{Z_t,\dots,Z_1,X_{t+1},X_t,\dots,X_1, U_{t+1}, U_t,\dots,U_1} =: \Q_{t+1}\;.
\]
Now, up to a set with $\Q_{t+1}$-probability zero, if $x_1, \dots, x_{t+1} \in \cX$, $u_1, \dots, u_{t+1} \in [0,1]$, and $z_1, \dots, z_t \in \cZ$, by the $\tilde{\P}$-independence of $\tilde{Y}_1, \ldots, \tilde Y_{t+1}$, Item~\ref{bullthree}, and the $\slh(\tilde \P)$-independence of $Y_1, \ldots, Y_{t+1}$, we have
\begin{multline*}
    \tilde{\P}_{\hat{Z}_{t+1} \mid \hat{Z}_t = z_t,\dots,\hat{Z}_1 = z_1,\hat{X}_{t+1} =x_{t+1},\dots,\hat{X}_1 = x_1, \tilde{U}_{t+1} = u_{t+1},\dots,\tilde{U}_1=u_1}
\\
\begin{aligned}
&=
    \tilde{\P}_{\slg\Brb{\tilde{\fhi}\brb{\slf(\hat{X}_{t+1}),\tilde{Y}_{t+1}}}\mid \hat{Z}_t = z_t,\dots,\hat{Z}_1 = z_1,\hat{X}_{t+1} =x_{t+1},\dots,\hat{X}_1 = x_1, \tilde{U}_{t+1} = u_{t+1},\dots,\tilde{U}_1=u_1} 
=
    \tilde{\P}_{\slg\Brb{\tilde{\fhi}\brb{\slf(x_{t+1}),\tilde{Y}_{t+1}}}}
\\
&= \brb{\slh(\tilde{\P})}_{\fhi(x_{t+1},Y_{t+1})} =
\brb{\slh(\tilde{\P})}_{\fhi(X_{t+1},Y_{t+1}) \mid Z_t = z_t,\dots,Z_1 = z_1,X_{t+1} =x_{t+1},\dots,X_1 = x_1, U_{t+1} = u_{t+1},\dots,U_1=u_1}
\\
&= \brb{\slh(\tilde{\P})}_{Z_{t+1} \mid Z_t = z_t,\dots,Z_1 = z_1,X_{t+1} =x_{t+1},\dots,X_1 = x_1, U_{t+1} = u_{t+1},\dots,U_1=u_1} \;.
\end{aligned}
\end{multline*}
So, if $A_{t+1}\subset \cZ, D \subset \cZ^t \times \cX^{t+1}\times[0,1]^{t+1}$, we have that
\begin{multline*}
\tilde{\P}_{\hat{Z}_{t+1},\brb{\hat{Z}_{t:1},\hat{X}_{t+1:1},\tilde{U}_{t+1:1} } }(A_{t+1}\times D)
\\
\begin{aligned}
&=
\int_{D} \tilde{\P}_{\hat{Z}_{t+1} \mid \hat{Z}_{t:1} = z_{t:1},\hat{X}_{t+1:1} =x_{t+1:1}, \tilde{U}_{t+1:1} = u_{t+1:1}}(A_{t+1}) \dif\tilde{\P}_{\hat{Z}_{t:1},\hat{X}_{t+1:1}, \tilde{U}_{t+1:1}}(z_{t:1}, x_{t+1:1}, u_{t+1:1})
\\
&=
\int_{D} \brb{\slh(\tilde{\P})}_{Z_{t+1} \mid Z_{t:1} = z_{t:1},C_{t+1:1} =x_{t+1:1}, U_{t+1:1} = u_{t+1:1}}(A_{t+1}) \dif\brb{\slh(\tilde{\P})}_{Z_{t:1},X_{t+1:1}, U_{t+1:1}}(z_{t:1}, x_{t+1:1}, u_{t+1:1})
\\
&=
\brb{\slh(\tilde{\P})}_{Z_{t+1},\brb{ Z_{t:1},X_{t+1:1}, U_{t+1:1} }}(A_{t+1}\times D) \;,
\end{aligned}
\end{multline*}
from which follows that $\tilde{\P}_{\hat{Z}_{t+1},\dots,\hat{Z}_1,\hat{X}_{t+1},\dots,\hat{X}_1, \tilde{U}_{t+1},\dots,\tilde{U}_1} = \brb{\slh(\tilde{\P})}_{Z_{t+1},\dots,Z_1,X_{t+1},\dots,X_1, U_{t+1},\dots,U_1}$.
In particular, for each $t \in [T]$ we have that $\tilde{\P}_{\hat{X}_t} = \brb{\slh(\tilde{\P})}_{X_t}$. 
Hence, using the $\slh(\tilde \P)$-independence of $Y_1,\ldots,Y_T$, Item~(\ref{bulltwo}), and the $\tilde{\P}$-independence of  $\tilde{Y}_1,\ldots,\tilde{Y}_T$, we get
\begin{align*}
\sum_{t=1}^T \E_{\slh(\tilde{\P})} \bsb{\rho(X_t,Y_t) } &= \sum_{t=1}^T \int_{\cX}\E_{\slh(\tilde{\P})} \bsb{\rho(x,Y_t)} \dif  \brb{\slh(\tilde{\P})}_{X_t}(x)
\\
&\le
\sum_{t=1}^T \int_{\cX}\E_{\tilde{\P}}[\tilde{\rho}(\slf(x),\tilde{Y}_t)] \dif  \brb{\slh(\tilde{\P})}_{X_t}(x)
\\
&= 
\sum_{t=1}^T \int_{\cX}\E_{\tilde{\P}}[\tilde{\rho}(\slf(x),\tilde{Y}_t)] \dif  \tilde{\P}_{\hat{X}_t}(x)
\\
&=
\sum_{t=1}^T \E_{\tilde{\P}} \bsb{\tilde \rho(\slf(\hat{X}_t),\tilde{Y}_t) } =
\sum_{t=1}^T \E_{\tilde{\P}} \bsb{\tilde \rho(\tilde{X}_t,\tilde{Y}_t) } \;.
\end{align*}
Then, using Item~(\ref{bullone}), we have
\begin{align*}
R_T^{\slh(\tilde{\P})}(\alpha) &= \sup_{x \in \cX} \bbrb{ \sum_{t=1}^T \E_{\slh(\tilde{\P})} \bsb{\rho(x,Y_t) } - \sum_{t=1}^T \E_{\slh(\tilde{\P})} \bsb{\rho(X_t,Y_t) }}
\\
&
\ge
\sup_{\tilde x \in \tilde \cX} \bbrb{ \sum_{t=1}^T \E_{\tilde{\P}} \bsb{\tilde \rho(\tilde x,\tilde{Y}_t) } - \sum_{t=1}^T \E_{\tilde{\P}} \bsb{\tilde \rho(\tilde{X}_t,\tilde{Y}_t) }} =
R_T^{\tilde{\P}}(\tilde{\alpha}) \;.
\end{align*}
Since $\tilde \P$ was arbitrary, we get
\[
    \Rs_T(\tilde{\sG}) 
= 
    \inf_{\beta \in \sA(\tilde{\sG})} R_T^{\tilde \sS}(\beta) 
\le 
    R_T^{\tilde \sS}(\tilde{\alpha}) 
= 
    \sup_{\tilde{\P} \in \tilde{\sS} } R_T^{\tilde{\P}}(\tilde{\alpha}) \le \sup_{\tilde{\P} \in \tilde{\sS} } R_T^{\slh(\tilde{\P})}(\alpha) \le \sup_{\P \in \sS } R_T^{\P}(\alpha) = R_T^{\sS}(\alpha) \;,
\]
and since $\alpha$ was arbitrary, we get
\[
\Rs_T(\tilde{\sG}) \le \inf_{\alpha \in \sA(\sG)} R_T^{\sS}(\alpha) = \Rs_T(\sG) \;. \qedhere
\]
\end{proof}

We now prove the Simulation lemma we introduced in \cref{s:keylemmas} showing how to get rid of uninformative feedback.

\lsimulation*

\begin{proof}
For each number $a\in [0,1]$, fix a binary representation $0.a_1 a_2 a_3 \ldots$ of $a$
and define $\xi(a) := 0. a_1 a_3 a_5 \ldots$, $\zeta(a) := 0. a_2 a_4 a_6 \ldots$.
Note that the two resulting functions $\xi,\zeta \colon [0,1] \to [0,1]$ are $\mu_L$-independent with common (uniform) push-forward distribution $(\mu_L)_\xi = \mu_L = (\mu_L)_\zeta$.

Let $(Y_t)_{t \in \N}, (U_t)_{t \in \N}$ be the sequences of adversary's actions and player's randomization for the sequential game $\sG$ and note that they are also the same for the sequential game $\tilde{\sG}$.
For each $t \in \N$ define $\beta_t \colon \cX \times \cV \times [0,1] \to \cZ$ via
\[
(x,v,u) \mapsto
\begin{cases}
\brb{v,\psi_{t,x}(u)} \;, &\text{ if } x \in \cR \;, \\
\gamma_{t,x}(u) \;,       &\text{ if } x \in \cU \;,
\end{cases}
\]
if $t \le T$ and in an arbitrary manner if $t\ge T+1$. Fix $\alpha = (\alpha_t)_{t\in\N} \in \sA(\sG)$. Let $(X_t)_{t \in \N}, (Z_t)_{t \in \N}$ be the sequences of player's actions and feedback associated to the strategy $\alpha$.

Fix $(u_t)_{t \in \N} \subset [0,1]$ and $(v_t)_{t\in \N} \subset \cV$. Define by induction (on $t$) the sequences $(x_t)_{t \in \N}$ and $(z_t)_{t\in \N}$ via the relationships
\[
x_t = \alpha_t\brb{\xi(u_1),\dots,\xi(u_{t}), z_1, \dots, z_{t-1}}, \qquad z_t = \beta_t\brb{x_t, v_t, \zeta(u_t)}.
\]
Note that for each $t \in \N$, we have that $x_t$ depends only on $u_1,\dots, u_t, v_1,\dots, v_{t-1}$, so we can define
\[
\tilde{\alpha}_{t}(u_1,\dots, u_t, v_1,\dots, v_{t-1}) := x_t.
\]
Being $(u_t)_{t \in \N}$ and $(v_t)_{t\in \N}$ arbitrary, this defines a sequence of functions $(\tilde{\alpha}_t)_{t \in \N}$ such that, for all $t\in \N$,
\[
\tilde{\alpha}_{t} \colon [0,1]^t \times \cV^{t-1} \to \cX
\]
i.e., $\tilde{\alpha}:=(\tilde{\alpha}_t)_{t \in \N} \in \sA(\tilde{\sG})$. Let $(\tilde{X}_t)_{t \in \N}$ and $(\tilde{V}_t)_{t \in \N}$ be respectively the sequence of player's actions and the feedback sequence associated with the strategy $\tilde{\alpha}$. For each $t \in \N$, define also $\tilde{Z}_t := \beta_{t}\brb{\tilde{X}_t, \tilde{V}_t, \zeta(U_t)}$. Note that for each $t \in \N$ it holds that $\tilde{X}_t = \alpha_t \brb{\xi(U_1), \dots, \xi(U_t), \tilde{Z}_1, \dots, \tilde{Z}_{t-1}}$.

Fix a scenario $\P \in \sS$. 
Note first that $\P_{\xi(U_1)} = \P_{U_1}$, and since $X_1 = \alpha_1(U_1)$ and $\tilde{X}_1 = \tilde{\alpha}_1(U_1) = \alpha_1\brb{\xi(U_1)}$, we also have that $\P_{\tilde{X}_1, \xi(U_1)} = \P_{X_1, U_1} =: \Q_1$. 
Now, up to a set with $\Q_{1}$-probability zero, if $x_1 \in \cX$ and $u_1 \in [0,1]$, using Items~(\ref{buone})~and~(\ref{butwo}), we have that
\begin{multline*}
    \P_{\tilde{Z}_1 \mid \tilde{X}_1=x_1, \xi(U_1) = u_1} 
=
    \P_{\beta_1\brb{\tilde{X}_1, \tilde{\fhi}(\tilde{X}_1,Y_1), \zeta(U_1)} \mid \tilde{X}_1=x_1, \xi(U_1) = u_1}
=
    \P_{\beta_1\brb{x_1, \tilde{\fhi}(x_1,Y_1), \zeta(U_1)}}
\\
\begin{aligned}
&=
\begin{cases}
\P_{\beta_1\Brb{x_1 , \pi\brb{\fhi(x_1,Y_1)}, \zeta(U_1)} } &\text{if } x_1 \in \cR \\
\P_{\beta_1\brb{ x_1, *, \zeta(U_1) }} &\text{if } x_1 \in \cU \\
\end{cases}
=
\begin{cases}
\P_{\Brb{ \pi\brb{\fhi(x_1,Y_1)} , \psi_{1,x_1} \brb{\zeta(U_1)} } } &\text{if } x_1 \in \cR \\
\P_{\gamma_{1,x_1}\brb{\zeta(U_1)}} &\text{if } x_1 \in \cU \\
\end{cases}
\\
&=
\begin{cases}
\P_{ \pi\brb{\fhi(x_1,Y_1)} } \otimes \P_{ \psi_{1,x_1} \brb{\zeta(U_1)}} &\text{if } x_1 \in \cR \\
\P_{\gamma_{1,x_1}\brb{\zeta(U_1)}} &\text{if } x_1 \in \cU \\
\end{cases}
=
\begin{cases}
\P_{ \pi\brb{\fhi(x_1,Y_1)} } \otimes \brb{\P_{\zeta(U_1)}}_{ \psi_{1,x_1} } &\text{if } x_1 \in \cR \\
\brb{\P_{\zeta(U_1)}}_{\gamma_{1,x_1}} &\text{if } x_1 \in \cU \\
\end{cases}
\\
&=
\begin{cases}
\P_{ \pi\brb{\fhi(x_1,Y_1)} } \otimes \brb{\mu_L}_{ \psi_{1,x_1} } &\text{if } x_1 \in \cR \\
\brb{\mu_L}_{\gamma_{1,x_1}} &\text{if } x_1 \in \cU \\
\end{cases}
=
\P_{\fhi(x_1,Y_1)} 
=
    \P_{\fhi(X_1,Y_1) \mid X_1=x_1, U_1=u_1} = \P_{Z_1 \mid X_1=x_1, U_1=u_1} \;.
\end{aligned}
\end{multline*}
So, if $A_1 \subset \cZ$ and $D \subset \cX \times [0,1]$, then
\begin{align*}
\P_{\tilde{Z}_1, \brb{\tilde{X}_1, \xi(U_1)}}(A_1 \times D) &= \int_{D} \P_{\tilde{Z}_1 \mid \tilde{X}_1=x_1, \xi(U_1) = u_1}(A_1) \dif\P_{\tilde{X}_1, \xi(U_1)}(x_1,u_1)
\\
&=
\int_{D} \P_{Z_1 \mid X_1=x_1, U_1 = u_1}(A_1) \dif\P_{X_1, U_1}(x_1,u_1) = \P_{Z_1, (X_1, U_1)}(A_1 \times D) \;,
\end{align*}    
from which it follows that $\P_{\tilde{Z}_1, \tilde{X}_1, \xi(U_1)} = \P_{Z_1, X_1, U_1}$.
By induction, suppose that for $t\in [T-1]$ we have that
\[
\P_{\tilde{Z}_t,\dots,\tilde{Z}_1,\tilde{X}_t,\dots,\tilde{X}_1, \xi(U_t),\dots,\xi(U_1)} = \P_{Z_t,\dots,Z_1,X_t,\dots,X_1, U_t,\dots,U_1} \;.
\]
Then, using independence we have that
\[
\P_{\tilde{Z}_t,\dots,\tilde{Z}_1,\tilde{X}_t,\dots,\tilde{X}_1, \xi(U_{t+1}),\xi(U_t),\dots,\xi(U_1)} = \P_{Z_t,\dots,Z_1,X_t,\dots,X_1, U_{t+1},U_t,\dots,U_1} \;.
\]
Furthermore, since $X_{t+1} = \alpha_{t+1}(U_1,\dots, U_{t+1}, Z_1, \dots, Z_t)$ and 
\[
    \tilde{X}_{t+1} 
= 
    \tilde{\alpha}_{t+1} (U_1,\dots, U_{t+1}, \tilde{V}_1, \dots, \tilde{V}_t) = \alpha_{t+1} (\xi(U_1),\dots, \xi(U_{t+1}), \tilde{Z}_1, \dots, \tilde{Z}_t)\;,
\]
we have that
\[
\P_{\tilde{Z}_t,\dots,\tilde{Z}_1,\tilde{X}_{t+1},\tilde{X}_t,\dots,\tilde{X}_1, \xi(U_{t+1}),\xi(U_t),\dots,\xi(U_1)} = \P_{Z_t,\dots,Z_1,X_{t+1},X_t,\dots,X_1, U_{t+1},U_t,\dots,U_1} =: \Q_{t+1} \;.
\]
Now, up to a set with $\Q_{t+1}$-probability zero, if $x_1, \dots, x_{t+1} \in \cX$, $u_1, \dots, u_{t+1} \in [0,1]$ and $z_1, \dots, z_t \in \cZ$, using the $\P$-independence of $Y_1,\ldots,Y_{t+1}$ and Items~\eqref{buone}--\eqref{butwo}, we have that
\begin{multline*}
    \P_{\tilde{Z}_{t+1} \mid \tilde{Z}_t = z_t,\dots,\tilde{Z}_1 = z_1,\tilde{X}_{t+1} =x_{t+1},\dots,\tilde{X}_1 = x_1, \xi(U_{t+1}) = u_{t+1},\dots,\xi(U_1)=u_1}
\\
\begin{aligned}
&=
    \P_{\beta_{t+1}\brb{\tilde{X}_{t+1}, \tilde{\fhi}(\tilde{X}_{t+1},Y_{t+1}), \zeta(U_{t+1})} \mid \tilde{Z}_t = z_t,\dots,\tilde{Z}_1 = z_1,\tilde{X}_{t+1} =x_{t+1},\dots,\tilde{X}_1 = x_1, \xi(U_{t+1}) = u_{t+1},\dots,\xi(U_1)=u_1}
\\
&=
    \P_{\beta_{t+1}\brb{x_{t+1}, \tilde{\fhi}(x_{t+1},Y_{t+1}), \zeta(U_{t+1})}}
=
\begin{cases}
\P_{\beta_{t+1}\Brb{x_{t+1} , \pi\brb{\fhi(x_{t+1},Y_{t+1})}, \zeta(U_{t+1})} } &\text{if } x_{t+1} \in \cR \\
\P_{\beta_{t+1}\brb{ x_{t+1}, *, \zeta(U_{t+1}) }} &\text{if } x_{t+1} \in \cU \\
\end{cases}
\\
&=
\begin{cases}
\P_{\Brb{ \pi\brb{\fhi(x_{t+1},Y_{t+1})} , \psi_{{t+1},x_{t+1}} \brb{\zeta(U_{t+1})} } } &\text{if } x_{t+1} \in \cR \\
\P_{\gamma_{{t+1},x_{t+1}}\brb{\zeta(U_{t+1})}} &\text{if } x_{t+1} \in \cU \\
\end{cases}
\\
&
=
\begin{cases}
\P_{ \pi\brb{\varphi(x_{t+1},Y_{t+1})} } \otimes \P_{ \psi_{{t+1},x_{t+1}} \brb{\zeta(U_{t+1})}} &\text{if } x_{t+1} \in \cR \\
\P_{\gamma_{{t+1},x_{t+1}}\brb{\zeta(U_{t+1})}} &\text{if } x_{t+1} \in \cU \\
\end{cases}
\\
&=
\begin{cases}
\P_{ \pi\brb{\varphi(x_{t+1},Y_{t+1})} } \otimes \brb{\P_{\zeta(U_{t+1})}}_{ \psi_{{t+1},x_{t+1}} } &\text{if } x_{t+1} \in \cR \\
\brb{\P_{\zeta(U_{t+1})}}_{\gamma_{{t+1},x_{t+1}}} &\text{if } x_{t+1} \in \cU \\
\end{cases}
\\
&
=
\begin{cases}
\P_{ \pi\brb{\varphi(x_{t+1},Y_{t+1})} } \otimes \brb{\mu_L}_{ \psi_{{t+1},x_{t+1}} } &\text{if } x_{t+1} \in \cR \\
\brb{\mu_L}_{\gamma_{{t+1},x_{t+1}}} &\text{if } x_{t+1} \in \cU \\
\end{cases}
\\
&=
\P_{\varphi(x_{t+1},Y_{t+1})} 
=
\P_{\varphi(X_{t+1},Y_{t+1}) \mid Z_t = z_t,\dots,Z_1 = z_1,X_{t+1} =x_{t+1},\dots,X_1 = x_1, U_{t+1} = u_{t+1},\dots,U_1=u_1} 
\\
&= \P_{Z_{t+1} \mid Z_t = z_t,\dots,Z_1 = z_1,X_{t+1} =x_{t+1},\dots,X_1 = x_1, U_{t+1} = u_{t+1},\dots,U_1=u_1} \;.
\end{aligned}
\end{multline*}

So, if $A_{t+1}\subset \cZ, D \subset \cZ^t \times \cX^{t+1}\times[0,1]^{t+1}$, we have that
\begin{multline*}
\P_{\tilde{Z}_{t+1},\brb{\tilde{Z}_{t}\dots,\tilde{Z}_1,\tilde{X}_{t+1},\dots,\tilde{X}_1, \xi(U_{t+1}),\dots,\xi(U_1)}}(A_{t+1}\times D)
\\
\begin{aligned}
&=
\int_{D} \P_{\tilde{Z}_{t+1} \mid \tilde{Z}_{t:1} = z_{t:1},\tilde{X}_{t+1:1} =x_{t+1:1}, \brb{\xi(U_{t+1}),\dots, \xi(U_{1})} = u_{t+1:1}}(A_{t+1}) \dif \Q_{t+1}
(z_{t:1}, x_{t+1:1}, u_{t+1:1})
\\
&=
\int_{D} \P_{Z_{t+1} \mid Z_{t:1} = z_{t:1},X_{t+1:1} =x_{t+1:1}, U_{t+1:1} = u_{t+1:1}}(A_{t+1})
\dif \Q_{t+1}
(z_{t:1}, x_{t+1:1}, u_{t+1:1})
\\
&=
\P_{Z_{t+1},\brb{Z_{t},\dots,Z_1,X_{t+1},\dots,X_1, U_{t+1},\dots,U_1}}(A_{t+1}\times D)
\\
\end{aligned}
\end{multline*}
from which it follows that $\P_{\tilde{Z}_{t+1:1}, \tilde{X}_{t+1:1}, \brb{\xi(U_{t+1}),\dots,\xi(U_1)} } = \P_{Z_{t+1:1}, X_{t+1:1}, U_{t+1:1}}$.
In particular, for each $t \in [T]$ we have that $\P_{X_t} = \P_{\tilde{X}_t}$. So, for each $t \in [T]$, using the $\P$-independence of $Y_1,\ldots,Y_{t}$, we have that
\[
\P_{X_t,Y_t} = \P_{X_t}\otimes\P_{Y_t} = \P_{\tilde{X}_t}\otimes\P_{Y_t} = \P_{\tilde{X}_t,Y_t} \;,
\]
and then
\[
\E_\P \bsb{\rho(X_t,Y_t)} = \E_{\P_{X_t,Y_t}} \bsb{\rho} = \E_{\P_{\tilde{X}_t,Y_t}} \bsb{\rho} = \E_\P \bsb{\rho(\tilde{X}_t,Y_t)} \;.
\]
In conclusion
\begin{multline*}
    R_T^\P(\alpha) = \sup_{x \in \cX} \E_\P \lsb{\sum_{t=1}^T \rho(x,Y_t) - \sum_{t=1}^T \rho(X_t,Y_t)} = \sup_{x \in \cX} \bbrb{\sum_{t=1}^T \E_\P \lsb{ \rho(x,Y_t)}-\sum_{t=1}^T\E_\P \lsb{\rho(X_t,Y_t)}} 
\\
= \sup_{x \in \cX} \bbrb{\sum_{t=1}^T \E_\P \lsb{ \rho(x,Y_t)}-\sum_{t=1}^T\E_\P \lsb{\rho(\tilde{X}_t,Y_t)}} =  \sup_{x \in \cX} \E_\P \lsb{\sum_{t=1}^T \rho(x,Y_t) - \sum_{t=1}^T \rho(\tilde{X}_t,Y_t)} = R_T^\P(\tilde{\alpha}) \;.
\end{multline*}
Since $\P$ was arbitrary, it follows that $R_T^{\sS}(\alpha) = R_T^{\sS}(\tilde{\alpha})$. Since $\alpha$ was arbitrary, it follows that
\[
\Rs_T(\sG) = \inf_{\alpha \in \sA(\sG)} R_T^{\sS}(\alpha) = \inf_{\alpha \in \sA(\sG)} R_T^{\sS}(\tilde{\alpha}) \ge \inf_{\alpha' \in \sA(\tilde{\sG})} R_T^{\sS}(\alpha') = \Rs_T(\tilde{\sG}) \;.
\]
\end{proof}

\section{\texorpdfstring{$\sqrt{T}$}{sqrt(T)} Lower Bound Under Full-Feedback (iv+bd) \tccheck}
\label{s:lower-full}

In this section, we prove that in the full-feedback case, no strategy can beat the $\sqrt{T}$ rate that we proved in 
\Cref{thm:upper_full} when the seller/buyer pair $(S_t,B_t)$ is drawn i.i.d. from an unknown fixed distribution, not even under the further assumptions that the valuations of the seller and buyer are independent of each other and have bounded densities.

The idea of the proof is to build a family of scenarios $\P^{\pm\e}$ parameterized by $\e \in [0,1]$, like in \cref{f:root-t-full}.
The only way to avoid suffering linear regret in a scenario $\P^{\pm\e}$ is to identify the sign of $\pm \e$.
Leveraging the Embedding and Simulation lemmas (\cref{l:embedding,l:simulation}), this construction leads to a reduction to a two-armed bandit problem, which has a know lower bound on the regret of order $\sqrt{T}$.

\begin{theorem*}[\cref{thm:lower-full}, restated]
In the full-feedback stochastic (iid) setting with independent valuations (iv) and densities bounded by a constant $M\ge 4$ (bd), for all horizons $T\in \N$, the minimax regret satisfies
\[
    \Rs_T \ge \frac{1}{160} \sqrt{T} \;.
\]
\end{theorem*}

\begin{proof}
Fix any horizon $T\in \N$ and $M \ge 4$. 
Recalling \cref{s:biltrad-setting-appe}, the full-feedback stochastic (iid) setting with independent valuations (iv) and densities bounded (bd) by $M$ is a game $\sG := (\cX, \cY, \cZ, \rho, \varphi, \sP)$, where $\cX = [0,1]$, $\cY = [0,1]^2$, $\cZ = [0,1]^2$, $\rho = \gft$, $\fhi\colon \brb{p, (s,b)} \mapsto (s,b)$, and $\sP = \sPivbd$.
Define, for each $\e \in [-1,1]$, the densities $f_{S,\e} = 2(1+\e)\I_{[0,\frac{1}{4}]} + 2(1-\e)\I_{[\frac{1}{2},\frac{3}{4}]}$ and $f_B = 2\I_{[\frac{1}{4},\frac{1}{2}]\cup[\frac{3}{4},1]}$. 
Fix the adversary's behavior $\sP_1$ as the subset of $\sP$ whose elements have the form $\bmu_\e := \otimes_{t \in \N} (f_{S,\e}\mu_L \otimes f_B\mu_L)$, for some $\e \in [-1,1]$. 
Since $\sP_1 \subset \sP$, the game $\sG_1 := (\cX, \cY, \cZ, \rho, \varphi, \sP_1)$ is easier than $\sG$ (i.e., $\Rs_T(\fG) \ge \Rs_T(\fG_1)$) by the Embedding lemma (\cref{l:embedding}) with $\slf$ and $\slg$ as the identities, and $\slh$ as the inclusion. 
Now, define $\rho_1 \colon \cX \times \cY \to [0,1]$, $\lrb{p,(s,b)} \mapsto (b-s)\I\lcb{s \le \frac{1}{4} \le b}\I\lcb{p \le \frac{1}{2}} + (b-s)\I\lcb{s \le \frac{3}{4} \le b}\I\lcb{p > \frac{1}{2}}$ and note that, defining $\fG_2 := (\cX, \cY, \cZ, \rho_1, \varphi, \fP_1)$, by the Embedding lemma with $\slf,\slg,\slh$ as the identities, we have that the game $\fG_2$ is easier than the game $\fG_1$ (i.e., $\Rs_T(\fG_1) \ge \Rs_T(\fG_2)$). 
Then, let $\cZ_3 := \{0,1\} \times \bsb{ 0, \frac{1}{4} } \times [0,1]$ and $\fhi_3 \colon \cX \times \cY \to \cZ_3$, $\brb{ p, (s,b) } \mapsto \brb{ \I\{s \le \nicefrac{1}{4}\}, \,  s\I\{s \le \nicefrac{1}{4}\} + (s-\nicefrac{1}{2})\I \{ \nicefrac{1}{2} \le s \le \nicefrac{3}{4} \}, \, b }$.
Define the game $\sG_3 := (\cX, \cY, \cZ_3, \rho_1, \fhi_3, \sP_1)$. 
By the Embedding lemma with $\slf,\slh$ as the identities and $\slg \colon \cZ_3 \to \cZ$, $(i, \tilde{s}, b) \mapsto \brb{ \tilde s i + (\nicefrac{1}{2}+ \tilde s)(1-i) , \, b }$, we have that the game $\sG_3$ is easier than the game $\sG_2$ (i.e., $\Rs_T(\fG_2) \ge \Rs_T(\fG_3)$). 
Next, let $\varphi_4 \colon \cX \times \cY \to \cZ_3$, $\brb{p,(s,b)} \mapsto \I\{s \le \frac{1}{4}\}$, and define the game $\fG_4 := (\cX, \cY, \cZ_3, \rho_1, \varphi_4, \fP_1)$.
Let $(Y_t)_{t\in \N}$ be the adversary's actions in $\sG_4$.
A tedious computation verifies that for all $t\in \N$, $p\in \cX$, and scenarios $\P$ of game $\sG_3$, $\P_{\fhi_3 (p, Y_t)} = \P_{\pi ( \fhi_3(p, Y_t) ) } \otimes (\nu \otimes f_B \mu_L)$, where $\pi \colon \cZ_3 \to \{0,1\}$ is the projection on the first component $\{0,1\}$ of $\cZ_3$ and $\nu$ is the uniform distribution on $[0,\nicefrac{1}{4}]$.
By the well-known Skorokhod representation \cite[Section 17.3]{williams1991probability}, there exists $\psi \colon [0,1] \to [0,\nicefrac{1}{4}] \times [0,1]$ such that $\nu \otimes f_B \mu_L = (\mu_L)_\psi$.
Thus, by the Simulation lemma (\cref{l:simulation}) with $\cR = \cX$ and $\cU = \varnothing$, the game $\fG_4$ is easier than $\fG_3$ (i.e., $\Rs_T(\fG_3) \ge \Rs_T(\fG_4)$).
Finally, consider the game $\fG_5 := \brb{ \{1,2\}, \{1,2\}, \{0,1\}, \rho_5, \fhi_5, \sP_5 }$, where in matrix notation, $\rho_5 = \bsb{ \rho_5(i,j) }_{i,j \in \{1,2\}}$ and $\fhi_5 = \bsb{ \fhi_5(i,j) }_{i,j \in \{1,2\}}$ are given by
\[
\rho_5:=\begin{bmatrix}
1/2 & 3/8   \\
3/8 & 1/2 
\end{bmatrix} \;,
\qquad
\fhi_5:=\begin{bmatrix}
1 & 0   \\
1 & 0 
\end{bmatrix} \;,
\]
and $\sP_5$ is the set of all measures $\tilde \bmu_\e$ of the form $\tilde \bmu_\e = \otimes_{t=1}^\infty \brb{ \frac{1+\e}{2}\delta_{1}+\frac{1-\e}{2} \delta_{2}}$ for some $\e \in [-1,1]$, where $\delta_i$ is the Dirac measure at $i \in \{1,2\}$.
Thus, letting $\sS_4$ and $\sS_5$ be the two sets of scenarios in games $\sG_4$ and $\sG_5$ respectively (note that $\sS_4$ coincides with the set of scenarios of $\sG_1$) and using again the Embedding lemma, this time with $\slf\colon [0,1] \to \{1,2\}$, $p\mapsto \I\{p\le \nicefrac{1}{2}\} + 2 \I\{p > \nicefrac{1}{2}\}$, $\slg\colon \{0,1\} \to \{0,1\}$, $i\mapsto i$, and $\slh \colon \sS_5 \to \sS_4$, $\tilde{\bmu}_\e \otimes \bmu_L \mapsto \bmu_\e \otimes \bmu_L$, we obtain that $\fG_5$ is easier than $\sG_4$ (i.e., $\Rs_T(\fG_4) \ge \Rs_T(\fG_5)$).
This last game $\sG_5$ is a two-armed bandit problem with gap $\Delta = \nicefrac{1}{2} - \nicefrac{3}{8} = \nicefrac{1}{8}$, whose minimax regret is known to be lower bounded by $\frac{1}{8} \brb{ \frac{1}{20} \sqrt{T} }$ \cite{Nicolo06,BubeckC12}.
In conclusion, we proved that $\Rs_T(\sG)\ge \Rs_T(\sG_5) \ge \frac{1}{160}\sqrt{T}$.
\end{proof}

% \newpage

\section{Proof of \texorpdfstring{$T^{2/3}$}{T\^{}(2/3)} Lower Bound Under Realistic Feedback (iv+bd) \tccheck}
\label{s:proof-t-two-thrid-lower-bound-appe}

In this section we give a detailed proof of our $T^{2/3}$ lower bound of \cref{sec:candidate}.
which hinges in a non-trivial way on our Embedding and Simulation lemmas (\cref{l:embedding,l:simulation}).
We denote Bernoulli distributions with parameter $\lambda$ by $\ber_\lambda$. 

\begin{theorem*}[\cref{thm:lower-real-iv+bd}]
In the realistic-feedback stochastic (iid) setting with independent valuations (iv) and densities bounded by a constant $M\ge 24$ (bd), for all horizons $T \in \N$, the minimax regret satisfies
\[
    \Rs_T \ge \frac{11}{672} T^{2/3} \;.
\]
\end{theorem*}

\begin{proof}
Fix an arbitrary horizon $T\in \N$ and any $M \ge 24$. 
Recalling \cref{s:biltrad-setting-appe}, the realistic-feedback stochastic (iid) setting with independent valuations (iv) and densities bounded (bd) by $M$ is a game $\sG := (\cX, \cY, \cZ, \rho, \varphi, \sP)$, where $\cX = [0,1]$, $\cY = [0,1]^2$, $\cZ = \{0,1\}^2$, $\rho = \gft$, $\fhi\colon \brb{p, (s,b)} \mapsto \brb{ \I\{s\le p\},\,\I\{p\le b\} }$, and $\sP = \sPivbd$.
The idea of the proof is to build a sequence of games, each one easier than the former, the last of which has a known lower bound on its minimax regret.
In the first step we limit the adversary's behavior to a parametric family which is easily manageable and well-represents the difficulty of the problem (see \cref{f:t-two-third-lower-bound}).
In the second step, we increase the reward of suboptimal actions in order to have only three possible expected-reward values in each scenario.
In the third and fifth steps we increase the feedback, presenting it in a way that highlights that only its first component is informative.
In step four and six, we simulate-away the uninformative parts of the feedback.
Finally, in step 7 we show that the resulting game is harder than a known partial monitoring game with minimax regret of order at least $T^{2/3}$.

\paragraph{Step 1}
Let $\tht := \nicefrac{1}{48}$. 
Define the following densities of the seller and buyer, respectively, by
\begin{align*}
    f_{S,\e}
&
:=
    \frac{1}{4\tht} \lrb{
    (1+\e) \I_{[0,\tht]}
+ 
    (1-\e) \I_{\lsb{ \frac{1}{6}, \frac{1}{6} + \tht }} 
+
    \I_{\lsb{ \frac{1}{4}, \frac{1}{4} + \tht }}
+
    \I_{\lsb{ \frac{2}{3}, \frac{2}{3} + \tht }}
    }\;,
    \forall \e \in [-1,1] \;,
    \tag{\text{red/blue in \cref{f:t-two-third-lower-bound}}}
\\
    f_B
&
:=
    \frac{1}{4\tht} \lrb{
    \I_{\lsb{ \frac{1}{3}-\tht,\,\frac{1}{3} }}
+ 
    \I_{\lsb{ \frac{3}{4} - \tht,\, \frac{3}{4} }} 
+
    \I_{\lsb{ \frac{5}{6} - \tht,\, \frac{5}{6} }}
+
    \I_{\lsb{ 1-\tht,\, 1 }}
    } \;.
    \tag{\text{green in \cref{f:t-two-third-lower-bound}}}
\end{align*}
Define $\fP_1$ as the subset of $\fP$ whose elements have the form $\bmu_\e := \otimes_{t \in \N} (f_{S,\e}\mu_L \otimes f_B\mu_L)$ for $\e \in [-1,1]$.
Since $\sP_1 \subset \sP$, the game $\sG_1 := (\cX, \cY, \cZ, \rho, \varphi, \sP_1)$ is easier than $\sG$ (i.e., $\Rs_T(\fG) \ge \Rs_T(\fG_1)$) by the Embedding lemma (\cref{l:embedding}) with $\slf$ and $\slg$ as the identities, and $\slh$ as the inclusion.
\paragraph{Step 2}
Define 
$\rho_2\colon \cX \times \cY \to [0,1]$, 
$\brb{ p, (s,b) } \mapsto 
    \gft\brb{ \frac{1}{6} + \tht, (s,b) } \I\bcb{ p < \frac{1}{4} } 
    + \gft\brb{ \frac{1}{4} + \tht, (s,b) } \I\bcb{ \frac{1}{4} \le p < \frac{1}{3} } 
    + \gft\brb{ \frac{2}{3} + \tht, (s,b) } \I\bcb{ \frac{1}{3} < p  }
$.
By the Embedding lemma with $\slf$, $\slg$, and $\slh$ as the identities, we have that the game $\sG_2 := (\cX, \cY, \cZ, \rho_2, \varphi, \sP_1)$ is easier than $\sG_1$ (i.e., $\Rs_T(\fG_1) \ge \Rs_T(\fG_2)$).

\paragraph{Step 3}
Define $\cZ_3 := \bcb{0, \frac{1}{6}, \frac{1}{4}, \frac{2}{3}} \times [0,\tht] \times \{0,1\} \times \{0,1\} \times \cX$ and
$\fhi_3 \colon \cX \times \cY \to \cZ_3$, 
\[
    \brb{ p, (s,b) }
    \mapsto
    \begin{cases}
        \brb{ \eta(s), s-\eta(s), 0, \I\{p \le b\}, p } \;,
    & 
        \text{ if } p < \frac{1}{4} \;,
    \\
        \brb{ 0, 0, \I\{s \le p\}, \I\{p \le b\}, p } \;,
    & 
        \text{ if } p \ge \frac{1}{4} \;,
    \end{cases}
\]
where $\eta \colon [0,1] \to \bcb{0, \frac{1}{6}, \frac{1}{4}, \frac{2}{3}}$,
$
    s 
    \mapsto 
    \frac{1}{6} \I \bcb{ \frac{1}{6} \le s \le \frac{1}{6} +\tht }
    +
    \frac{1}{4} \I \bcb{ \frac{1}{4} \le s \le \frac{1}{4} +\tht }
    +
    \frac{2}{3} \I \bcb{ \frac{2}{3} \le s \le \frac{2}{3} +\tht }
$.
Define the game $\sG_3 := (\cX, \cY, \cZ_3, \rho_2, \fhi_3, \sP_1 )$.
By the Embedding lemma with $\slf,\slh$ as the identities and 
\[
    \slg \colon \cZ_3 \to \cZ\;,
    \quad
    (v,u,i,j,p)
    \mapsto
    \begin{cases}
    \brb{ \I\{v+u\le p\}, j }
    &
        \text{ if } p < \frac{1}{4} \;,
    \\
        (i,j) \;,
    & 
        \text{ if } p \ge \frac{1}{4} \;,
    \end{cases}
\]
we have that the game $\sG_3$ is easier than $\sG_2$ (i.e., $\Rs_T(\fG_2) \ge \Rs_T(\fG_3)$).

\paragraph{Step 4}
Let $\cZ_4 := \bcb{0, \frac{1}{6}, \frac{1}{4}, \frac{2}{3}}$
and
$
    \fhi_4 \colon \cX \times \cY \to \cZ_4
$, 
$
    \brb{ p, (s,b) }
    \mapsto
    \eta(s) \I \lcb{ p < \frac{1}{4} } 
$.
Define the game $\sG_4 := (\cX, \cY, \cZ_4, \rho_2, \fhi_4, \sP_1)$.
Let $(Y_t)_{t\in \N} = (S_t,B_t)_{t\in \N}$ be the adversary's actions in $\sG_4$,
$E := \bsb{0, \tht} \cup \bsb{\frac{1}{6}, \frac{1}{6} + \tht} \cup \bsb{\frac{1}{4}, \frac{1}{4} + \tht} \cup \bsb{\frac{2}{3}, \frac{2}{3} + \tht}$ and $F := \bsb{\frac{1}{3} - \tht, \frac{1}{3}} \cup \bsb{\frac{3}{4} - \tht, \frac{3}{4}} \cup \bsb{\frac{5}{6} - \tht, \frac{5}{6}} \cup \bsb{1 - \tht, 1}$. 
A long and tedious computation verifies that for all $t\in \N$, 
\begin{itemize}
    \item for each $p\in [0,\nicefrac{1}{4})$ and any scenario $\P$ of game $\sG_3$, $\P_{\fhi_3 (p, Y_t)} = \P_{ \eta(S_t) } \otimes (\nu \otimes \delta_0 \otimes \ber_{\lambda_{F,p}} \otimes \delta_p)$, where $\nu$ is the uniform distribution on $[0,\tht]$ and $ \lambda_{F,p} := \frac{1}{4\tht} \mu_L \bsb{ [p,1] \cap F } $.
    By the well-known Skorokhod representation \cite[Section 17.3]{williams1991probability}, there exists $\psi_p \colon [0,1] \to [0,\tht] \times \{0,1\} \times \{0,1\} \times \cX$ such that $\nu \otimes \delta_0 \otimes \ber_{\lambda_{F,p}} \otimes \delta_p = (\mu_L)_{\psi_p}$;
    \item for each $p\in [\nicefrac{1}{4},1]$ and any scenario $\P$ of game $\sG_3$, $\P_{\fhi_3 (p, Y_t)} = \delta_0 \otimes \delta_0 \otimes \ber_{\lambda_{E,p}} \otimes \ber_{\lambda_{F,p}} \otimes \delta_p$, where $\lambda_{E,p} := \frac{1}{4\tht} \mu_L \bsb{ [0,p] \cap E }$ and $ \lambda_{F,p} := \frac{1}{4\tht} \mu_L \bsb{ [p,1] \cap F } $.
    By the Skorokhod representation, there exists $\gamma_p \colon [0,1] \to \cZ_3$ such that $\delta_0 \otimes \delta_0 \otimes \ber_{\lambda_{E,p}} \otimes \ber_{\lambda_{F,p}} \otimes \delta_p = (\mu_L)_{\gamma_p}$.
\end{itemize}
Thus, by the Simulation lemma (\cref{l:simulation}) with $\cR = [0, \nicefrac{1}{4})$ and $\cU = [\nicefrac{1}{4}, 1]$, the game $\fG_4$ is easier than $\fG_3$ (i.e., $\Rs_T(\fG_3) \ge \Rs_T(\fG_4)$).

\paragraph{Step 5}
Let 
$\cY_5 := \cY^\N$, 
$\cZ_5 := \{0,1\} \times \brb{ \N \cup \{\iop\} } \times \{0,1\} \times \cX$, 
$\rho_5 \colon \cX \times \cY_5 \to [0,1]$, 
$\brb{ p, (s_k, b_k)_{k\in\N} } \mapsto \rho_2 (p,s_1,b_1)$,
\[
    \fhi_5 \colon \cX \times \cY_5 \to \cZ_5\;,
    \quad
    \brb{ p, (s_k, b_k)_{k\in\N} } 
    \mapsto 
    \begin{cases}
        \Brb{ \I \bcb{ \eta(s_\tau) = 0 }, \tau, \I \bcb{ \eta(s_1) = \frac{1}{4} }, p } \;,
    &
        \text{ if } p \in \bigl[ 0, \frac{1}{4} \bigr) \;,
    \\
        (0,1,0,p) \;,
    &
        \text{ if } p \in \bigl[ \frac{1}{4}, 1 \bigr]\;,
    \end{cases}
\]
where $\eta$ is defined in game $\sG_3$, $\tau := \inf \bcb{ k\in \N \mid \eta( s_k ) \in \{0, \nicefrac{1}{6} \} } \in \N \cup \{\iop\}$, and $s_\iop := 0$.
Let $\sP_5$ be the set of measures on $\cY_5^\N$ of the form $\tilde \bmu_\e := \otimes_{t\in \N} \brb{ \otimes_{k\in\N} ( f_{S,\e} \mu_L \otimes f_B \mu_L ) }$ for $\e \in [-1,1]$, and define the game $\sG_5 := ( \cX, \cY_5, \cZ_5, \rho_5, \fhi_5, \sP_5 )$.
By the Embedding lemma with $\slf$ as the identity,
\[
    \slg \colon \cZ_5 \to \cZ_4 \;,
    \quad
    (z,k,j,p)
    \mapsto
    \frac{1}{6}(1-z) \I \lcb{ p < \frac{1}{4}, k = 1 }
    + \lrb{ \frac{1}{4}j + \frac{2}{3} (1-j) } \I \lcb{ p < \frac{1}{4}, k > 1 } \;,
\]
and $\slh \colon \tilde{ \bmu }_\e \otimes \bmu_L \mapsto \bmu_\e \otimes \bmu_L$, 
we have that the game $\sG_5$ is easier than $\sG_4$ (i.e., $\Rs_T(\fG_4) \ge \Rs_T(\fG_5)$).

\paragraph{Step 6}
Now, define $\pi\colon \cZ_5 \to \{0,1\}$ as the projection on the first component $\{0,1\}$ of $\cZ_5$, $\cZ_6 := \{0,1\}$, $\fhi_6 := \pi \circ \fhi_5$, and the game $\sG_6 := ( \cX, \cY_5, \cZ_6, \rho_5, \fhi_6, \sP_5)$.
Let $(\tilde Y_t)_{t\in \N}$ be the adversary's actions in $\sG_5$. 
A straightforward verification shows that for all $t\in \N$, 
\begin{itemize}
    \item for each $p\in [0,\nicefrac{1}{4})$ and any scenario $\P$ of game $\sG_5$, $\P_{\fhi_5 (p, \tilde Y_t)} = \P_{ \pi \brb{ \fhi_5 (p, \tilde Y_t) } } \otimes (\nu \otimes \delta_p)$, where $\nu$ is the unique distribution on $\brb{ \N \cup \{\iop\} } \times \{0,1\}$ such that, for all $k \in\N\cup\{\iop\}$, $j\in \{0,1\}$, $\nu \bsb{ \{ (k,j) \} } = \frac{1}{2} \I \{ k=1, j=0\} + \frac{1}{2^{k+1}} \I \{1<k<\iop\}$.
    Using again the Skorokhod representation, there exists $\psi_p \colon [0,1] \to \brb{ \N \cup \{\iop\} } \times \{0,1\} \times [0,1]$ such that $\nu \otimes \delta_p = (\mu_L)_{\psi_p}$;
    \item for each $p\in [\nicefrac{1}{4},1]$ and any scenario $\P$ of game $\sG_5$, $\P_{\fhi_5 (p, \tilde Y_t)} = \delta_{(0,1,0,p)} = (\mu_L)_{\gamma_p}$, where $\gamma_p \colon [0,1] \to \cZ_5$, $\lambda \mapsto (0,1,0,p)$.
\end{itemize}
Thus, by the Simulation lemma with $\cR = [0, \nicefrac{1}{4})$ and $\cU = [\nicefrac{1}{4}, 1]$, the game $\fG_6$ is easier than $\fG_5$ (i.e., $\Rs_T(\fG_5) \ge \Rs_T(\fG_6)$).

\paragraph{Step 7}

Finally, consider the game 
$\sG_7 := \brb{ \{1,2,3\}, \{1,2\}, \{0,1\}, \rho_7, \fhi_7, \sP_7 } $, where in matrix notation, $\rho_7 = \bsb{ \rho(i,j) }_{i \in \{1,2,3\}, j \in \{1,2\}}$ and $\fhi_7 = \bsb{ \fhi(i,j) }_{i \in \{1,2,3\}, j \in \{1,2\}}$ are given by
\[
\rho_7:= \frac{1}{96}\begin{bmatrix}
34 & 34   \\
45 & 37   \\
38 & 44
\end{bmatrix} \;,
\qquad
\fhi_7:=\begin{bmatrix}
1 & 0   \\
0 & 0   \\
0 & 0
\end{bmatrix} \;,
\]
and $\sP_7$ is the set of all measures of the form $\otimes_{t\in \N} \brb{ \frac{1+\e}{2}\delta_1 + \frac{1-\e}{2} \delta_2 }$, for $\e \in [-1,1]$.
Thus, using again the Embedding lemma, this time with $\slf\colon [0,1] \to \{1,2,3\}$, $p\mapsto \I\{p < \nicefrac{1}{4}\} + 2 \I\{ \nicefrac{1}{4} \le p \le \nicefrac{1}{3}\} + 3 \I \{ \nicefrac{1}{3} < p \}$, $\slg\colon \{0,1\} \to \{0,1\}$, $i\mapsto i$, and $\slh \colon \otimes_{t\in \N} \brb{ \frac{1+\e}{2}\delta_1 + \frac{1-\e}{2} \delta_2 } \otimes \bmu_L \mapsto \tilde \bmu_\e \otimes \bmu_L$, we obtain that $\fG_7$ is easier than $\sG_6$ (i.e., $\Rs_T(\fG_6) \ge \Rs_T(\fG_7)$).
This last game is an instance of the so-called revealing action partial monitoring game, whose minimax regret is known to be lower bounded by $\frac{11}{96} \brb{ \frac{1}{7} T^{2/3} }$ \cite{cesa2006regret}.
In conclusion, we proved that $\Rs_T(\sG)\ge \Rs_T(\sG_7) \ge \frac{11}{672} T^{2/3}$.
\end{proof}

\section{Linear Lower Bound Under Realistic Feedback (bd) \tccheck}
\label{s:lower-bd-appe}

In this section, we prove that in the realistic-feedback case, no strategy can achieve sublinear regret in the worst case if the valuations of the buyer and the seller may be dependent, not even if they have a bounded density.

The idea of the proof is to exploit the lack of observability in this setting, building a family of scenarios $\P^\lambda$ (parameterized by $\lambda \in [0,1]$) as convex combinations of the two measures in \cref{f:linear-lower-bound-lip}.
If $\lambda < \nicefrac{1}{2}$, the optimal action is $\nicefrac{3}{8}$, while if $\lambda > \nicefrac{1}{2}$, the optimal action becomes $\nicefrac{5}{8}$.
This family is built is such a way that the feedback gives no information on $\lambda$, making it impossible to distinguish between the two cases.
Leveraging the Embedding and Simulation lemmas (\cref{l:embedding,l:simulation}), this construction leads to a reduction to an instance of a non-observable partial monitoring game, whose regret is trivially lower bounded by $T/24$.

\begin{theorem*}[\cref{thm:lower-real-bd}]
In the realistic-feedback stochastic (iid) setting with joint density bounded by a constant $M\ge \nicefrac{64}{3}$ (bd), for all horizons $T \in \N$, the minimax regret satisfies
\[
    \Rs_T \ge \frac{1}{24} T \;.
\]
\end{theorem*}

\begin{proof}
Fix any horizon $T \in \N$ and $M\ge \nicefrac{64}{3}$.
Recalling \cref{s:biltrad-setting-appe}, the realistic-feedback stochastic (iid) setting with joint density bounded by $M$ (bd) is a game $\sG := (\cX, \cY, \cZ, \rho, \varphi, \sP)$, where $\cX = [0,1]$, $\cY = [0,1]^2$, $\cZ = \{0,1\}^2$, $\rho = \gft$, $\fhi\colon \brb{p, (s,b)} \mapsto \brb{ \I\{s\le p\},\,\I\{p\le b\} }$, and $\sP = \sPbd$.
Define the two joint densities
$
    f 
= 
    \frac{64}{3} 
    \brb{ 
        \I_{[\nicefrac{0}{8}, \nicefrac{1}{8}]\times[\nicefrac{3}{8}, \nicefrac{4}{8}]} 
    +
        \I_{[\nicefrac{2}{8}, \nicefrac{3}{8}]\times[\nicefrac{7}{8}, \nicefrac{8}{8}]} 
    +
        \I_{[\nicefrac{4}{8}, \nicefrac{5}{8}]\times[\nicefrac{5}{8}, \nicefrac{6}{8}]} 
    }
$
and $g\colon [0,1]^2\to [0,M]$, $(s,b) \mapsto f(1-b,1-s)$ (see~\cref{f:linear-lower-bound-lip}, left).
Let $\fP_1$ be the subset of $\sPbd$ whose elements have the form $\bmu_\lambda := \otimes_{t \in \N} \brb{ \brb{ (1-\lambda)f + \lambda g }  (\mu_L\otimes\mu_L) }$ for $\lambda \in [0,1]$.
Since $\fP_1 \subset \fP$ the game $\fG_1 := (\cX, \cY, \cZ, \rho, \fhi, \fP_1)$ is easier than $\sG$ (i.e., $\Rs_T(\fG) \ge \Rs_T(\fG_1)$) by the Embedding lemma (\cref{l:embedding}) with $\slf$ and $\slg$ as the identities, and $\slh$ as the inclusion.
Define $\cZ_1 := \{0\}$ and $\fhi_1 \colon \cX \times \cY \to \cZ_1 \ , \brb{p,(s,b)} \mapsto 0$.
Let $(Y_t)_{t\in \N}$ be the adversary's actions in $\cG_1$.
Now, since for all $t\in \N$, any two scenarios $\P$ and $\Q$ of game $\sG_1$, and each $p \in [0,1]$,  $\P_{\fhi(p,Y_t)} = \Q_{\fhi(p,Y_t)}$, 
then by the well-known Skorokhod representation \cite[Section 17.3]{williams1991probability}, for each $t\in \N$ and each $p \in [0,1]$ there exists $\gamma_{t,p} \colon [0,1] \to \{0,1\}^2$ such that for any scenario $\P$ of game $\sG_1$, $\P_{\varphi(x,Y_t)} = (\mu_L)_{\gamma_{t,x}}$.
Thus, the Simulation lemma (\cref{l:simulation}) with $\cR = \varnothing$ and $\cU = \cX$ implies that the game $\fG_2 := (\cX, \cY, \cZ_2, \rho, \fhi_2, \fP_1)$ is easier than $\fG_1$ (i.e., $\Rs_T(\fG_1) \ge \Rs_T(\fG_2)$). 
Define $\rho_3 \colon \cX \times \cY \to [0,1] \ , \brb{p,(s,b)} \mapsto (b-s)\I\bcb{s \le \frac{3}{8} \le b}\I\lcb{p \le \frac{1}{2}} + (b-s)\I\lcb{s \le \frac{5}{8} \le b}\I\lcb{p > \frac{1}{2}}$ and $\fG_3 := (\cX, \cY, \cZ_2, \rho_3, \varphi_2, \fP_1)$. 
By the Embedding lemma with $\slf,\slg,\slh$ as the identities, we have that the game $\fG_3$ is easier than the game $\fG_2$ (i.e., $\Rs_T(\fG_2) \ge \Rs_T(\fG_3)$). 
Finally, consider the game 
$\sG_4 := \brb{ \{1,2\}, \{1,2\}, \{0\}, \rho_4, \fhi_4, \sP_4 } $, where in matrix notation, $\rho_4 = \bsb{ \rho(i,j) }_{i,j \in \{1,2\}}$ and $\fhi_4 = \bsb{ \fhi(i,j) }_{i,j \in \{1,2\}}$ are given by
\[
\rho_4:= \begin{bmatrix}
\nicefrac{1}{3} & \nicefrac{1}{4}   \\
\nicefrac{1}{4} & \nicefrac{1}{3} 
\end{bmatrix} \;,
\qquad
\fhi_4:=\begin{bmatrix}
0 & 0   \\
0 & 0 
\end{bmatrix} \;,
\]
and $\sP_4$ is the set of all measures of the form $(1-\lambda)\delta_1 + \lambda\delta_2$, for $\lambda \in [0,1]$.
Using again the Embedding lemma, this time with $\slf\colon [0,1] \to \{1,2\}$, $p\mapsto \I\{p \le \nicefrac{1}{2}\} + 2 \I\{ \nicefrac{1}{2} < p \}$, $\slg\colon \{0\} \to \{0\}$, $i\mapsto i$, and $\slh \colon \otimes_{t\in \N} \brb{ (1-\lambda)\delta_1 + \lambda \delta_2 } \otimes \bmu_L \mapsto 
\bmu_\lambda \otimes \bmu_L$, we obtain that $\fG_4$ is easier than $\sG_3$ (i.e., $\Rs_T(\fG_3) \ge \Rs_T(\fG_4)$).
This last game has (trivially) minimax regret at most $\brb{ \frac{1}{3}-\frac{1}{4} } \frac{T}{2}$.
In conclusion, we proved that $\Rs_T(\sG) \ge \Rs_T(\sG_4) \ge \frac{1}{24} T$.
\end{proof}

\section{Linear Lower Bound Under Realistic Feedback (iv) \tccheck}
\label{sec:linear_real-appe}

In this section, we prove that in the realistic-feedback case, no strategy can achieve sublinear regret without any limitations on how concentrated the distributions of the valuations of the seller and buyer are, not even if they are independent of each other (iv).

The idea of the proof is that if the two distributions are very concentrated in a small region, finding an optimal price is like finding a needle in a haystack.
Each strategy that (at each time step) receives as feedback only a finite number of bits, as in our realistic setting, can assign positive probability to at most a countable set of points.
Thus one could find concentrated distributions of the buyer and seller that have a unique optimal point in which the strategy has zero probability of posting prices at all time steps, and such that \emph{all} other prices suffer large regret.

\begin{theorem*}[\cref{thm:lower-real-iv}]
    In the realistic-feedback stochastic (iid) setting  with independent valuations (iv), for all horizons $T \in \N$, the minimax regret satisfies
\[
    \Rs_T \ge \frac{1}{8} T \;.
\]
\end{theorem*}
\begin{proof}
To lighten the notation, for any $n \in \N$ and a family $(\lambda_k)_{k\in\N}$, we let $\lambda_{1:n} := (\lambda_1, \ldots, \lambda_{n})$.
Fix an arbitrary horizon $T \in \N$.
Recalling \cref{s:biltrad-setting-appe}, the realistic-feedback stochastic (iid) setting with independent valuations (iv) is a game $\sG := (\cX, \cY, \cZ, \rho, \varphi, \sP)$, where $\cX = [0,1]$, $\cY = [0,1]^2$, $\cZ = \{0,1\}^2$, $\rho = \gft$, $\fhi\colon \brb{p, (s,b)} \mapsto \brb{ \I\{s\le p\},\,\I\{p\le b\} }$, and $\sP = \sPiv$.
Let $\sS$ be the set of scenarios of $\sG$.
Fix a strategy $\alpha$ for game $\sG$ and let $\e \in (0, 1)$.
Define $\bar \alpha_1 := \alpha_1$, $\nu_1 := (\mu_L)_{\bar \alpha _1}$, and for each $t \in \N$ and $z_1, \dots, z_{t} \in \{0,1\}^2$, 
\[
    \bar{\alpha}_{t+1,z_{1:t}} \colon [0,1]^{t+1} \to [0,1], 
    \quad
    u_{1:t+1} \mapsto \alpha_{t+1}(u_{1:t+1}, z_{1:t})
\qquad
\text{ and }
\qquad
    \nu_{t+1,z_{1:t}} 
:= 
    (\otimes_{s=1}^{t+1}\mu_L)_{\bar{\alpha}_{t+1,z_{1:t}}} \;.
\]
Define also the set $A_1 := \bcb{x \in \lsb{0,1} \mid \nu_{1}[\{x\}] > 0 }$ and, for each $t \in \N$, the union
$
    A_{t+1} 
:= 
    \bigcup_{z_{1:t}\in \{0,1\}^2} \bcb{x \in \lsb{0,1} \mid \nu_{t,z_{1:t}}[\{x\}] > 0 }
$.
Note that, for each $t\in \N, A_t$ is countable, being the union of $4^{t-1}$ countable sets.
Then $A := \bigcup_{t \in \N} A_t$ is countable.
Since $B:=[\frac{1-\e}{2}, \frac{1+\e}{2}]$ has the power of continuum, we have that the same holds for $B \m A$.
In particular, $B\m A$ is non-empty.
Pick $\xs \in B \backslash A$ and define $\mu_S := \frac{1}{2}\delta_{0}+\frac{1}{2}\delta_{\xs}$, 
$\mu_B := \frac{1}{2}\delta_{\xs} + \frac{1}{2}\delta_{1}$, 
and $\P := \lrb{ \otimes_{t \in \N} (\mu_S \otimes \mu_B) }\otimes \bmu_L  \in \sS$. Then for each $t \in \N$, we have that
\[
\E_{\P} \bsb{\rho(\xs,Y_t)} =  \frac{\xs + (1-\xs) + 1}{4}\;.
\]
On the other hand, $\P[X_1 = \xs] = \nu_1[\{\xs\}] = 0$ and for each $t \in \N$, we have that
\begin{multline*}
\P[X_{t+1} = \xs] = \P\lsb{\alpha_{t+1} (U_1, \dots, U_{t+1}, Z_1, \dots, Z_{t}) = \xs}
\\
\begin{aligned}
&= \sum_{z_1, \dots, z_{t} \in \{0,1\}^2} \P\lsb{\alpha_{t+1} (U_1, \dots, U_{t+1}, z_1, \dots, z_{t}) = \xs \cap Z_1 = z_1 \cap \dots \cap Z_{t} = z_{t}}
\\
&\le \sum_{z_1,\dots, z_{t} \in \{0,1\}^2} \P\lsb{\alpha_{t+1} (U_1, \dots, U_{t+1}, z_1, \dots, z_{t}) = \xs}
= \sum_{z_1,\dots, z_{t} \in \{0,1\}^2} \nu_{t+1,z_1,\dots, z_{t}} \lsb{\{\xs\}} = 0 \;,
\end{aligned}
\end{multline*}
which in turn gives
\begin{align*}
&
\E_{\P} \bsb{\rho(X_t,Y_t)} 
= \frac{\E_\P \lsb{\rho\brb{X_t, (0,\xs)}} + \E_\P \lsb{\rho\brb{X_t, (\xs,1)}} + \E_\P \lsb{\rho\brb{X_t, (0,1)}} + \E_\P \lsb{\rho\brb{X_t, (\xs,\xs)}}}{4}
\\
& \hspace{29.25865pt} = \frac{\xs \P_{X_t} \bsb{[0,\xs]} + (1-\xs) \P_{X_t} \bsb{[\xs,1]} + 1}{4}
= \frac{\xs \P_{X_t} \bsb{[0,\xs)} + (1-\xs) \P_{X_t} \bsb{(\xs,1]} + 1}{4}
\\
& \hspace{29.25865pt}
\le \frac{ \max(\xs, 1-\xs) + 1 }{4}
= \frac{\xs + (1-\xs) + 1 - \min(\xs, 1-\xs)}{4} \;.  
\end{align*}
So, if $T \in \N$ we get
\[
R^{\P}_T(\alpha) = \E_{\P}\lsb{\sum_{t=1}^T \rho(x^\star,Y_t) - \sum_{t=1}^T \rho(X_t,Y_t)} \ge  \frac{\min(x^\star, 1-x^\star)}{4} T \ge \frac{1-\e}{8} T.
\]
Since $\e$ was arbitrary, we get, for all $T \in \N$,
$
R_T^{\sS}(\alpha) = \sup_{\P \in \sS} R^{\P}_T(\alpha) \ge \sup_{\e \in (0, 1)} \frac{1-\e}{8} T = \nicefrac{T}{8}
$.
Since $\alpha$ was arbitrary we get, for each $T \in \N$,
$
\Rs_T = \inf_{\alpha \in \fA} R_T^{\sS}(\alpha) \ge \nicefrac{T}{8}
$.
\end{proof}

\section{Adversarial Setting: Linear Lower Bound Under Full Feedback \tccheck}
\label{sec:adversarial-appe}

In this section, we give a more detailed proof of \cref{thm:adv-lower} with a notation consistent to our abstract setting of sequential games.

\begin{theorem*}[\cref{thm:adv-lower}]
In the full-feedback adversarial (adv) setting, for all horizons $T \in \N$, we have
\[
    \Rs_T \ge \frac{1}{4} T \;.
\]
\end{theorem*}
\begin{proof}
Recalling \cref{s:biltrad-setting-appe}, the full-feedback adversarial (adv) bilateral trade setting is a game $\sG := (\cX, \cY, \cZ, \rho, \varphi, \sP)$, where $\cX = [0,1]$, $\cY = [0,1]^2$, $\cZ = [0,1]^2$, $\rho = \gft$, $\fhi\colon \brb{p, (s,b)} \mapsto (s,b)$, and $\sP = \sPadv$.
Let $\sS$ be the set of scenarios of $\sG$.
Fix a strategy $\alpha \in \fA$ and an $\e \in (0, 1/18)$. 
Define $\bar{\alpha}_1 := \alpha_1$, $\nu_1 := (\mu_L)_{\bar{\alpha}_1}$, and
\[
    \begin{cases}
        c_1:=\frac{1}{2}-\frac{3}{2}\varepsilon, \  d_1:=\frac{1}{2}-\frac{1}{2}\varepsilon, \  s_1 := 0,  \ b_1:=d_1,
    &
        \text{ if } \nu_{1}\bsb{\bsb{0,\frac{1}{2}-\frac{1}{2}\varepsilon}}\le\frac{1}{2} \;,
    \\
        c_1:=\frac{1}{2}+\frac{1}{2}\varepsilon, \  d_1:=\frac{1}{2}+\frac{3}{2}\varepsilon, \  s_1 := c_1, \  b_1:=1,
    &
        \text{ otherwise}.
    \end{cases}
\]
If $t \in \N$, suppose we defined $\bar{\alpha}_t, \nu_t, c_t, d_t, s_t, b_t$ and let
\[
\bar{\alpha}_{t+1} : [0,1]^{t+1} \to [0,1], (u_1,\dots,u_{t+1}) \mapsto \alpha_{t+1} \lrb{ u_1, \dots, u_{t+1}, (s_1,b_1), \dots, (s_t,b_t) }, 
\]
$\nu_{t+1} := \brb{\otimes_{s=1}^{t+1}\mu_L}_{\bar{\alpha}_{t+1}}$, and
\[
    \begin{cases}
        c_{t+1}:=c_{t}, \ d_{t+1}:=d_{t}-\frac{2\varepsilon}{3^{t}}, \ s_{t+1}:=0, \ b_{t+1}:=d_{t+1},
    &
        \text{if } \nu_{t+1}\bsb{ \bsb{0,c_t+\frac{\varepsilon}{3^{t}}} }\le\frac{1}{2} \;,
    \\
        c_{t+1}:=c_{t}+\frac{2\varepsilon}{3^{t}}, \ d_{t+1}:=d_{t}, \ s_{t+1} := c_{t+1}, \ b_{t+1}:=1,
    &
        \text{otherwise}.
    \end{cases}
\]
Then the sequences $(\bar{\alpha}_t)_{t \in \N}, (\nu_t)_{t \in \N}, (c_t)_{t \in \N}, (d_t)_{t \in \N}, (s_t)_{t \in \N}, (b_t)_{t \in \N}$ are well-defined by induction and satisfy:
\begin{itemize}
    \item for each $t \in \N, d_t-c_t = \frac{\varepsilon}{3^{t-1}}$;
    \item  for each $t \in \N, c_1\le c_2 \le c_3 \le \dots \le c_t \le d_t \le \dots \le d_3 \le d_2 \le d_1$;
    \item $\exists! x^\star \in \bigcap_{t=1}^{\infty}[c_t,d_t]$;
    \item for each $t \in \N, \rho\lrb{x^\star, (s_t,b_t)} = b_t - s_t \ge \frac{1-3\e}{2}$;
    \item for each $t\in \N, \P\bsb{\alpha_t \brb{U_1,\dots, U_t, (s_1,b_1), \dots, (s_{t-1},b_{t-1})} \in [s_t, b_t] }\le\frac{1}{2}$
\end{itemize}
Now, define $\P := \lrb{ \otimes_{t \in \N} \delta_{(s_t,b_t)}}  \otimes \bmu_L \in \sS$. Then, for each $t \in \N$,
\begin{align*}
    \E_{\P}[\rho \lrb{X_t, Y_t} ] &= \E_{\P}\Bsb{ \rho\Brb{\alpha_t \brb{U_1,\dots, U_t, (s_1,b_1), \dots, (s_{t-1},b_{t-1})} , (s_t,b_t)} }
    \\
    &\le \lrb{\frac{1}{2}+\frac{3\varepsilon}{2}} \P\bsb{\alpha_t \brb{U_1,\dots, U_t, (s_1,b_1), \dots, (s_{t-1},b_{t-1})} \in [s_t, b_t] } \le  \frac{1}{4}+\frac{3\varepsilon}{4} \;,
\end{align*}
and so, for each $T \in \N$
\begin{align*}
    R^{\P}_T(\alpha) & = \E_{\P}\lsb{\sum_{t=1}^T \rho(x^\star,Y_t) - \sum_{t=1}^T \rho(X_t,Y_t)} = \sum_{t=1}^T \rho(x^\star,(s_t,b_t)) -  \sum_{t=1}^T \E_{\P} \lsb{  \rho \lrb{X_t, Y_t} }
    \\
    &\ge
    \sum_{t=1}^T (b_t-s_t)\brb{1 - \P\bsb{\alpha_t \brb{U_1,\dots, U_t, (s_1,b_1), \dots, (s_{t-1},b_{t-1})} \in [s_t, b_t] } }
    \ge \frac{1-3\e}{4}T \;.
\end{align*}
Since $\e$ was arbitrary, we get, for all $T \in \N$,
$
R_T^{\sS}(\alpha) = \sup_{\P \in \sS} R^{\P}_T(\alpha) \ge \sup_{\e \in (0, 1/18)} \frac{1-3\e}{4}T = \frac{T}{4}
$.
Since $\alpha$ arbitrarity, we get, for each $T \in \N$,
$
\Rs_T = \inf_{\alpha \in \fA} R_T^{\sS}(\alpha) \ge \frac{T}{4}
$.
\end{proof}

\end{document}